\documentclass{article}
\usepackage{PRIMEarxiv}
\usepackage[utf8]{inputenc}
\usepackage[T1]{fontenc}
\usepackage{booktabs}
\usepackage{amsfonts}
\usepackage{amssymb}
\usepackage{amsthm}
\usepackage{mathtools}
\usepackage{amsmath,amssymb,amsthm,mathtools,bm,mathrsfs}
\usepackage{booktabs,tabularx,array,graphicx,xcolor,microtype}
\usepackage{natbib}
\usepackage[unicode,hidelinks]{hyperref}
\usepackage{placeins,float}
\usepackage[capitalise,noabbrev]{cleveref}
\usepackage{tabularx}
\usepackage{array}
\usepackage{graphicx}
\usepackage{float}
\usepackage{placeins}
\usepackage{xcolor}
\usepackage{natbib}
\usepackage{hyperref}
\usepackage{cleveref}
\usepackage{url}
\usepackage{microtype}
\newcolumntype{P}[1]{>{\raggedright\arraybackslash}p{#1}}
\newcolumntype{Y}{>{\centering\arraybackslash}X}
\newtheorem{theorem}{Theorem}[section]
\newtheorem{proposition}[theorem]{Proposition}
\newtheorem{corollary}[theorem]{Corollary}

\newtheorem{assumption}[theorem]{Assumption}
\newtheorem{statement}[theorem]{Statement}
\newtheorem{conjecture}[theorem]{Conjecture}
\theoremstyle{remark}
\newtheorem{remark}[theorem]{Remark}

\usepackage{mathtools}
\usepackage{amsmath}
\newcommand{\Vstar}{V_\star}

\definecolor{brown}{RGB}{115,65,0}

\title{Does Acceleration Accelerate Statistics?\\ Exact Finite-Budget Calibration of Averaged Momentum Methods}
\author{Anonymous authors}
\date{}
\newcommand{\R}{\mathbb{R}}
\newcommand{\E}{\mathbb{E}}
\newcommand{\Pp}{\mathbb{P}}
\newcommand{\Cov}{\operatorname{Cov}}
\newcommand{\Var}{\operatorname{Var}}
\newcommand{\tr}{\operatorname{tr}}
\newcommand{\diag}{\operatorname{diag}}
\newcommand{\spec}{\operatorname{spec}}

\newcommand{\cN}{\mathcal{N}}
\newcommand{\cF}{\mathcal{F}}
\newcommand{\cL}{\mathcal{L}}

\newcommand{\Op}{O_{\mathbb P}}

\newcommand{\norm}[1]{\left\lVert #1\right\rVert}
\newcommand{\abs}[1]{\left\lvert #1\right\rvert}
\newcommand{\inner}[2]{\left\langle #1,#2\right\rangle}
\newcommand{\ind}{\mathbf{1}}
\newcommand{\trans}{\mathsf{T}}

\newcommand{\Wass}{d_{\mathrm W}}

\setcitestyle{authoryear,round}
\definecolor{conferenceblue}{HTML}{0057B8}
\hypersetup{
  colorlinks=true,
  linkcolor=black,
  citecolor=conferenceblue,
  urlcolor=conferenceblue,
  filecolor=conferenceblue
}

\newcommand{\dd}[1]{{\color{black}#1}}

\title{When Momentum Helps Polyak--Ruppert Averaging}
\author{
  Alexandra Suvorikova \\
  Independent researcher
  \And
 Darina Dvinskikh \\
  HSE University \\
  \texttt{dmdvinskikh@hse.ru} \\
  \And
   Yuriy Dorn\\
  MSU AI Institute\\
  \And
   Alexander Gasnikov\\
  Innopolis University\\
}
\date{}
\begin{document}

\maketitle

\begin{abstract}
Polyak--Ruppert averaging gives SGD and momentum methods
the same asymptotic covariance.
We characterize how fast finite runs reach it
and how far they can be parallelized.
For strongly convex quadratics with additive noise,
we prove an online--offline bridge:
the empirical risk minimizer on the samples of a tail average
is one Newton step from the tail-averaged query point of any online procedure.
For SGD, heavy ball, and Nesterov's method with constant momentum,
this step is a boundary term at the two ends of the tail
and splits the gap to the offline estimator
into center, covariance, and shape errors.
Necessary and sufficient conditions show that,
on the least curved direction, the center error sets the pace:
SGD needs $m\gg\kappa^2$ rounds but momentum only $m\gg\kappa$,
where $\kappa$ is the condition number,
so at sample size $N$ the critical batch size
grows from $\Theta(\sqrt{N}/\kappa)$ to $\Theta(\sqrt{N/\kappa})$.
The covariance error separates the momentum methods:
heavy ball has relative variance excess $\Theta(\kappa^{3/2}/m)$ along the stiffest direction,
whereas Nesterov has relative error $O(\sqrt{\kappa}/m)$ in every direction.
The shape error is method-independent by a Berry--Esseen bound
uniform in $\kappa$ and batch size,
so confidence intervals are valid when the other two errors vanish.
We further examine how the conclusions change under heavy-tailed or state-dependent noise and clipping, and establish a first-order comparison with ERM for nonlinear Nesterov averaging under explicit conditions.
\end{abstract}

\section{Introduction}
\label{sec:introduction}

Averaging connects stochastic optimization with statistical estimation.
Under suitable conditions, an averaged trajectory has an asymptotically
Gaussian error with the Polyak--Ruppert covariance, providing a basis for
confidence sets as well as point estimates
\citep{polyak1992acceleration}. Yet a limiting law does not specify how
many updates are necessary before it accurately describes the output of an
algorithm. The distinction matters when observations are
processed in large batches: at a fixed sample budget, larger batches
leave fewer updates for the algorithm to approach its asymptotic regime.
This makes the transient behavior of the method central to statistical
inference. Momentum can shorten the persistence of initialization,
suggesting that it could make the limiting approximation accurate
earlier. However, it also changes how gradient noise accumulates along
the trajectory, so faster decay of initialization does not
guarantee this benefit. We therefore ask:
\emph{Does momentum reduce the number of updates needed for the
asymptotic distribution to accurately describe the averaged output?}

Answering this question requires separating three properties of the
output: its initialization bias should be negligible on the statistical
scale, its covariance should match the target covariance, and its
centered, standardized law should be close to Gaussian.

The question builds on several established lines of work. Analyses of averaged linear stochastic approximation characterize
asymptotic covariance and nonasymptotic concentration
\citep{mou2020finegrained}.Gaussian approximation and bootstrap methods
provide finite-sample inference for averaged stochastic approximation and
SGD \citep{samsonov2024gaussian,sheshukova2026gaussian}; in particular,
approximating the finite-sample distribution can be more accurate than
using the limiting covariance. Complementary analyses quantify the
trade-off between acceleration and noise amplification
\citep{mohammadi2021robustness}, including the effect of averaging on
heavy-ball dynamics \citep{samuelson2024averaging}.
\emph{Our focus is on when the initialization bias becomes negligible
relative to statistical fluctuations and the covariance of the averaged
output matches its asymptotic prediction up to a small relative error.}

We study constant-parameter SGD, heavy ball (HB), and Nesterov's
accelerated gradient method (NAG) on strongly convex quadratics with
Hessian $H$, extreme eigenvalues $\mu,L$, and condition number
$\kappa=L/\mu$. The additive oracle noise is independent and centered,
with covariance $\Sigma$. From $N=mb$ observations in $m$ batches of
size $b$, each method returns the average $\widehat x$ of its query
points. We derive the pathwise representation
\[
\widehat x-x_\star
=\underbrace{-H^{-1}\bar\xi}_{\widehat x^{\rm ERM}-x_\star}+R,
\]
where $\bar\xi$ is the sample noise mean, $\widehat x^{\rm ERM}$ is
the empirical minimizer on the same observations, and $R$ is an
explicit method-dependent boundary term involving only trajectory
endpoints. The resulting finite-window laws separate initialization
bias, covariance corrections, and Gaussian approximation error.
\paragraph{{Contribution.}}
Our main contributions are as follows.

\begin{itemize}
\item \textbf{Bias and variance thresholds.}
For full averages from a common deterministic start, we give necessary
and sufficient conditions for negligible bias on the statistical scale
and vanishing relative variance error in the least-curved eigendirection.
Let $v_\mu$ be a unit eigenvector of $H$ corresponding to $\mu$, and set $\theta:=\frac{\mu\,v_\mu^\top(x_0-x_\star)}{\sigma_\mu}$, $\sigma_\mu^2:=v_\mu^\top\Sigma v_\mu>0$.
These conditions yield the
following thresholds (Section~\ref{sec:main-calibration}).
\end{itemize}
\begin{table}[ht]
\begin{center}
{\color{black}
\caption{Calibration thresholds on the least-curved direction for full averages from a deterministic start.}
\label{tab:tab}
}
\small
\begin{tabularx}{\linewidth}{@{}P{0.39\linewidth}Y Y Y@{}}
\toprule
Criterion in the $\mu$ direction & SGD & HB & NAG\\
\midrule
Vanishing relative variance error &
$m\gg\kappa$ &
$m\gg\sqrt\kappa$ &
$m\gg\sqrt\kappa$\\
Negligible normalized bias and relative variance error,
$b=1$, fixed $\theta\ne0$ &
$m\gg\kappa^2$ &
$m\gg\kappa$ &
$m\gg\kappa$\\
Initialization-induced batch scale at fixed $N$, up to constants &
$\dfrac{\sqrt N}{|\theta|\kappa}$ &
$\dfrac{\sqrt N}{|\theta|\sqrt\kappa}$ &
$\dfrac{\sqrt N}{|\theta|\sqrt\kappa}$\\
\bottomrule
\end{tabularx}

\end{center}
\end{table}
\begin{itemize}
\item \textbf{Uniform Gaussian approximation.}
Under a finite third-moment assumption, we prove an $O(N^{-1/2})$
Kolmogorov bound for the least-curved projection after centering
and scaling, uniformly in $\kappa$ and $b$. Approximation by the
Polyak--Ruppert law additionally incurs an initialization term and a variance correction of order \(\kappa/m\) for SGD and \(\sqrt{\kappa}/m\) for HB and NAG. This decomposition quantifies coverage errors and identifies which terms acceleration improves (Section~\ref{sec:main-calibration}).

\item \textbf{A spectral separation between HB and NAG.}
For the stated HB tuning, the relative variance
in the most-curved direction is
$1+\kappa^{3/2}(1+o(1))/(8m)$ as $\kappa\to\infty$ and
$m/\sqrt{\kappa}\to\infty$; calibration there requires
$m\gg\kappa^{3/2}$. For the stated NAG tuning, the relative variance
error is at most $30\sqrt{\kappa}/m$ uniformly over Hessian
eigendirections, with a matching-order correction in the least-curved
direction (Section~\ref{sec:main-acceleration}).

\item \textbf{Extensions to heavy tails and nonlinear objectives.}
We derive conditions under which the normalized averaged output
converges to the same stable law as the ERM
in the scalar model, and show that finite variance alone does not guarantee uniformly valid sub-Gaussian confidence bounds centered at the ERM or at the averaged algorithmic output. For smooth strongly
convex objectives with additive noise and a Lipschitz Hessian, a
trajectory-dependent remainder extends the online--offline comparison.
Under explicit burn-in and batch-growth conditions, tail-averaged NAG
is first-order equivalent to the ERM at the same total observation budget.
(Section~\ref{sec:main-extensions}).
\end{itemize}

The sharp thresholds concern full averages from a cold start under the
specified parameter policies. Burn-in and tail averaging can change
these thresholds; the NAG constants refer to averages of query points.
The resulting comparison therefore makes the output rule, initialization,
and spectral direction explicit when assessing the statistical benefit
of acceleration.

\paragraph{Organization.}
\Cref{sec:setup} specifies the model, algorithms, output, and same-data
ERM benchmark, while \cref{sec:main-calibration,sec:main-acceleration}
develop the finite-budget calibration criteria and spectral comparison.
\Cref{sec:main-extensions} treats heavy-tailed noise and nonlinear
objectives, \cref{sec:numerical-summary} reports the experiments, and the
appendix contains proofs, auxiliary results, and extended experimental
details.

\section{Problem Statement}
\label{sec:setup}
\dd{This section specifies the common model, algorithms, output, and same-data ERM benchmark.}

We study averaged stochastic first-order methods on
\begin{equation}
 f(x):=\tfrac12\inner{x-x_\star}{H(x-x_\star)},
 \label{eq:quadratic}
\end{equation}
where $H=H^\trans\succ0$, $0<\mu\le L$,
$\mu I\preceq H\preceq LI$, and $\kappa:=L/\mu$.
In comparisons on the extreme eigendirections, $\mu$ and $L$ denote
the smallest and largest eigenvalues of the original Hessian.
The oracle returns $G(x,\zeta)=H(x-x_\star)+\xi(\zeta)$.

\begin{assumption}[Additive noise]
\label{ass:additive-noise}
Unless stated otherwise, all copies of $\xi:=\xi(\zeta)$ are independent, with $\E \xi = 0$ and finite covariance $\Sigma$.
\end{assumption}

\paragraph{Batching and methods.}
At round $k$, a batch of $b\ge1$ samples gives
$\eta_{k+1}^{(b)}:=b^{-1}\sum_{j=1}^b\xi_{k+1,j}$,
with covariance $\Sigma/b$.
The method queries the oracle at a point $s_k$ and receives the
batch-averaged answer $g_k:=b^{-1}\sum_{j=1}^bG(s_k,\zeta_{k+1,j})$.
The two-step family computes
\[
 s_k=x_k+\gamma(x_k-x_{k-1}),~~
 x_{k+1}=x_k+\beta(x_k-x_{k-1})-\alpha g_k,
\]%
with constant $\alpha>0$, $\beta\in[0,1)$, and $\gamma\ge0$.
SGD has $(\beta,\gamma)=(0,0)$, HB has $\gamma=0$, and NAG has
$\gamma=\beta$. The parameter policies used in the comparison are
given in \cref{prop:cold-bias}.
Writing $e_k:=x_k-x_\star$ and $u_k:=s_k-x_\star$, so that
$g_k=Hu_k+\eta_{k+1}^{(b)}$, we get the error recursion,
\begin{align}
 u_k&=e_k+\gamma(e_k-e_{k-1}),                  \label{eq:query}\\
 e_{k+1}&=e_k+\beta(e_k-e_{k-1})-\alpha\bigl(Hu_k+\eta_{k+1}^{(b)}\bigr).
 \label{eq:two-step}
\end{align}
Unless stated otherwise, all methods use the same deterministic
start with $e_{-1}=e_0$, and parameters and averaging windows are
fixed before observing the noise.
We write $\cF_k:=\sigma(\xi_{i,j}:i\le k,\ 1\le j\le b)$;
thus $u_k$ is $\cF_k$-measurable and $\eta^{(b)}_{k+1}$ is
independent of $\cF_k$.

{
\color{black}
\paragraph{Output and comparison.}
We define the output by averaging the points $s_k$ at which
the stochastic gradients are evaluated. For $m\ge1$ and $t\ge0$,
the resulting estimator of $x_\star$ is
\[
 \widehat x_{t,m}:=\frac1m\sum_{k=t}^{t+m-1}s_k,
 ~~
 \widehat x_{t,m}-x_\star
 =\bar u_{t,m}:=\frac1m\sum_{k=t}^{t+m-1}u_k.
\]
Here $t=0$ gives a full average from a cold start, while $t>0$
gives tail averaging after $t$ burn-in rounds.
For SGD and HB, $s_k=x_k$, so this output coincides with the
iterate average. For NAG, it averages the extrapolated query
points and satisfies
\[
 \widehat x_{t,m}
 =\frac1m\sum_{k=t}^{t+m-1}x_k
  +\frac{\gamma}{m}(x_{t+m-1}-x_{t-1}).
\]
All finite-time NAG constants below refer to the query-point
average $\widehat x_{t,m}$.
}

\paragraph{Offline reference.}
Let $N:=mb$.
The offline reference is the minimizer
$\widehat x^{\rm ERM}_{t,m}$ of the cumulative loss
$\tfrac1N\sum_{k=t}^{t+m-1}\sum_{j=1}^b\ell(x,\xi_{k+1,j})$, with $\ell(x,\xi):=f(x)+\langle\xi,x\rangle$,
\[
 \widehat x^{\rm ERM}_{t,m}-x_\star
 =-\frac{H^{-1}}m\sum_{k=t}^{t+m-1}\eta_{k+1}^{(b)}.
\]
Its covariance is $V_\star/N$, where
$V_\star:=H^{-1}\Sigma H^{-1}$ is the Polyak--Ruppert (PR)
covariance.
The last batch $\eta^{(b)}_{t+m}$ is included in this ERM
and in $N$, although it does not affect $\widehat x_{t,m}$.
This benchmark uses the empirical risk, whereas the online methods
only see batch-averaged gradients at their query points.

We call a method stable if, in the absence of noise, its iterates
$x_k$ converge to $x_\star$ for every initial point; for the stated parameter range this holds
if and only if $\alpha\lambda_{\max}(H)(1+2\gamma)<2(1+\beta)$.
Appendix~\ref{appx:schur_stability} provides the details.

\subsection{Averaged SGD: bounds and risk}
\label{sec:same-sgd-new}

Consider SGD with $0<\alpha\le1/L$, a deterministic start satisfying
$\|x_0-x_\star\|\le R_0$, and the averaged output
$\widehat x:=\widehat x_{0,m}=m^{-1}\sum_{k=0}^{m-1}x_k$.
Let $N=mb$, $\sigma^2:=\tr\Sigma$, and
$\mathcal R_{m,b}:=\mathbb E\|\widehat x-x_\star\|^2$.
A descent argument and quadratic contraction, respectively, yield
\begin{subequations}\label{eq:sgd-main-bounds}
\begin{align}
 \mathcal R_{m,b}
 &\le U^{\rm da}_{m,b}:=
 \frac{R_0^2}{\alpha\mu m}
 +\frac{\alpha\sigma^2}{\mu b},
 \label{eq:sgd-da}\\
 \mathcal R_{m,b}
 &\le U^{\rm quc}_{m,b}:=
 \frac{R_0^2}{\alpha^2\mu^2m^2}
 +\frac{\sigma^2}{\mu^2N}.
 \label{eq:sgd-quc}
\end{align}
\end{subequations}
Both inequalities concern the same estimator. At fixed $\alpha$ and $b$,
the first leaves a noise term independent of $m$, whereas the second
captures its $1/N$ decay.

Both bounds summarize the problem through scalar parameters and therefore
lose information about the relative orientation of $H$, $\Sigma$, and
$e_0$. For example, fixing $H$ and $e_0$ and placing the same noise
variance along a $\mu$-eigendirection instead of an $L$-eigendirection
leaves the bias and both bounds unchanged. Nevertheless, the ratio of
the resulting variance contributions tends to $\kappa^2$ as
$m\to\infty$ at fixed $\kappa$
(\cref{prop:geometry-pair-new}).

The MSE, its second-order expansion, and a refined finite-window
contraction bound are given in Appendix~\ref{app:same-sgd-new}.
The refined bound equals the supremum of the risk over the class
specified by the scalar parameters (\cref{prop:sharp-scalar-new}).
Thus, worst-case sharpness is compatible with substantial
differences between individual problems sharing those parameters.

\paragraph{From risk to the output distribution.}

For either upper bound $U$, Markov's inequality gives a valid confidence
ball with radius $\sqrt{U/\delta}$ at level $1-\delta$.
However, agreement with the ERM risk does not ensure agreement of the
mean and covariance separately, since
\[
\mathcal R_{m,b}
=\|\mathbb E\widehat x-x_\star\|^2
+\tr\operatorname{Cov}(\widehat x).
\]
Squared bias can compensate for a variance deficit, even when the
relative risk equals one (\cref{sec:main-calibration}).
We therefore analyze the finite-window laws of the averaged outputs,
separating initialization bias, covariance error, and Gaussian
approximation error.

\section{Bias, variance, and Gaussian approximation}
\label{sec:main-calibration}

 \dd{The exact online--offline bridge reduces finite-window calibration to the analysis of a scalar response sequence. We first obtain sharp
conditions for center and variance calibration on the least-curved
direction, then separate these effects from Gaussian approximation error.}

For every trajectory of \cref{eq:two-step} and all $t\ge0$, $m\ge1$, \cref{thm:exact,prop:observable-bridge} ensure
\begin{align}
 &\widehat x_{t,m}-\widehat x^{\rm ERM}_{t,m}
 =\frac{H^{-1}\Delta_{t,m}}{\alpha m},\\
  &\Delta_{t,m}:=e_t-e_{t+m}+\beta(e_{t+m-1}-e_{t-1})
 \label{eq:main-decomposition}
\end{align}
 Hence, for i.i.d.\ noise and
every member of the family, $\sqrt N\,\bar u_{t,m}\Rightarrow\cN(0,V_\star)$
whenever $\sqrt{b/m}\,\Delta_{t,m}\to0$ in probability (\cref{thm:exact}). Under constant noise, every stable member would settle at the
same point $x_\star-H^{-1}\eta$; the endpoint term carries everything that
depends on how a method gets there. Proofs and further results are in
\cref{sec:statistical-information,sec:applicability,sec:improvement}.

\paragraph{The least-curved direction.}
Let $v_\mu$ be a unit eigenvector of $H$ for the eigenvalue $\mu$, and set
\[
 \sigma_\mu^2:=v_\mu^\trans\Sigma v_\mu,~
 \theta:=\frac{\mu\,v_\mu^\trans e_0}{\sigma_\mu},~
 Z_m:=\frac{\mu\sqrt N}{\sigma_\mu}\,v_\mu^\trans\bar u_{0,m}.
\]
Since $v_\mu^\trans H=\mu v_\mu^\trans$, the projection $v_\mu^\trans e_k$
follows the scalar recursion with curvature $\mu$, driven by independent
noises of variance $\sigma_\mu^2$. Let $r_j$ be the noiseless response of the oracle to a unit start:
$r_j:=u^\circ_j$, where $e^\circ_{-1}=e^\circ_0=1$ and
\[
 u^\circ_j=e^\circ_j+\gamma(e^\circ_j-e^\circ_{j-1}),\quad
 e^\circ_{j+1}=e^\circ_j+\beta(e^\circ_j-e^\circ_{j-1})-\alpha\mu\,u^\circ_j .
\]
By linearity, a start $v_\mu^\trans e_0$ yields the
noiseless oracle points $(v_\mu^\trans e_0)\,r_j$. Set
\[
 S_m:=\sum_{j<m}r_j,\qquad Q_m:=\sum_{j<m}r_j^2,\qquad
 B:=\sum_{j\ge0}r_j.
\]
By \cref{thm:cold-exact}, applied to the projected recursion, for all
$m,b\ge1$,
\[
 v_\mu^\trans\bar u_{0,m}=\frac{(v_\mu^\trans e_0)\,S_m}{m}
 -\frac1{\mu mb}\sum_{s=1}^{m}\sum_{\ell=1}^{b}(1-r_{m-s})\,
 v_\mu^\trans\xi_{s,\ell};
\]
multiplying by $\mu\sqrt N/\sigma_\mu$ and using $N=mb$ gives
\begin{equation}
 Z_m=\theta\sqrt{\frac bm}\,S_m
 -\frac1{\sqrt N}\sum_{s=1}^{m}\sum_{\ell=1}^{b}
 (1-r_{m-s})\,\frac{v_\mu^\trans\xi_{s,\ell}}{\sigma_\mu}.
 \label{eq:main-exact-law}
\end{equation}
The method enters only through $r_j$: through $S_m$ as the memory of the
start, and through the weights $1-r_{m-s}$ as the unfinished response to
late noise.
 
We compare SGD with $\alpha=1/L$; HB with $\alpha=4/(\sqrt L+\sqrt\mu)^2$
and $\beta=((\sqrt\kappa-1)/(\sqrt\kappa+1))^2$; and NAG with $\alpha=1/L$
and $\beta=\gamma=(\sqrt\kappa-1)/(\sqrt\kappa+1)$. For these policies
$0\le r_{j+1}\le r_j\le1$, and the memory $B$ equals $\kappa$,
$\sqrt\kappa$, and $2\kappa/(\sqrt\kappa+1)$, respectively
(\cref{prop:cold-bias}).
 
\begin{statement}
\label{thm:main-calibration}
For each policy, set $ h_m:=\E Z_m=\theta\sqrt{b/m}\, S_m$, $v_m:=\Var(Z_m)=1-\frac{2S_m-Q_m}{m}.$
Along any sequence with $m\to\infty$, allowing $\kappa\ge1$,
$b\ge1$, and $\theta$ to vary,
\begin{align*}
 v_m\to1
 &\iff \frac Bm\to0,\\
 (h_m,v_m)\to(0,1)
 &\iff \frac Bm\to0\quad\text{and}\quad
       |\theta|B\sqrt{\frac bm}\to0.
\end{align*}
\end{statement}
\cref{thm:uniform-calibration} presents the result in detail. Matching the PR risk is not
enough: the relative risk equals $v_m+h_m^2$, and
\begin{align*}
&\frac{\E|v_\mu^\trans\bar u_{0,m}|^2}{\sigma_\mu^2/(\mu^2N)}
 =v_m+h_m^2,
\\
 &\theta_m^2=\frac{2S_m-Q_m}{bS_m^2}
 \ \Longrightarrow\ v_m+h_m^2=1,
\end{align*}
also along sequences with $B/m$ bounded away from zero, where
$v_m\not\to1$.
 
\begin{statement}[Critical batch]
\label{cor:main-critical-batch}
Assume that
$v_\mu^\trans\xi_{s,\ell}/\sigma_\mu$ are i.i.d.\ with
zero mean and unit variance. Their distribution remains fixed along the
asymptotic sequence, even when $\kappa$, $b$, and $\theta$ vary. If $N\to\infty$, $ \frac{bB}{N}\to0$, and $\frac{\theta Bb}{\sqrt N}\to h\in\R$,
then $Z_m\Rightarrow\cN(h,1)$. The $(1-\delta)$ interval for
$v_\mu^\trans x_\star$ centered at $v_\mu^\trans\widehat x_{0,m}$ with
half-width $z_{1-\delta/2}\sigma_\mu/(\mu\sqrt N)$ has limiting coverage
$\Phi(z_{1-\delta/2}-h)-\Phi(-z_{1-\delta/2}-h)$.
\end{statement}
\cref{cor:shifted-clt} presents the result in detail.
For $\theta\ne0$ and $m\gg B$, the scale
$b_{\rm crit}(N):=\sqrt N/(|\theta|B)$ marks the batch size
at which the initialization bias of the averaged estimator
becomes comparable to the ERM standard error along $v_\mu$.
At a fixed budget $N=mb$, larger batches leave fewer iterations
to reduce this bias, while the ERM standard error remains unchanged.
Calibration therefore requires both $b\ll b_{\rm crit}(N)$,
so that the bias is negligible, and $b\ll N/B$,
so that the variance matches its PR limit. At $b=1$ and fixed $\theta\ne0$, the onset is $m\gg\kappa^2$
for SGD and $m\gg\kappa$ for HB and NAG. Smaller noise at a fixed start
increases $|\theta|$ and therefore delays calibrated inference. Since
$\theta^{-2}=v_\mu^\trans\Sigma v_\mu/(v_\mu^\trans\nabla f(x_0))^2$ is a
directional gradient noise scale at the initial point, $b_{\rm crit}$ is the
geometric mean of the observation budget and this noise scale, divided by
$B$; in the scalar model it is the noise scale of
\citet{mccandlish2018empirical} (\cref{rem:critical-batch}).
 
\begin{statement}[Gaussian shape versus PR calibration]
\label{thm:main-shape}
Let the noises be i.i.d.\ with
$\nu_3:=\E|v_\mu^\trans\xi|^3/\sigma_\mu^3<\infty$. There is a universal
constant $C$ such that, for each of the three policies and all
$\kappa\ge1$, $b\ge1$, $m\ge2$,
\[
 d_{\rm K}\Bigl(\cL\Bigl(\frac{Z_m-h_m}{\sqrt{v_m}}\Bigr),\cN(0,1)\Bigr)
 \le\frac{C\nu_3}{\sqrt N},
\]
and, if $m\ge4B$,
\begin{align}
 d_{\rm K}\bigl(\cL(Z_m),\cN(0,1)\bigr)
 \le \underbrace{\frac{C\nu_3}{\sqrt N}}_{\text{Gauss. approx.}}+\underbrace{2\phi(1)\frac Bm}_{\text{variance correction}}
 +\underbrace{\frac{|h_m|}{\sqrt{2\pi}}}_{\text{initialization}},
 \label{eq:main-three-errors}
\end{align}
where $\phi$ is the standard normal density.
\end{statement}
 
This is \cref{thm:sharp-normal} together with \cref{eq:three-errors}; the
variance correction $B/m$ is sharp even for Gaussian noise
(\cref{prop:sharp-normal-lower}). The coverage of the interval in
Proposition~\ref{cor:main-critical-batch} is within twice the right-hand side of
\cref{eq:main-three-errors} of $1-\delta$. Gaussian shape is thus uniform in
$\kappa$ and $b$; the center and the variance are the bottlenecks. These
intervals use known curvature and noise scale; multivariate balls and a
paid certificate for unknown $H$ and $\Sigma$ are in
\cref{sec:confidence-sets}.

\section{Acceleration gains and their limits}
\label{sec:main-acceleration}
 \dd{We now compare what acceleration changes across the Hessian spectrum.
Although HB and NAG have the same order of memory on the least-curved
direction, their covariance transients can differ sharply on stiff
directions; we also record how burn-in changes the comparison.}

Substituting $B$ into Statements~\ref{thm:main-calibration}, \ref{cor:main-critical-batch}
compares the policies on the least-curved direction, for the same
full-average output and cold start (see Tab.~\ref{tab:tab}).

Acceleration shortens the memory from $\kappa$ to order $\sqrt\kappa$ and
enlarges the admissible batch by the same factor; it does not change the
$N^{-1/2}$ term in Statement~\ref{thm:main-shape}.
 
For the whole spectrum, let $(\lambda,v_\lambda)$ be an eigenpair of $H$
with $\sigma_\lambda^2:=v_\lambda^\trans\Sigma v_\lambda>0$, and define the
relative variance
\[
 v_m(\lambda):=\frac{N\lambda^2}{\sigma_\lambda^2}
 \Var\bigl(v_\lambda^\trans\bar u_{0,m}\bigr),
\]
so that $v_m(\mu)=v_m$.
 
\begin{statement}[Heavy-ball covariance inflation]
\label{thm:main-hb-stiff}
For HB with the parameters above and $\sigma_L^2>0$,
\[
 v_m(L)=1+\frac{\kappa^{3/2}}{8m}\{1+o(1)\}
 \quad\text{as }\kappa\to\infty,\ m/\sqrt\kappa\to\infty,
\]
whereas SGD with $\alpha=1/L$ has $v_m(L)=1-1/m$ . In this regime,
covariance calibration of HB on this direction is equivalent to
$m/\kappa^{3/2}\to\infty$.
\end{statement}
\cref{thm:stiff} provides the detail.
For example, at $m=\kappa$ and $b=1$ the stiff-mode HB variance ratio
diverges as $\sqrt\kappa/8$, while the least-curved HB mode is already
close to its PR variance. A Frobenius error weighted by
$H^{-1}\Sigma H^{-1}$ can hide this, because the stiff-mode PR variance is
smaller by a factor $\kappa^2$ when the gradient-noise variances are equal.
General noise-amplification tradeoffs
\citep{mohammadi2021robustness,mohammadi2024tradeoffs} and the sign of
averaged HB modal corrections \citep[Theorem 3]{samuelson2024averaging} are
known; the precise averaged-output formula and its relation to the onset
criteria are the point of this counterexample.
 
\begin{statement}[Uniform spectral calibration of NAG]
\label{thm:main-nag}
For NAG with the parameters above and every $m,b\ge1$,
\[
 \sup_{\lambda:\,\sigma_\lambda^2>0}|v_m(\lambda)-1|
 \le\frac{30\sqrt\kappa}{m},
\]
and $1-v_m(\mu)\sim\frac{11}{4}\sqrt\kappa/m$ as $\kappa\to\infty$ and
$m/\sqrt\kappa\to\infty$.
\end{statement}
\cref{thm:nag-spectrum} provides the result in detail.

\begingroup
\color{black}
\paragraph{Arbitrary directions with noncommuting noise.}
The result admits a matrix extension. If $\Sigma\succ0$ and
$\chi_\Sigma=\lambda_{\max}(\Sigma)/\lambda_{\min}(\Sigma)$, then
for the same NAG output
\[
 \begin{aligned}
 \sup_{c\ne0}\left|
 \frac{N\,\Var(c^\trans\widehat x_{0,m})}{c^\trans V_\star c}-1
 \right|
 \le(18\chi_\Sigma+6\sqrt{\chi_\Sigma})
                  \frac{\sqrt\kappa}{m}.
 \end{aligned}
\]
In particular, the error is
$O(\sqrt\kappa/m)$ in every direction when the noise condition
number remains bounded. \Cref{thm:nag-all-directions} gives the sharper
comparison with any positive definite matrix commuting with $H$,
and \cref{rem:nag-geometry-necessary} shows why unrestricted
noise geometry cannot admit a universal constant.
\par
\endgroup

If $H$ and $\Sigma$ commute, the same bound holds in every direction with
positive PR variance. All these thresholds concern full averages from a
cold start.
\Cref{sec:outputs-erm} compares full averaging, tail averaging, and
switching to SGD at a common budget $N'$, keeping the random state after
burn-in. Burn-in can improve the cold-start threshold, so the threshold is
not intrinsic to momentum alone, and there is no automatic ranking between
continuing accelerated tail averaging and switching to SGD.

\begingroup
\color{black}
\paragraph{What burn-in changes.}
For a deterministic zero-velocity start $x_{-1}=x_0$, the retained
average has mean
\begin{equation}
 \E\bigl[v_\mu^\trans(\widehat x_{t,m}-x_\star)\bigr]
 =\frac{v_\mu^\trans e_0}{m}
   \sum_{j=t}^{t+m-1}r_j.
 \label{eq:main-tail-mean}
\end{equation}
Full averaging uses the prefix $S_m$; tail averaging instead uses the
decaying tail of the initialization response. This explains how burn-in
can reduce the cold-start bias. However, the state at round $t$ is
random: its covariance should be propagated through the retained window
and added to the contribution of the subsequent noises, as in
\cref{eq:tail-exact}. Counting all observations gives
$N'=(t+m)b$, so calibration to the retained-window covariance
$V_\star/(mb)$ need not imply efficiency relative to the all-sample ERM,
whose covariance is $V_\star/N'$. The full covariance formulas and the
comparison with switching to SGD remain in \cref{sec:outputs-erm}.
\par
\endgroup

\section{Extensions: heavy tails and nonlinearity}
\label{sec:main-extensions}

\dd{We next examine how the calibration picture changes beyond
finite-variance additive quadratics. We treat heavy-tailed noise first and
then extend the online--offline comparison to smooth strongly convex
nonlinear objectives.}
 
\subsection{Heavy tails}
\label{sec:main-heavy-tails}
In the quadratic model, the ERM error is 
$-H^{-1}\bar\eta^{(b)}_{t,m}$, so every deviation bound for the empirical
noise mean applies verbatim to the offline estimator, and by
\cref{eq:main-decomposition} the online output adds only the endpoint term
(\cref{rem:tails-sls}). Heavy tails thus act on the statistical term
identically online and offline; the two differ in the robustifications they
admit.
 
\begin{center}
\small
\begin{tabularx}{\linewidth}{@{}p{0.27\linewidth}X@{}}
\toprule
Condition on the noise & Available result\\
\midrule
Finite third moment & Uniform normal approximation
(\cref{thm:sharp-normal}).\\
Finite variance only & CLT with centering and variance
(\cref{cor:calibrated-clt}); no sub-Gaussian radius uniformly over the
class (\cref{prop:catoni-lower}).\\
Strictly stable noise & stable law and calibration criteria
(\cref{prop:stable-calibration}).\\
Stable domain of attraction & Stable limit with the correct normalizer,
under a boundary condition (\cref{prop:stable-limit}).\\
\bottomrule
\end{tabularx}
\end{center}
 
\begin{statement}[High-confidence limit]
\label{prop:main-catoni}
In the scalar model with curvature $\mu$ and centered i.i.d.\ noise of
variance $\sigma^2$, let $N=mb$. For every $N$ and $\delta\in(0,1/8]$ there
is a symmetric noise law with variance $\sigma^2$ such that
$\Pp\{|\widehat x^{\rm ERM}-x_\star|\ge\sigma/(\mu\sqrt{2N\delta})\}\ge\delta$.
For the full averages of the three policies started at $e_0=0$ with
$m\ge4B$, and every $\delta\in(0,1/16]$, there is such a law with
$\Pp\{|\bar u_{0,m}|\ge\sigma/(4\mu\sqrt{N\delta})\}\ge\delta$.
\end{statement}
 \cref{prop:catoni-lower} provides the detail.
Hence no ball of sub-Gaussian radius centered at the ERM or at the averaged
output is valid uniformly over the finite-variance class. A one-pass
estimator that clips full gradients at a fixed query does not have this
obstruction: with known $H$ and moment bounds, its radius depends on
$\delta$ only through $\sqrt{\log(2d/\delta)}$
(\cref{prop:robust-fixed-query}). Along a moving trajectory, clipping
yields a CLT under checkable conditions on the clipping bias, the boundary,
and the conditional covariance (\cref{thm:clipped-transfer}); clipping,
however, leaves the stochastically linear class, and its offline
counterpart is a Huber-type $M$- or $Z$-estimator rather than a clipped
empirical risk (\cref{rem:clipping-sls}).
 
\begin{statement}[Calibration under strictly stable noise]
\label{prop:main-stable}
In the scalar model, let the noises be i.i.d.\ symmetric $\alpha_0$-stable
with $1<\alpha_0<2$ and scale $\sigma_{\alpha_0}$, let $Z_{\alpha_0}$ be
standard symmetric $\alpha_0$-stable, and put
$\theta_{\alpha_0}=\mu e_0/\sigma_{\alpha_0}$. For any two-step parameters
and all $m,b\ge1$,
\begin{align*}
 &\frac{\mu N^{1-1/\alpha_0}}{\sigma_{\alpha_0}}\,\bar u_{0,m}
 \overset d=h_m^{(\alpha_0)}
 +\Bigl(\frac1m\sum_{j<m}|1-r_j|^{\alpha_0}\Bigr)^{1/\alpha_0}Z_{\alpha_0},
 \\
&h_m^{(\alpha_0)}:=\theta_{\alpha_0}S_m\,b^{1-1/\alpha_0}m^{-1/\alpha_0}.
\end{align*}
For the three policies, along any sequence with $m\to\infty$, allowing
$\kappa$, $b$, and $\theta_{\alpha_0}$ to vary, the scale factor tends to one
if and only if $B/m\to0$, and the output is calibrated if and only if, in
addition, $|\theta_{\alpha_0}|B\,b^{1-1/\alpha_0}m^{-1/\alpha_0}\to0$.
\end{statement}
 \cref{prop:stable-calibration} provides the details.

\begingroup
\color{black}
For the onset comparison below, keep $\theta_{\alpha_0}\ne0$ fixed,
as in the corresponding specialization of
\cref{prop:stable-calibration}.
\par
\endgroup
Consequently, at $b=1$ the onset is $m\gg\kappa^{\alpha_0}$ for SGD and
$m\gg\kappa^{\alpha_0/2}$ for HB and NAG, and the critical batch becomes
$N^{1/\alpha_0}/(|\theta_{\alpha_0}|B)$. Heavier tails lower the onset and
raise the critical batch, because the statistical scale $N^{1/\alpha_0-1}$
decays more slowly, although the accuracy is worse. Mean-square
guarantees carry no information in this regime, whereas the law
yields calibrated intervals from stable quantiles when $\alpha_0$ and
$\sigma_{\alpha_0}$ are known.
 
\subsection{Nonlinear objectives}
\label{sec:main-nonlinear}
Let $f\in C^2$ be $\mu$-strongly convex with an $M_2$-Lipschitz Hessian, and
let the oracle be $G(x,\zeta)=\nabla f(x)+\xi(\zeta)$. For any online
procedure that queries $s_t,\dots,s_{t+m-1}$ and returns their average
$\widehat x_{t,m}$, with $\bar g$ the average of the observed gradients and
$\widehat H=\nabla^2f(\widehat x^{\rm ERM}_{t,m})$, pathwise
\[
 \bigl\|\widehat x_{t,m}-\widehat H^{-1}\bar g
        -\widehat x^{\rm ERM}_{t,m}\bigr\|
 \le\frac{M_2}{2\mu m}\sum_{k=t}^{t+m-1}
 \|s_k-\widehat x^{\rm ERM}_{t,m}\|^2
\]
(\cref{prop:trajectory-term}). For $M_2=0$ the right-hand side vanishes,
which is the quadratic bridge behind \cref{eq:main-decomposition}. Momentum
and step policies enter the gap only through the spread of the queries,
which for constant parameters is governed by the noise amplification of the
method \citep{mohammadi2021robustness}; the right-hand side is the
online--offline counterpart of the stationary bias of constant-step SGD
\citep{dieuleveut2020bridging}.
 
\begin{statement}[First-order ERM equivalence for averaged NAG]
\label{thm:main-nonlinear-nag}
Assume, globally, $\mu I\preceq\nabla^2f\preceq LI$ with $0<\mu<L$ and an
$M_2$-Lipschitz Hessian, and let the noises be i.i.d.\ centered with
covariance $\Sigma$ and $0<\sigma^2=\tr\Sigma<\infty$. Fix $f$, the noise
law, the dimension, and $x_0$. Run NAG with the parameters above from
$x_{-1}=x_0$ for $t+m$ rounds, where $ t=\Bigl\lceil\sqrt\kappa\,
   \log\Bigl(1+\frac{bL\mathcal V_0}{\sigma^2\sqrt\kappa}\Bigr)\Bigr\rceil,$ $\mathcal V_0=f(x_0)-f(x_\star)+\frac\mu2\|x_0-x_\star\|^2$.
Let $N'=(t+m)b$, let $\widehat x^{\rm ERM}_{N'}$ be the ERM on all $N'$
observations, and let $H=\nabla^2f(x_\star)$. Along any sequence with
$m\to\infty$, $m/b\to0$, and $\log b=o(m)$,
\begin{align*}
 \sqrt{N'}\bigl(\widehat x_{t,m}-\widehat x^{\rm ERM}_{N'}\bigr)
 &\longrightarrow0\quad\text{in }L^1,\\
 \sqrt{N'}\bigl(\widehat x_{t,m}-x_\star\bigr)
 &\Rightarrow\cN(0,H^{-1}\Sigma H^{-1}).
\end{align*}
\end{statement}

\begingroup
\color{black}
\paragraph{The three costs of nonlinear averaging.}
For the burn-in, $m\ge\sqrt\kappa$ and $t\le m$, the explicit
bound in \cref{eq:nag-closed-full-budget} simplifies to
\begin{equation}
\begin{aligned}
 \frac{\mu\sqrt{N'}}{\sigma}
   \E\bigl\|\widehat x_{t,m}-\widehat x^{\rm ERM}_{N'}\bigr\| \lesssim
   \underbrace{\frac{\kappa^{1/4}}{\sqrt m}}_{\text{endpoint}}
   +\underbrace{\frac{M_2\sigma\sqrt\kappa}{\mu L}
      \sqrt{\frac mb}}_{\text{nonlinearity}}
   +\underbrace{\sqrt{\frac tm}}_{\text{burn-in cost}}.
\end{aligned}
\label{eq:main-nag-three-terms}
\end{equation}
The hidden constant is universal. The endpoint contribution becomes
small once $m\gg\sqrt\kappa$, exhibiting the accelerated memory scale.
The nonlinear contribution requires $m/b\to0$ for a fixed problem.
Finally, the last term pays for the observations omitted from the
average but included in the all-sample ERM. Since the prescribed
burn-in is $O(1+\log b)$ for a fixed problem, $\log b=o(m)$ makes this
cost vanish. These three terms explain the conditions of Statement
\ref{thm:main-nonlinear-nag}; the bound with constants is retained in
the appendix.
\par
\endgroup
This is \cref{cor:nonlinear-nag-erm}; \cref{thm:nonlinear-nag-closed} gives
an explicit nonasymptotic bound on
$\E\|\widehat x_{t,m}-\widehat x^{\rm ERM}_{N'}\|$. The class includes the
regularized Bernoulli logistic loss, for which the ERM is the usual
regularized empirical minimizer (\cref{app:nag-examples}). With
multiplicative noise, the local covariance at the solution no longer
determines the covariance accumulated along the trajectory: fixed-step SGD
already acquires a covariance correction (\cref{prop:multiplicative-scalar}),
and over oracles of arbitrary multiplicative intensity the risk of the
two-step family is not bounded by local quantities
(\cref{prop:local-insufficient-new}); a second-moment recursion certifies a
burn-in for SGD, HB, and NAG (\cref{prop:ms-operator-new}).
Diminishing-gain and increasing-momentum transfers, conditional on moment
bounds, and the comparison with Richardson--Romberg extrapolation are in
\cref{sec:step-policies,sec:rr-context}.

\section{Numerical Experiments}
\label{sec:numerical-summary}

\dd{We test the predicted center and covariance transitions in a controlled
finite-population study based on Census income data. Holding the observation
budget fixed, we compare finite-sample coverage predictions with Monte Carlo
coverage as the batch size varies.}

\paragraph{Setup.}
We treat the continuous-income UCI Adult reconstruction of
\citet{ding2021retiring} as a finite population of $49{,}531$ persons with
$d=44$ encoded covariates.  A $10^4$-person pilot fits the predictor used to
construct a prediction-powered, stochastically linear ridge oracle; subsequent
calls sample records uniformly with replacement.  At the selected ridge level,
$\kappa=1.5\times10^4$.  Using the policies of \cref{sec:setup}, we evaluate
nominal-$95\%$ PR intervals over $500$ replicates at $N=2^{16}$ and
\textcolor{black}{$b\in\{2^0,\ldots,2^{12}\}$.}  Predictions use finite-sample moments, not fitted
simulation outputs.  The pilot and population geometry are oracle resources;
\cref{app:experiments-census} gives the full protocol and other ridge levels.

\begingroup
\color{black}
Throughout the linear experiments, the output is the full query-point
average $\widehat x_{0,m}$, the observation budget is $N=mb$, and
$V_\star=H^{-1}\Sigma H^{-1}$, as in \cref{sec:setup}.
For a fixed contrast $c\in\R^d$ with $c^\trans V_\star c>0$, we infer
$c^\trans x_\star$: $c=v_\mu$ in panel (a), $c=v_L$ in panel (b),
and $c$ is the coordinate vector selecting the age coefficient in panel (c).
\par
\endgroup

\begin{figure}[h]
\centering
\includegraphics[width=0.6\columnwidth]{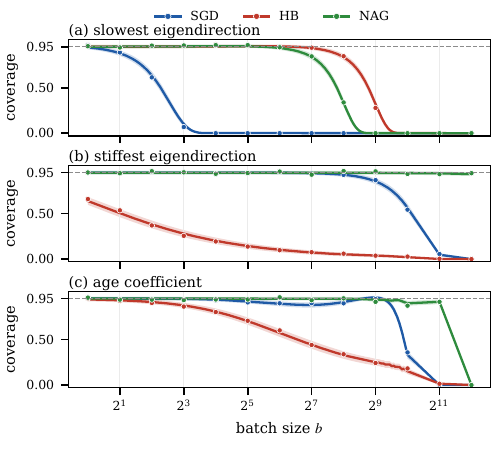}
\caption{Coverage of nominal $95\%$ intervals against the batch size
$b$ at the fixed budget $N=2^{16}$, with $\kappa=1.5\times10^4$. Panels (a) and
(b) show the slowest and the stiffest Hessian eigendirection, and panel (c) the
age coefficient. Curves are \textcolor{black}{Gaussian coverage predictions
based on finite-sample moments}, and markers are simulated
coverages. Bands show the $95\%$ range of the Monte Carlo estimate.}
\label{fig:census-coverage}
\end{figure}

\begingroup
\color{black}
\paragraph{Coverage at a Fixed Budget.}
\Cref{fig:census-coverage} shows how the coverage of the nominal $95\%$ PR
interval $ c^\trans\widehat x_{0,m}
 \pm1.96\left(\frac{c^\trans V_\star c}{N}\right)^{1/2}$
changes with the batch size when the budget $N=mb$ is fixed.
The half-width is $1.96$ times the asymptotic standard error of
$c^\trans\widehat x^{\rm ERM}_{0,m}$. Under a Gaussian approximation,
a negligible bias on this standard-error scale and a variance close
to $c^\trans V_\star c/N$ in relative terms are sufficient for coverage near $95\%$.
At a single confidence level, bias and scale errors can also compensate
(\cref{sec:confidence-level}). Larger batches leave fewer rounds to
remove the initialization bias and calibrate the covariance. A bias
causes an abrupt loss of coverage once it reaches the half-width,
and an excess covariance a gradual one.
\par
\endgroup

\paragraph{Panel (a): Slowest Eigendirection.}
{\color{black}
On the direction $c=v_\mu$, coverage is governed by the normalized
bias $h_m=\theta S_m b/\sqrt N$, with
$h_m\sim\theta B b/\sqrt N$ when $m/B\to\infty$ and $\theta\ne0$,
using the notation of Statement \ref{thm:main-calibration}.
}
Coverage falls below $90\%$ from $b=2$
for SGD, from $b=128$ for NAG, and from $b=256$ for HB. The ratios $64$ and
$128$ agree, up to the factor-two grid, with the predicted
$(\sqrt\kappa+1)/2\approx62$ and $\sqrt\kappa\approx124$. Momentum thus
enlarges the critical batch \textcolor{black}{$b_{\rm crit}(N)$ of
Statement \ref{cor:main-critical-batch}} from order
$\sqrt N/\kappa$ to order $\sqrt{N/\kappa}$\textcolor{black}{, for fixed nonzero $\theta$}.

\paragraph{Panel (b): Stiffest Eigendirection.}
{\color{black}
Here $c=v_L$, and coverage is governed by the relative variance
$v_m(L)$ after the start is forgotten. For HB,
Statement \ref{thm:main-hb-stiff} gives
$v_m(L)-1\sim\kappa^{3/2}/(8m)$ in its stated asymptotic regime.
At $b=1$, the relative variance is $v_m(L)=4.6$,
}
which implies coverage $2\Phi(1.96/\sqrt{4.6})-1\approx64\%$; the observed
coverage is $66\%$. Since the excess grows with $b=N/m$, the coverage of HB
decreases gradually. NAG, \textcolor{black}{for which
$|v_m(L)-1|\le30\sqrt\kappa/m$ by Statement \ref{thm:main-nag}},
remains nominal throughout.

\paragraph{Panel (c): Age Coefficient.}
For an ordinary regression coefficient both effects are present. HB loses
coverage first, from $b=8$, because of its covariance. SGD loses it from
$b=2^{10}$ because of its center. NAG retains coverage of at least $87\%$ up to
$b=2^{11}$.

\paragraph{Agreement with the Theory.}
In every panel the observed breakdown batch coincides with the predicted one,
and $110$ of the $117$ Monte Carlo estimates lie within their $95\%$ bands,
consistent with Monte Carlo variability. Over all $1026$ configurations,
predicted and observed coverage differ by a median of $0.7$ percentage points.
\begingroup
\color{black}
These observations are consistent with the center condition
(panel a; Statement \ref{thm:main-calibration}) and the spectral covariance
comparison (panel b; Statements \ref{thm:main-hb-stiff},\ref{thm:main-nag}).
\endgroup

\section{Conclusion}
\label{sec:discussion}

Our results give a unified finite-budget account of when averaged stochastic
methods become Polyak--Ruppert efficient. The online--offline bridge separates
common sampling noise from method-dependent endpoint dynamics, exposing
distinct scales for center, covariance, and Gaussian-shape calibration. It
explains why momentum permits larger batches and why HB and NAG can approach
the same limit through different covariance transients. The nonlinear NAG
extension shows that this ERM viewpoint persists beyond quadratics when burn-in
is charged. Acceleration therefore shortens the path to a usable statistical
limit; it does not improve that limit.

\bibliographystyle{plainnat}
\bibliography{references}

\appendix

\section{Setup Details and the Online--Offline Decomposition}
\label{app:setup-details}
\label{app:linear-proofs}

\dd{This section fixes the notation and state-space identities used
throughout the appendix. It then develops the exact online--offline
decomposition and clarifies the output, budget, and same-data ERM
conventions.}

\paragraph{Notation.}
The appendix uses the following equivalents of the main-text notation:
$y_k$ for the query point $s_k$;
$\widehat x^{\rm on}_{t,m,b}$ and $\widehat x^{\rm ERM}_{t,m,b}$ for
$\widehat x_{t,m}$ and $\widehat x^{\rm ERM}_{t,m}$;
$a_\lambda$ for the unit eigenvector $v_\lambda$;
$S_\infty$ (or $S_{\infty,{\rm alg}}$) for the memory
$B=\sum_{j\ge0}r_j$ (also written $B_{\rm alg}$); and
$N_{\rm win}=mb$, $N_{\rm tot}=(t+m)b$ for the budgets $N$ and $N'$.
The bound $U_{\rm en}$ of \cref{prop:same-sgd-new} coincides with
$U^{\rm da}_{m,b}$ in \cref{eq:sgd-da}, and $U_{\rm con}\le U^{\rm quc}_{m,b}$
is the sharper finite-window contraction bound.

\paragraph{Outputs and the same-data ERM.}
For a window $k=t,\ldots,t+m-1$, write
$\bar u_{t,m}:=m^{-1}\sum_{k=t}^{t+m-1}u_k$. The online estimator averages
the points at which gradients are evaluated:
\begin{equation}
 \widehat x^{\rm on}_{t,m,b}
 :=\frac1m\sum_{k=t}^{t+m-1}y_k
 =x_\star+\bar u_{t,m}.
 \label{eq:online-output}
\end{equation}
For any eigenvalue $\lambda$ and a chosen unit eigenvector $a_\lambda$ with
$Ha_\lambda=\lambda a_\lambda$, write
$\bar u_{t,m}^{(\lambda)}:=a_\lambda^\trans\bar u_{t,m}$ and
$\sigma_\lambda^2:=a_\lambda^\trans\Sigma a_\lambda$.
The same $N_{\rm win}=mb$ observations have average noise
$\bar\eta_{t,m}^{(b)}:=m^{-1}\sum_{k=t}^{t+m-1}\eta_{k+1}^{(b)}$ and define
\begin{equation}
\begin{aligned}
  \widehat x^{\rm ERM}_{t,m,b}
  &:=\arg\min_x\bigl\{f(x)
      +\inner{\bar\eta_{t,m}^{(b)}}{x-x_\star}\bigr\}=x_\star-H^{-1}\bar\eta_{t,m}^{(b)}.
\end{aligned}
\label{eq:window-erm}
\end{equation}
The main question is how quickly \cref{eq:online-output} approaches the
same-data benchmark \cref{eq:window-erm}. For SGD and HB, $y_k=x_k$; for NAG,
$y_k$ is the extrapolated query. Averaging NAG iterates instead changes an
endpoint term.

Let $\cF_k$ be the sigma-algebra generated by the samples observed through
round $k$. The query $u_k$ is $\cF_k$-measurable, whereas the next batch noise
$\eta_{k+1}^{(b)}$ is independent of $\cF_k$.

\paragraph{State-space representation.}
\label{appx:schur_stability}
Momentum makes \cref{eq:two-step} a second-order recursion. Substituting \cref{eq:query} into the update shows that
$e_{k+1}$ depends on $(e_k,e_{k-1})$ and the batch noise. Stack two
consecutive errors as
$\mathbf z_k=[e_k^\trans,\,e_{k-1}^\trans]^\trans\in\R^{2d}$; then
\begin{equation}
 \mathbf z_{k+1}=A\mathbf z_k+B_\alpha\eta_{k+1}^{(b)},
 \qquad u_k=\mathsf O_\gamma\mathbf z_k =(1+\gamma)e_k-\gamma e_{k-1}=u_k              \label{eq:ss}
\end{equation}
where
\[
 A:=\begin{bmatrix}
 (1+\beta)I_d-\alpha(1+\gamma)H &
 -\beta I_d+\alpha\gamma H\\
 I_d&0
 \end{bmatrix},
 \quad B_\alpha:=\begin{bmatrix}-\alpha I_d\\0\end{bmatrix},
 \quad
 \mathsf O_\gamma:=\bigl[(1+\gamma)I_d,\,-\gamma I_d\bigr],
\]
with $B_\alpha\in\R^{2d\times d}$ and
$\mathsf O_\gamma\in\R^{d\times2d}$. In particular,
$\mathsf O_0=[I_d,\,0]$ selects the current error. The powers $A^j$ propagate the
cold start, while $\mathsf O_\gamma A^jB_\alpha$ are the impulse-response
weights applied to past gradient noise.

We call the method \emph{stable} when $\rho(A)<1$, so $A^jz_0\to0$ and the
noise filter is summable. Since the blocks of $A$ are polynomials in $H$, the
dynamics decouple by Hessian eigen-direction. On an eigenvalue $\lambda$, the
characteristic polynomial is
$r^2-[1+\beta-\alpha\lambda(1+\gamma)]r
 +(\beta-\alpha\lambda\gamma)$. The Schur test gives the condition
\[
 \alpha\lambda_{\max}(H)(1+2\gamma)<2(1+\beta).
\]

The form also explains the common Polyak--Ruppert covariance.
Set the gradient-noise input identically equal to a constant vector
$g_{\rm in}$ in a stable recursion. At equilibrium,
$e_{k+1}=e_k=e_{k-1}=e$ and hence $u=e$; \cref{eq:two-step} reduces to
$He+g_{\rm in}=0$. Thus the equilibrium response is
$e=-H^{-1}g_{\rm in}$, independently of
$\alpha,\beta,\gamma$. In matrix form,
\begin{equation}
  \mathsf O_\gamma(I_{2d}-A)^{-1}B_\alpha=-H^{-1}.            \label{eq:dcgain}
\end{equation}

 \subsection{Online--Offline Decomposition and Finite-Round Laws}
\label{sec:statistical-information}

\begin{theorem}[boundary decomposition]
\label{thm:exact}
For every trajectory of \cref{eq:two-step}, every $t\ge0$, and every $m\ge1$,
\begin{equation}
 \bar u_{t,m}
 =-H^{-1}\bar\eta_{t,m}^{(b)}
 +\frac{1}{\alpha m}H^{-1}\Delta_{t,m},                          \label{eq:exact-decomp}
\end{equation}
where
\begin{align*}
   \bar\eta_{t,m}^{(b)}&=\frac1m\sum_{k=t}^{t+m-1}\eta_{k+1}^{(b)} = \frac1m\sum_{k=t}^{t+m-1}\left[\frac1b\sum_{j=1}^b\xi_{k+1,j}\right],\\
 \Delta_{t,m}&=e_t-e_{t+m}+\beta(e_{t+m-1}-e_{t-1}).  
\end{align*}
Fix $H,\alpha,\beta,\gamma$. Suppose the single-sample noises have a common i.i.d. law with mean zero and finite covariance $\Sigma$.  Let $m\to\infty$ and allow $b=b_m$ and $t=t_m$ to vary.  If
\begin{equation}
 \sqrt{b/m}\,\Delta_{t,m}\xrightarrow{\Pp}0,                    \label{eq:boundary-condition}
\end{equation}
then
\begin{equation}
 \sqrt{mb}\,\bar u_{t,m}\xrightarrow{\mathrm d} \cN(0,\Vstar),
 \qquad \Vstar=H^{-1}\Sigma H^{-1}.                           \label{eq:universal-clt}
\end{equation}
The same limit holds for the iterate average if, additionally,
$\sqrt{b/m}\,(e_{t+m-1}-e_{t-1})\xrightarrow{\Pp}0$.
\end{theorem}

\begin{proof}[Proof of \cref{thm:exact}]
Rearranging \cref{eq:two-step} gives
\[
 \alpha H u_k
 =e_k-e_{k+1}+\beta(e_k-e_{k-1})-\alpha\eta_{k+1}^{(b)}.
\]
Summing from $k=t$ to $t+m-1$ telescopes both finite differences:
\[
 \alpha H\sum_{k=t}^{t+m-1}u_k
 =e_t-e_{t+m}+\beta(e_{t+m-1}-e_{t-1})
  -\alpha\sum_{k=t}^{t+m-1}\eta_{k+1}^{(b)}.
\]
Dividing by $\alpha m$ and multiplying by $H^{-1}$ proves
\cref{eq:exact-decomp}.

For the limit, enumerate the $N=mb$ single-sample noises in the averaging
window.  The first term in \cref{eq:exact-decomp} satisfies
\[
 -\sqrt{mb}\,H^{-1}\bar\eta_{t,m}^{(b)}
 =-H^{-1}\frac1{\sqrt{mb}}
   \sum_{k=t}^{t+m-1}\sum_{j=1}^b\xi_{k+1,j}.
\]
The multivariate iid central limit theorem gives convergence to
$\cN(0,H^{-1}\Sigma H^{-1})$.  The normalized boundary equals
\[
 \frac{\sqrt{mb}}{\alpha m}H^{-1}\Delta_{t,m}
 =\frac1\alpha\sqrt{b/m}\,H^{-1}\Delta_{t,m},
\]
which is $o_{\Pp}(1)$ by \cref{eq:boundary-condition}.  Slutsky's theorem
proves \cref{eq:universal-clt}.

Finally, writing
$\bar e_{t,m}=m^{-1}\sum_{k=t}^{t+m-1}e_k$, the query definition gives
\begin{equation}
 \bar u_{t,m}-\bar e_{t,m}
 =\frac\gamma m\left(e_{t+m-1}-e_{t-1}\right).
 \label{eq:u-e-boundary}
\end{equation}
After multiplication by $\sqrt{mb}$, the additional endpoint condition in
the theorem makes the right-hand side $o_{\Pp}(1)$ and proves the assertion
for the iterate average.
\end{proof}

Condition \cref{eq:boundary-condition} also makes the role of burn-in explicit.
If $\norm{A^k}\le K\rho_A^k$ with $\rho_A<1$, then under a deterministic
start and mini-batch noise,
\begin{equation}
 \E\norm{\mathbf z_t}
 \lesssim \rho_A^t\norm{\mathbf z_0}+b^{-1/2}.                  \label{eq:state-bound}
\end{equation}
Consequently, it is sufficient that $m\to\infty$ and
$\sqrt{b/m}\,\rho_A^t\norm{\mathbf z_0}\to0$. A larger batch suppresses stochastic
state fluctuations, but it does not suppress a deterministic cold start;
without extra burn-in, aggressive batching can therefore destroy the CLT at
fixed total work.

\paragraph{The Same-Observation ERM.}
\label{sec:paired-erm}
The online and offline estimators in \cref{eq:online-output,eq:window-erm}
use the same $mb$ observations. Combining their definitions with
\cref{thm:exact} gives the pathwise identity
\begin{equation}
 \hat x^{\rm on}_{t,m,b}-\hat x^{\rm ERM}_{t,m,b}
 =\frac{H^{-1}\Delta_{t,m}}{\alpha m}.
 \label{eq:paired-erm}
\end{equation}

\begin{proposition}[A pathwise gradient bridge]
\label{prop:observable-bridge}
In the additive quadratic model, let an arbitrary online procedure query the
points $x_\star+u_k$, $t\le k\le t+m-1$, with $u_k$ being $\cF_k$-measurable,
and receive $g_k=Hu_k+\eta^{(b)}_{k+1}$. Put
$\bar u_{t,m}=m^{-1}\sum_{k=t}^{t+m-1}u_k$ and
$\bar g_{t,m}=m^{-1}\sum_{k=t}^{t+m-1}g_k$. Then, pathwise,
\begin{equation}
 \hat x^{\rm ERM}_{t,m,b}=x_\star+\bar u_{t,m}-H^{-1}\bar g_{t,m}.
 \label{eq:observable-bridge}
\end{equation}
For the two-step family, $\alpha m\,\bar g_{t,m}=\Delta_{t,m}$, so that
\cref{eq:observable-bridge} coincides with \cref{eq:paired-erm}.
\end{proposition}

If $\widehat L_{t,m}$ is the empirical risk on the window, then
$\nabla\widehat L_{t,m}(x_\star+\bar u_{t,m})
=H\bar u_{t,m}+\bar\eta_{t,m}^{(b)}=\bar g_{t,m}$. Hence
\cref{eq:observable-bridge} is one Newton step from the average.

Indeed, $\bar g_{t,m}=H\bar u_{t,m}+\bar\eta^{(b)}_{t,m}$, while
$\hat x^{\rm ERM}_{t,m,b}-x_\star=-H^{-1}\bar\eta^{(b)}_{t,m}$. For the two-step
family, summing \cref{eq:two-step} over $t\le k\le t+m-1$ gives
$e_{t+m}-e_t=\beta(e_{t+m-1}-e_{t-1})-\alpha\sum_kg_k$, that is,
$\alpha\sum_kg_k=\Delta_{t,m}$. Thus $\bar g_{t,m}$ is observable from the
queried trajectory, while the Newton correction $H^{-1}\bar g_{t,m}$ is
directly computable only with known or separately estimated curvature; see
\cref{rem:info-classes}.

\section{Finite-Budget Calibration Results}
\label{app:finite-budget-details}
\dd{This section provides the detailed finite-budget results underlying the
calibration claims in the main text. It separates initialization,
covariance, and Gaussian-shape effects and then compares their behavior
across SGD, HB, and NAG.}

\label{app:finite-budget-proofs}
\subsection{Sharp Calibration Thresholds and Critical Batching}
\label{sec:applicability}
\paragraph{Scalar Law and Risk.}
In the following scalar result, $H=\mu>0$, $e_{-1}=e_0$ is deterministic, and the noises are independent with mean zero and variance $\sigma^2>0$. Write $\theta=\mu e_0/\sigma$.

\begin{theorem}[cold-start distribution and risk]
\label{thm:cold-exact}
For any scalar two-step parameters, set
\[
 \iota=(1,1)^\trans,  ~ r_j=\mathsf O_\gamma A^j\iota, ~
 S_m=\sum_{j=0}^{m-1}r_j, ~ Q_m=\sum_{j=0}^{m-1}r_j^2.
\]
Here $r_j$ is the deterministic oracle-point response to a unit start.
For every $m,b\ge1$, without a stability assumption,
\begin{equation}
 \bar u_{0,m}
 =\frac{e_0S_m}{m}
 -\frac1{\mu mb}\sum_{s=1}^{m}\sum_{j=1}^{b}
       (1-r_{m-s})\xi_{s,j}.
 \label{eq:cold-exact-law}
\end{equation}
Consequently,
\begin{equation}
 \E\bar u_{0,m}^2
 =\frac{e_0^2S_m^2}{m^2}
  +\frac{\sigma^2}{b\mu^2m^2}(m-2S_m+Q_m).
 \label{eq:cold-exact-risk}
\end{equation}
In particular, for every signed $\theta=\mu e_0/\sigma$,
\begin{equation}
 R_{m,b}:=\frac{\E\bar u_{0,m}^2}{\sigma^2/(\mu^2mb)}
 =1+\frac{\theta^2bS_m^2-2S_m+Q_m}{m}.
 \label{eq:cold-relative-risk}
\end{equation}
\end{theorem}

\begin{proof}[Proof of \cref{thm:cold-exact}]
In the scalar model write $\mathbf z_k=(e_k,e_{k-1})^\trans$ and
$\iota=(1,1)^\trans$. Then
$\mathbf z_{k+1}=A\mathbf z_k+B_\alpha\eta_{k+1}^{(b)}$,
$u_k=\mathsf O_\gamma\mathbf z_k$, and
$r_j=\mathsf O_\gamma A^j\iota$. Directly from the matrices in
\cref{eq:ss},
\[
 (I-A)\iota=(\alpha\mu,0)^\trans,
 \qquad B_\alpha=(-\alpha,0)^\trans,
 \qquad (I-A)^{-1}B_\alpha=-\frac1\mu \iota.
\]
Superposition therefore gives $u_k=e_0r_k+\sum_{s=1}^k \mathsf O_\gamma A^{k-s}B_\alpha\eta_s^{(b)}.$
For $n\ge0$, the cumulative coefficient of one noise input is
\[
 \mathsf O_\gamma\sum_{j=0}^{n-1}A^jB_\alpha
 =\mathsf O_\gamma(I-A^n)(I-A)^{-1}B_\alpha
 =-\frac1\mu\{1-\mathsf O_\gamma A^n\iota\}
 =-\frac1\mu(1-r_n),
\]
where the empty sum at $n=0$ is zero and $r_0=1$.  Summing
$u_0,\ldots,u_{m-1}$, taking $n=m-s$, and expanding each batch average
proves \cref{eq:cold-exact-law}.  Its summands are independent and
centered, so squaring the deterministic part and summing their variances
gives \cref{eq:cold-exact-risk}; division by
$\sigma^2/(\mu^2mb)$ gives \cref{eq:cold-relative-risk}.
\end{proof}

\paragraph{Responses and Tails.}
\label{sec:batching}
Fixed-method asymptotics justify a claim non-uniformly in the condition
number: an unspecified mixing constant can grow with $\kappa$.  We therefore
replace the cold-start remainder by an identity and then allow
$\kappa$, the batch size, and the number of rounds to vary together.
For compactness in the displayed parameter formulas, set
$\varkappa:=\sqrt\kappa$.
\begin{table*}[ht!]
\caption{Standard Parameter Policies and Least-Curved-Mode Response Constants}
\label{tab:hard-mode-parameters}
\begin{center}
\small
\renewcommand{\arraystretch}{1.25}
\begin{tabular}{@{}llccc@{}}
\toprule
\textbf{Method} & \textbf{Parameters} & $\boldsymbol{\varrho}$ &
$\boldsymbol{c}$ & $\boldsymbol{S_\infty:=\sum_{j\ge0}r_j}$\\
\midrule
SGD & $\alpha=\frac{1}{L}$ & $1-\frac{1}{\kappa}$ & $0$ & $\kappa$\\
HB & $\alpha=\frac{4}{(\sqrt L+\sqrt\mu)^2}$,
     $\beta=\left(\frac{\varkappa-1}{\varkappa+1}\right)^2$
 & $\frac{\varkappa-1}{\varkappa+1}$ & $1-\varrho$ & $\varkappa$\\
NAG & $\alpha=\frac{1}{L}$,
      $\beta=\gamma=\frac{\varkappa-1}{\varkappa+1}$
 & $1-\frac{1}{\varkappa}$ & $\frac{1-\varrho}{2-\varrho}$
 & $\frac{2\varkappa^2}{\varkappa+1}$\\
\bottomrule
\end{tabular}
\end{center}
\end{table*}

\begin{proposition}[Least-curved-mode responses and tail]
\label{prop:cold-bias}
For the standard parameters in \cref{tab:hard-mode-parameters}, with
$\varkappa=\sqrt\kappa$, the least-curved-mode
response is $ r_j=(1+cj)\varrho^j,$ 
$0\le r_{j+1}\le r_j\le1.$
At $\varrho=0$, use $r_0=1$ and $r_j=0$ for $j\ge1$.
For HB, $\gamma=0$; for SGD, $\beta=\gamma=0$.
The sum and its tail are
\begin{align}
 S_\infty&=\frac{1-\beta}{\alpha\mu}
   =\frac1{1-\varrho}+\frac{c\varrho}{(1-\varrho)^2},
   \label{eq:hard-mode-sum}\\
 S_m&=S_\infty-\varrho^m\left(S_\infty+\frac{cm}{1-\varrho}\right).
 \label{eq:hard-mode-tail}
\end{align}
For all three methods, $1\le(1-\varrho)S_\infty\le2$.
\end{proposition}

\begin{proof}[Proof of \cref{prop:cold-bias}]
Put $\varkappa=\sqrt\kappa$.  Substitution in the scalar characteristic
polynomial on the $\mu$-mode gives the following responses:
\[
\begin{array}{c|c|c}
 &\varrho&c\\ \hline
\mathrm{SGD}&1-\varkappa^{-2}&0\\
\mathrm{HB}&(\varkappa-1)/(\varkappa+1)&1-\varrho\\
\mathrm{NAG}&1-\varkappa^{-1}&(1-\varrho)/(2-\varrho).
\end{array}
\]
The initial conditions $e_{-1}=e_0=1$ determine the coefficient of the
general repeated-root solution and give $r_j=(1+cj)\varrho^j$ in each case.
For $\varrho=0$ the asserted convention follows directly.  If $\varrho>0$, then
\[
 r_j-r_{j+1}=\varrho^j\{(1-\varrho)(1+cj)-c\varrho\}\ge0.
\]
Indeed, the braces equal $1-\varrho$ for SGD,
$(1-\varrho)^2(j+1)$ for HB, and
$(1-\varrho)^2(j+2)/(2-\varrho)$ for NAG.  Hence
$0\le r_{j+1}\le r_j\le r_0=1$.

The geometric-series identities yield
\begin{align*}
 \sum_{j\ge0}(1+cj)\varrho^j
 =\frac1{1-\varrho}+\frac{c\varrho}{(1-\varrho)^2}=:S_\infty,\quad
 S_\infty-S_m
 =\varrho^m\left(S_\infty+\frac{cm}{1-\varrho}\right),
\end{align*}
which proves \cref{eq:hard-mode-sum,eq:hard-mode-tail}.  Equivalently,
summing the deterministic recursion gives
$S_\infty=(1-\beta)/(\alpha\mu)$.  Substitution gives
$S_\infty=\kappa$, $\sqrt\kappa$, and
$2\kappa/(\sqrt\kappa+1)$ for SGD, HB, and NAG.  Finally, $ (1-\varrho)S_\infty=1+\frac{c\varrho}{1-\varrho},$
which lies in $[1,2]$ for all three displayed choices.
\end{proof}

Since $Q_m\le S_m\le S_\infty$, the risk gives
\begin{equation}
 1+\frac{\theta^2bS_\infty^2}{m}\left(\frac{S_m}{S_\infty}\right)^2-\frac{2S_\infty}{m}
 \le R_{m,b}\le 1+\frac{\theta^2bS_\infty^2}{m},
 \qquad
 0\le1-\frac{S_m}{S_\infty}\le (1+m(1-\varrho))\varrho^m.
 \label{eq:uniform-risk-bounds}
\end{equation}

\subsubsection{Mean, Variance, and Risk}
\label{sec:moment-calibration}

\begin{theorem}[Separate mean and covariance calibration]
\label{thm:uniform-calibration}
Choose any one of the three parameter policies above, or switch between
policies along a sequence with $m\to\infty$, integer $b\ge1$, arbitrary
$\kappa\ge1$, and arbitrary signed $\theta$.  Define the standardized
mean and variance by
\begin{equation}
 h_m=\theta\sqrt{b/m}\,S_m,\qquad
 v_m=\frac{bm\mu^2}{\sigma^2}\Var(\bar u_{0,m})=1-\frac{2S_m-Q_m}{m}.
 \label{eq:theta-calibration}
\end{equation}
Then
\begin{align}
 v_m\to1
 &\quad\Longleftrightarrow\quad S_\infty/m\to0,
 \label{eq:covariance-calibration-iff}\\
 (h_m,v_m)\to(0,1)
 &\quad\Longleftrightarrow\quad
 \begin{gathered}
 S_\infty/m\to0,\\[-2pt]
 |\theta|S_\infty\sqrt{b/m}\to0.
 \end{gathered}
 \label{eq:mean-covariance-calibration-iff}
\end{align}
If $S_\infty/m\to0$, then $h_m=\theta S_\infty\sqrt{b/m}\,(1+o(1))$ and
\begin{equation}
 R_{m,b}=v_m+\frac{\theta^2bS_\infty^2}{m}
                    (S_m/S_\infty)^2.
 \label{eq:general-risk-phase}
\end{equation}
In particular, if also
$\theta^2bS_\infty^2/m\to\Lambda_{\rm risk}\in[0,\infty]$, then
$R_{m,b}\to1+\Lambda_{\rm risk}$.
\end{theorem}

\begin{proof}[Proof of \cref{thm:uniform-calibration}]
Let $ \Delta_m^{\rm cov}:=2S_m-Q_m
 =\sum_{j=0}^{m-1}r_j(2-r_j).$ Since $0\le r_j\le1$,
\begin{equation}
 S_m\le \Delta_m^{\rm cov}\le2S_m,
 \qquad 0\le S_m\le S_\infty.                  \label{eq:proof-response-bounds}
\end{equation}
We first show that $S_m/m\to0$ if and only if $S_\infty/m\to0$.  One direction
is immediate from $S_m\le S_\infty$.  Conversely, suppose along a subsequence
$S_\infty/m\ge\epsilon>0$. With $g_\varrho=1-\varrho$,
\cref{prop:cold-bias} gives
$1\le g_\varrho S_\infty\le2$, hence
$g_\varrho m\le2/\epsilon$ and $g_\varrho\to0$. For all large
$m$, $(1-g_\varrho)^m\ge\exp(-2g_\varrho m)
\ge\exp(-4/\epsilon)$. Since
$r_j\ge \varrho^j$, it follows that $ \frac{S_m}{m}\ge \varrho^{m-1}\ge\exp(-4/\epsilon),$
a contradiction. Thus, using $v_m=1-\Delta_m^{\rm cov}/m$ and
\cref{eq:proof-response-bounds},
$v_m\to1$ if and only if $S_\infty/m\to0$.

If $S_\infty/m\to0$, then
$mg_\varrho\ge m/S_\infty\to\infty$ and the tail bound gives
\[
 0\le1-\frac{S_m}{S_\infty}
 \le(1+mg_\varrho)\varrho^m
 \le(1+mg_\varrho)e^{-mg_\varrho}\longrightarrow0.
 \label{eq:proof-tail-zero}
\]
Consequently
$h_m=\theta S_\infty\sqrt{b/m}\{1+o(1)\}$, and the joint center--covariance
criterion follows.  Finally,
\[
 R_{m,b}=v_m+h_m^2
 =v_m+\frac{\theta^2bS_\infty^2}{m}(S_m/S_\infty)^2,
\]
which proves \cref{eq:general-risk-phase} and all its stated limits.
\end{proof}

The two calibration requirements must be kept separate when $\theta=0$
or $\theta\to0$.  A risk ratio approaching one alone is insufficient for
arbitrary starts: taking
\[
 \theta_m^2=\frac{2S_m-Q_m}{bS_m^2}
\]
makes $R_{m,b}=1$, including along sequences with $S_\infty/m$ bounded
away from zero.
For fixed $\theta\ne0$, the mean condition implies $S_\infty/m\to0$ and the
two conditions reduce to $\theta^2bS_\infty^2/m\to0$.

\begin{theorem}[Risk equivalence for the special normalization $\theta=1$]
\label{thm:uniform-efficiency}
Choose any one of the three parameter policies above, or switch between
them along a sequence. Let $m\to\infty$, $b=b_m\ge1$ be an integer, and
$\kappa=\kappa_m\ge1$ be arbitrary. Under $e_0=\sigma/\mu$,
\begin{equation}
 R_{m,b}\longrightarrow1
 \quad\Longleftrightarrow\quad
 \frac{bS_\infty^2}{m}\longrightarrow0.
 \label{eq:uniform-efficiency-iff}
\end{equation}
More generally, if
$bS_\infty^2/m\to\Lambda_{\rm risk}\in[0,\infty)$, then
$R_{m,b}\to1+\Lambda_{\rm risk}$. If
$bS_\infty^2/m\to\infty$, then $R_{m,b}\to\infty$.
\end{theorem}

\begin{proof}[Proof of \cref{thm:uniform-efficiency}]
For $\theta=1$,
\[
 R_{m,b}-1=\frac{bS_m^2-\Delta_m^{\rm cov}}{m}.
\]
If $bS_\infty^2/m\to0$, then, because $b\ge1$ and $S_\infty\ge1$,
$S_\infty/m\le S_\infty^2/m\to0$.  Hence \cref{eq:proof-tail-zero} and
\cref{eq:proof-response-bounds} imply $R_{m,b}\to1$.

Conversely, suppose $R_{m,b}\to1$.  If $S_\infty/m$ failed to vanish, the
argument in the preceding proof would give $S_m\ge c_{\rm sub}m$ along a
subsequence for some $c_{\rm sub}>0$. Since
$\Delta_m^{\rm cov}\le2m$ and $b\ge1$,
this would
imply $R_{m,b}-1\ge c_{\rm sub}^2m-2\to\infty$, a contradiction. Thus
$S_\infty/m\to0$, $S_m/S_\infty\to1$, and
$\Delta_m^{\rm cov}/m\to0$. Therefore
\[
 \frac{bS_\infty^2}{m}
 =\frac{bS_m^2}{m}(S_\infty/S_m)^2\longrightarrow0.
\]
If $bS_\infty^2/m\to\Lambda_{\rm risk}<\infty$, then 
$S_\infty/m\to0$ and the same calculation gives
$R_{m,b}\to1+\Lambda_{\rm risk}$. If
$bS_\infty^2/m\to\infty$, use $g_\varrho=1-\varrho$ and
$1\le g_\varrho S_\infty\le2$ to obtain
\[
 S_m\ge\sum_{j=0}^{m-1}\varrho^j
 =\frac{1-\varrho^m}{g_\varrho}
 \ge\frac{S_\infty}{2}\bigl(1-e^{-m/S_\infty}\bigr)
 \ge c_{\min}\min\{S_\infty,m\},
 \qquad c_{\min}:=\frac{1-e^{-1}}2.
\]
If $S_\infty\le m$, then
$bS_m^2/m\ge c_{\min}^2bS_\infty^2/m\to\infty$; if $S_\infty>m$,
then $bS_m^2/m\ge c_{\min}^2bm\to\infty$. Since
$\Delta_m^{\rm cov}/m\le2$, in both
cases $R_{m,b}\to\infty$ along the full sequence.
\end{proof}

\subsubsection{Batch Size and Persistent Bias}
\label{sec:critical-batching}

\begin{corollary}[The critical batch produces a shifted Gaussian]
\label{cor:shifted-clt}
Consider $\xi_{s,j}/\sigma$.  If $S_\infty/m\to0$, $\theta S_\infty\sqrt{b/m}\to\mathfrak m\in\mathbb R,$
then
\begin{equation}
 \frac{\mu\sqrt{mb}}{\sigma}\bar u_{0,m}
 \ \xrightarrow{\mathrm d}\ \cN(\mathfrak m,1).
 \label{eq:shifted-clt}
\end{equation}
For Gaussian noise, the finite-sample law is 
\begin{equation}
 \cN\left(\theta\sqrt{b/m}\,S_m,\;
                1-\frac{2S_m-Q_m}{m}\right).
 \label{eq:finite-gaussian-law}
\end{equation}
For a fixed $\delta\in(0,1)$, a $(1-\delta)$ interval for
$x_\star$, centered at the averaged oracle point with half-width
$\Phi^{-1}(1-\delta/2)\sigma/(\mu\sqrt{mb})$, has limiting coverage
\begin{equation}
 \Phi\bigl(\Phi^{-1}(1-\delta/2)-\mathfrak m\bigr)
 -\Phi\bigl(-\Phi^{-1}(1-\delta/2)-\mathfrak m\bigr).
 \label{eq:critical-coverage}
\end{equation}
\end{corollary}

\begin{proof}[Proof of \cref{cor:shifted-clt}]
By \cref{eq:cold-exact-law}, the normalized centered output is a
triangular weighted sum, with
weights $(1-r_j)/\sqrt N\in[0,N^{-1/2}]$.  Under $S_\infty/m\to0$, its variance
$v_m \rightarrow 1$ by \cref{thm:uniform-calibration}; moreover
$N=mb\to\infty$.  The Lindeberg condition follows from the common
finite-variance law and the vanishing maximal weight.  Hence the
Lindeberg--Feller theorem and
$h_m=\theta S_\infty\sqrt{b/m}\{1+o(1)\}\to\mathfrak m$ give
\cref{eq:shifted-clt}.  For Gaussian noise inputs, the weighted sum is
Gaussian at every finite $m,b$, which gives \cref{eq:finite-gaussian-law}.
Evaluating the limiting $\cN(\mathfrak m,1)$ cdf at the two endpoints of the interval proves \cref{eq:critical-coverage}.
\end{proof}

The coverage formula assumes the noise scale and curvature are
known, or consistently estimated in the normalization used. Coverage is
an even function of the signed shift $\mathfrak m$. With $\theta=1$, the
critical regime $bS_\infty^2/m\to\tau_{\rm crit}^2<\infty$ of
\cref{thm:uniform-efficiency}  implies $S_\infty/m\to0$ and
gives $\mathfrak m=\tau_{\rm crit}\ge0$.

The coverage formula assumes that the noise scale and curvature are
known, or that the estimated normalization divided by the true
normalization converges to one in probability. Set $N=mb$ and define
\[
 T_N=\frac{\mu\sqrt N}{\sigma}\bar u_{0,m},
 \qquad
 \delta_N=\E T_N
 =\theta S_m\sqrt{\frac bm}
 =\frac{\theta bS_m}{\sqrt N}.
\]
Thus $\delta_N$ is the standardized bias. Under the centered normal
approximation used in the coverage formula, if
$\delta_N\to\mathfrak m\in\mathbb R$ and the normalized variance tends
to one, then
\[
 T_N\ \xrightarrow{\mathrm d}\ \mathcal N(\mathfrak m,1).
\]
The limiting coverage of a symmetric two-sided normal confidence
interval depends only on $|\mathfrak m|$.

For the batch-size comparisons below, assume $S_\infty>0$.
Under the asymptotic estimates used in
\cref{thm:uniform-efficiency}, $S_m/S_\infty\to1$ when
$S_\infty/m\to0$. In particular, with $\theta=1$, the regime
$bS_\infty^2/m\to\tau_{\rm crit}^2<\infty$ implies
$S_\infty/m\to0$ and $\delta_N\to\tau_{\rm crit}\ge0$.

With $N=mb$, the asymptotic conditions for the normalized variance to
tend to one and for the standardized bias to vanish are, respectively,
\begin{equation}
 \frac{bS_\infty}{N}\to0,\qquad
 \frac{|\theta|S_\infty b}{\sqrt N}\to0.
 \label{eq:batch-ratio}
\end{equation}
Assume the first condition holds. For $\theta\ne0$, define
\begin{equation}
 b_{\rm crit}(N):=\frac{\sqrt N}{|\theta|S_\infty}
 \asymp
 \begin{cases}
  \displaystyle\frac{\sqrt N}{|\theta|\kappa},
    & \text{SGD},\\[6pt]
  \displaystyle\frac{\sqrt{N/\kappa}}{|\theta|},
    & \text{HB or NAG}.
 \end{cases}
 \label{eq:bcrit}
\end{equation}
Equivalently,
$b_{\rm crit}(N)=\sigma\sqrt N/(\mu|e_0|S_\infty)$.
If
\[
 \frac{b}{b_{\rm crit}(N)}\to c\in(0,\infty),
\]
then $|\delta_N|\to c$: the standardized bias does not vanish.
If, in addition, $\delta_N\to\mathfrak m$, the limiting normal
distribution has mean $\mathfrak m$, with $|\mathfrak m|=c$.
For fixed $\theta\ne0$, this limit is
$\mathfrak m=\operatorname{sgn}(\theta)c$.
Even when $\delta_N$ does not converge, $|\delta_N|\to c$ is sufficient
for the coverage of a symmetric two-sided interval to converge under
the same normal approximation.

Asymptotically correct nominal coverage requires a negligible
standardized bias, so the batch conditions are
\[
 b=o\bigl(b_{\rm crit}(N)\bigr),
 \qquad b=o(N/S_\infty).
\]
If $b_{\rm crit}(N)$ remains bounded, the first condition cannot hold
for positive integer batch sizes. When $\theta=0$, the bias is zero,
so only the variance condition $bS_\infty/N\to0$ remains.

At $b=1$ and fixed $\theta\ne0$, these conditions reduce to
$m/\kappa^2\to\infty$ for SGD and $m/\kappa\to\infty$ for the
accelerated parameter choices. For a general initial error, they are
\[
 \frac{S_\infty}{m}\to0,
 \qquad
 \frac{\theta^2bS_\infty^2}{m}\to0.
\]
Reducing $\sigma$ at fixed $e_0\ne0$ decreases the standard deviation
but increases the bias relative to that standard deviation. More
iterations are therefore needed for nominal coverage.
Both reducing the noise and increasing the batch size reduce the
variance of an individual update. At fixed observation budget $N$,
increasing the batch size also reduces the number of sequential
iterations, since $m=N/b$.

For fixed $m$ and $\theta S_m\ne0$, treating $b$ as a real variable
gives the following balance point between squared bias and variance:
\begin{equation}
 b_{\rm balance}(m)=\frac{m-2S_m+Q_m}{\theta^2S_m^2}.
 \label{eq:exact-batch-balance}
\end{equation}
It is positive when $m-2S_m+Q_m>0$. Under the same asymptotic
estimates, if $m/S_\infty\to\infty$, then
\[
 b_{\rm balance}(m)\sim\frac{m}{\theta^2S_\infty^2}.
\]
If $\theta S_m=0$, the bias is zero and there is no positive balance
point with a positive variance. If $b_{\rm balance}(m)<1$, the squared
bias already exceeds the variance at batch size one.

\begin{remark}[Critical batch sizes for inference and optimization]
\label{rem:critical-batch}
For $\theta\ne0$, the scalar model gives
\[
 \theta^{-2}
 =\frac{\sigma^2}{(\mu e_0)^2}
 =\frac{\tr\Sigma}{\|\nabla f(x_0)\|^2}.
\]
This is the simple gradient noise scale of
\citet{mccandlish2018empirical}, evaluated at the initial point.
With the definition in \cref{eq:bcrit},
\[
 b_{\rm crit}(N)=\frac{\sqrt{N\theta^{-2}}}{S_\infty}.
\]
Thus the batch scale for asymptotically correct coverage is the
geometric mean of the observation budget and the initial gradient
noise scale, divided by $S_\infty$.
For the accelerated parameter choices, $S_\infty\asymp\sqrt\kappa$,
whereas for SGD, $S_\infty\asymp\kappa$. At the same $N$ and $\theta$,
the accelerated batch scale is therefore larger by a factor of
order $\sqrt\kappa$.

In a related optimization setting, \citet{morwani2026compute}
derive spectral-radius lower bounds and interpret them as allowing
HB to maintain the compute efficiency of SGD over a larger range of
batch sizes. Their model has multiplicative noise, whereas the model
considered here has additive noise.
\end{remark}

\subsection{Gaussian Shape versus Polyak--Ruppert Calibration}

\label{sec:sharp-normal}
Gaussian shape can be accurate before the PR
covariance is accurate.  Throughout this part, the initial condition is
deterministic with zero initial velocity, and the parameters are the standard
SGD, HB, or NAG parameters of \cref{sec:batching}.  Noise samples are i.i.d.,
with variance $\sigma^2>0$ and
$\nu_3=\E|\xi/\sigma|^3<\infty$.

Write $S_\infty=S_{\infty,\rm alg}=\sum_{j\ge0}r_j$, $N=mb$, and
\begin{equation}
 Z_m=\frac{\mu\sqrt N}{\sigma}\bar u_{0,m},\qquad h_m=\E Z_m,\qquad
 w_j=1-r_j,\qquad
 v_m=\frac1m\sum_{j=0}^{m-1}w_j^2=1-\frac{2S_m-Q_m}{m}.
 \label{eq:sharp-definitions}
\end{equation}
Thus $0\le w_j\le1$, and $v_m>0$ when $m\ge2$.  The last oracle sample
has weight zero because the average ends at the query preceding that sample.

\begin{theorem}[Uniform finite-window normal approximation]
\label{thm:sharp-normal}
For all $\kappa\ge1$, $b\ge1$, and $m\ge2$, there is a universal
Berry--Esseen constant $C_{\rm BE}$ such that
\begin{align}
 \Wass\bigl(\cL(Z_m-h_m),\cN(0,v_m)\bigr)
 &\le \frac{3\nu_3}{\sqrt N},                         \label{eq:sharp-w-own}\\
 d_{\rm K}\left(\cL\left(\frac{Z_m-h_m}{\sqrt{v_m}}\right),
                    \cN(0,1)\right)
 &\le \frac{\sqrt{5}\,C_{\rm BE}\nu_3}{\sqrt N},      \label{eq:sharp-k-own}\\
 \Wass\left(\cL\left(\frac{Z_m-h_m}{\sqrt{v_m}}\right),
                    \cN(0,1)\right)
 &\le \frac{3\sqrt{5}\,\nu_3}{\sqrt N}.               \label{eq:sharp-w-standard}
\end{align}
In particular, these constants do not deteriorate when $\kappa$ grows, even if
$m/S_\infty\to0$.  If $m\ge4S_\infty$, then
\begin{align}
 \Wass\bigl(\cL(Z_m),\cN(h_m,1)\bigr)
 &\le \frac{3\nu_3}{\sqrt N}
        +2\sqrt{\frac2\pi}\frac {S_\infty}m,           \label{eq:sharp-w-pr}\\
 d_{\rm K}\bigl(\cL(Z_m),\cN(h_m,1)\bigr)
 &\le \frac{\sqrt{5}\,C_{\rm BE}\nu_3}{\sqrt N}
        +2\phi(1)\frac {S_\infty}m,                   \label{eq:sharp-k-pr}
\end{align}
where $\phi$ denotes the standard normal density.  Replacing $\cN(h_m,1)$ by
$\cN(0,1)$ adds at most $|h_m|$ in Wasserstein distance and
$|h_m|/\sqrt{2\pi}$ in Kolmogorov distance.
\end{theorem}

\begin{proof}[Proof of \cref{thm:sharp-normal}]
Put $w_j=1-r_j$.  The centered statistic is
\[
 Z_m-h_m=-\frac1{\sqrt N}\sum_{j=0}^{m-1}\sum_{i=1}^b
 w_j\,\frac{\xi_{j,i}}\sigma,
 \qquad v_m=\frac1m\sum_{j=0}^{m-1}w_j^2.
\]
Let
$d_i^{\rm imp}=w_{i+1}-w_i=r_i-r_{i+1}$. For SGD,
$d_i^{\rm imp}$ is a constant multiple
of $\varrho^i$.  For HB and NAG it is respectively a constant multiple of
$(i+1)\varrho^i$ and $(i+2)\varrho^i$. Thus, with
$r_{\rm pole}=1$ for SGD and $r_{\rm pole}=2$
for HB and NAG,
\[
 \frac{d_i^{\rm imp}}
      {\binom{i+r_{\rm pole}-1}{r_{\rm pole}-1}}
 \quad\text{is nonincreasing in }i.
\]
Write
$\bar d_i=d_i^{\rm imp}/\binom{i+r_{\rm pole}-1}{r_{\rm pole}-1}$.
Because $\bar d_i$ is
nonincreasing, its weighted average over a prefix is at least its
weighted average over the full range, with weights
$\binom{i+r_{\rm pole}-1}{r_{\rm pole}-1}$. Cross-multiplication and the hockey-stick
identity therefore imply, for $0\le j\le m-1$,
\[
 \frac{w_j}{w_{m-1}}
 \ge\frac{\binom{j+r_{\rm pole}-1}{r_{\rm pole}}}
          {\binom{m+r_{\rm pole}-2}{r_{\rm pole}}}
 \ge\left(\frac{j}{m-1}\right)^{r_{\rm pole}}.
\]
Consequently, using the bound
$\sum_{j=0}^{m-1}(j/(m-1))^{2r_{\rm pole}}
\ge m/(2r_{\rm pole}+1)$,
\begin{equation}
 \sum_{j=0}^{m-1}w_j^2
 \ge \frac{m}{2r_{\rm pole}+1}\,w_{m-1}^2
 \ge\frac m5\max_{j<m}w_j^2.                  \label{eq:proof-weight-ratio}
\end{equation}

The independent-summand Berry--Esseen inequality now gives
\[
 d_{\rm K}\!\left(\cL\!\left(\frac{Z_m-h_m}{\sqrt{v_m}}\right),\cN(0,1)\right)
 \le \frac{C_{\rm BE}\nu_3}{\sqrt b}\frac{\sum_jw_j^3}{(\sum_jw_j^2)^{3/2}}
 \le\frac{\sqrt5\,C_{\rm BE}\nu_3}{\sqrt{mb}}.
\]
The Wasserstein version of the same inequality and
$\sum_jw_j^3\le(\max_jw_j)\sum_jw_j^2$, with
$\max_jw_j\le1$, yield
\[
 \Wass\bigl(\cL(Z_m-h_m),\cN(0,v_m)\bigr)
 \le\frac{3\nu_3}{\sqrt N},
\]
and standardization together with \cref{eq:proof-weight-ratio} gives
\cref{eq:sharp-w-standard}.

If $m\ge4S_\infty$, then
$0\le1-v_m=(2S_m-Q_m)/m\le2S_\infty/m\le1/2$.  Coupling centered normals by
the same standard Gaussian gives
\[
 \Wass(\cN(0,v_m),\cN(0,1))
 =\sqrt{2/\pi}\,(1-\sqrt{v_m})
 \le2\sqrt{2/\pi}\,S_\infty/m.
\]
For Kolmogorov distance, differentiation gives, for $v_m\ge1/2$,
\[
 d_{\rm K}(\cN(0,v_m),\cN(0,1))
 \le\phi(1)(1-v_m)\le2\phi(1)S_\infty/m.
\]
The triangle inequality proves \cref{eq:sharp-w-pr,eq:sharp-k-pr}.
Finally, translation by $h_m$ costs at most $|h_m|$ in Wasserstein
distance and $|h_m|/\sqrt{2\pi}$ in Kolmogorov distance.
\end{proof}

\begin{corollary}[Gaussian shape under arbitrary conditioning and parallelism]
\label{cor:calibrated-clt}
Replace the third-moment assumption by a common
standardized noise law of mean zero and variance one.  Along any sequence
with $m\ge2$, $b\ge1$, $N=mb\to\infty$, and arbitrary
$\kappa\ge1$, the three parameter policies satisfy
\[
 \frac{Z_m-h_m}{\sqrt{v_m}}\xrightarrow{\mathrm d}\cN(0,1).
\]
\end{corollary}

\begin{proof}[Proof of \cref{cor:calibrated-clt}]
The same weighted triangular array has total variance one after exact
standardization.  From \cref{eq:proof-weight-ratio}, its largest
standardized coefficient is $O(N^{-1/2})$.  A common
law therefore implies Lindeberg's condition by uniform
integrability, and the triangular-array CLT proves the result without a
third-moment assumption.
\end{proof}

\subsection{Spectral Separation of Heavy Ball and Nesterov}
\label{sec:improvement}

Now the procedure changes. Momentum is one candidate, but what has to be
chosen is the combination of dynamics, steps, averaging and burn-in. 

\subsubsection{Comparing Methods on the Least-Curved Eigendirection}
\label{sec:soft-mode-comparison}
\begingroup
We first apply the criteria of \cref{sec:applicability} to the same
full-average output, zero-velocity cold start, and additive-noise model
on the eigendirection with curvature $\mu$. The parameter choices are
those of \cref{prop:cold-bias}.The comparison follows by substituting into \cref{thm:uniform-calibration,thm:uniform-efficiency},
\[
 S_{\infty,{\rm SGD}}=\kappa,\qquad
 S_{\infty,{\rm HB}}=\sqrt\kappa,\qquad
 S_{\infty,{\rm NAG}}=\frac{2\kappa}{\sqrt\kappa+1}.
\]

For arbitrary $\theta$, centered variance calibration requires
$m\gg S_\infty$.
Within this regime, negligible initial bias additionally requires
$m\gg\theta^2bS_\infty^2$. Thus the general comparison keeps two conditions,
\[
 m\gg S_\infty,\qquad m\gg\theta^2bS_\infty^2.
\]
Together they imply relative MSE tending to one. For fixed
$\theta\ne0$ the second condition implies the first; for $\theta=1$,
\cref{thm:uniform-efficiency} also proves necessity for risk equivalence.

The three method-dependent scales are summarized in the table of
\cref{sec:main-acceleration}.

At fixed observation budget $N=mb$, both constraints must still hold:
$b\ll N/S_\infty$ for covariance and
$b\ll\sqrt N/(|\theta|S_\infty)$ for a nonzero start. The second is absent
when $\theta=0$. A batch proportional to its critical scale can preserve
a nonzero normalized bias and hence a coverage deficit
(\cref{cor:shifted-clt}).

The distributional comparison has one more component.
\Cref{thm:sharp-normal} gives the same $O(N^{-1/2})$ bound for all
three methods after centering and variance normalization,
under a bounded standardized third moment. With the PR variance as
target, the extra error is of order $S_\infty/m$, and
\cref{prop:sharp-normal-lower} shows that this order is sharp even
for Gaussian noise. Momentum changes the center and covariance
transients, without improving the sample-size exponent of this
shape bound.
These conclusions concern the $\mu$ eigendirection and the stated
full averages. The next subsection tests the spectral extrapolation;
\cref{sec:outputs-erm} then changes the averaging window and charges
all discarded rounds to the budget.
\par
\endgroup

\subsubsection{Does the Gain Extend across the Spectrum?}
\label{sec:spectral-obstruction}
The criteria of \cref{thm:uniform-calibration,thm:uniform-efficiency}
describe the soft mode under the stated parameters and start. The following
results limit the transfer of this gain to the whole spectrum.
The least-curved mode does not control every relative directional inference
error. This section gives a counterexample.

\begin{theorem}[Heavy-ball covariance inflation on the stiff mode]
\label{thm:stiff}
Let $\varkappa=\sqrt\kappa>1$ and use the optimally tuned quadratic heavy-ball
parameters from \cref{sec:batching}. On the eigenmode $\lambda=L$, define
$\varrho=(\varkappa-1)/(\varkappa+1)$. Its zero-velocity deterministic residual is
\[
 r_j^{(L)}=(-\varrho)^j\{1+(1+\varrho)j\}.
\]
If $S_m^{(L)}=\sum_{j<m}r_j^{(L)}$ and
$Q_m^{(L)}=\sum_{j<m}(r_j^{(L)})^2$, then
\begin{equation}
 \frac{bmL^2}{\sigma_L^2}\Var(\bar u_{0,m}^{(L)})
 =1+\frac{Q_m^{(L)}-2S_m^{(L)}}{m},\qquad
 S_\infty^{(L)}=\frac1\varkappa,\qquad
 Q_\infty^{(L)}=\frac{\varkappa^4+2\varkappa^2+5}{8\varkappa}.
 \label{eq:stiff-Q}
\end{equation}
In a joint limit $\varkappa\to\infty$ and $m/\varkappa\to\infty$,
\begin{equation}
 \frac{bmL^2}{\sigma_L^2}\Var(\bar u_{0,m}^{(L)})
 =1+\frac{\varkappa^3}{8m}\{1+o(1)\}.
 \label{eq:stiff-asymptotic}
\end{equation}
Thus, in this regime, relative covariance efficiency is equivalent to
$m/\kappa^{3/2}\to\infty$. If
$m/\varkappa^3\to\tau_{\rm stiff}\in(0,\infty)$ and the noise
is Gaussian, the centered statistic
$L\sqrt{bm}(\bar u_{0,m}^{(L)}-\E\bar u_{0,m}^{(L)})/\sigma_L$
has limiting law $\cN(0,1+1/(8\tau_{\rm stiff}))$.
In contrast, SGD with $\alpha=1/L$ on this same mode has relative variance
$1-1/m$ .
\end{theorem}

\begin{proof}[Proof of \cref{thm:stiff}]
On the $L$-mode of HB the characteristic polynomial has the repeated
root $-\varrho$, where $\varrho=(\varkappa-1)/(\varkappa+1)$. The zero-velocity initial conditions
give
\[
 r_j^{(L)}=(-\varrho)^j\{1+(1+\varrho)j\}.
\]
Using $\sum z^j=(1-z)^{-1}$,
$\sum jz^j=z(1-z)^{-2}$, and
$\sum j^2z^j=z(1+z)(1-z)^{-3}$ gives
\[
 S_\infty^{(L)}=\frac1\varkappa,
 \qquad
 Q_\infty^{(L)}=\frac{\varkappa^4+2\varkappa^2+5}{8\varkappa}.
\]
The scalar law yields
\[
 \frac{bmL^2}{\sigma_L^2}\Var(\bar u_{0,m}^{(L)})
 =\frac1m\sum_{j=0}^{m-1}(1-r_j^{(L)})^2
 =1+\frac{Q_m^{(L)}-2S_m^{(L)}}{m}.
\]
If $\varkappa\to\infty$ and $m/\varkappa\to\infty$, the geometric tails of
$S_m^{(L)}$ and $Q_m^{(L)}$ are negligible relative to $\varkappa^3$, while
$Q_\infty^{(L)}=\varkappa^3\{1+o(1)\}/8$ and
$S_\infty^{(L)}=1/\varkappa$. This proves
\cref{eq:stiff-asymptotic} and the equivalence with
$m/\kappa^{3/2}\to\infty$ in the stated regime.  Under Gaussian noise
the centered statistic is Gaussian for every $m,\varkappa$; if
$m/\varkappa^3\to\tau_{\rm stiff}$, its variance tends to
$1+1/(8\tau_{\rm stiff})$, proving the final claim.
For SGD on the $L$-mode, $r_0^{(L)}=1$ and $r_j^{(L)}=0$ for $j\ge1$, so the exact
relative variance is $(m-1)/m$.
\end{proof}

\paragraph{What This Rules Out.}
For example, take $m=\varkappa^2=\kappa$ and $b=1$. The stiff-mode HB variance
ratio diverges as $\varkappa/8$, whereas the SGD ratio tends to one.
The least-curved HB mode at this horizon already has centered variance
approaching its PR limit. Thus acceleration of the slow mode cannot imply
uniform relative directional covariance acceleration over the spectrum.
A Frobenius error weighted by $H^{-1}\Sigma H^{-1}$ can hide the stiff
mode, because its limiting variance is smaller by a factor $\kappa^2$
when the gradient-noise variances are equal. General noise-amplification
tradeoffs \citep{mohammadi2021robustness,mohammadi2024tradeoffs} and the sign
of averaged HB modal corrections \citep[Theorem 3]{samuelson2024averaging} are
known; the precise averaged-output formula and its relation to the onset
criteria above are the point of this counterexample.

\paragraph{Nesterov Covariance Calibration on Every Hessian Eigendirection}
The stiff-mode obstruction above is not shared by the NAG tuning.
The following upper bound complements it with a full-spectrum
comparison. 

\begin{theorem}[Uniform eigendirection bound for NAG]
\label{thm:nag-spectrum}
Let $\varkappa=\sqrt\kappa\ge1$, $\alpha=1/L$,
$\beta=\gamma=(\varkappa-1)/(\varkappa+1)$, and use deterministic zero-velocity
initialization. On any eigenvalue $\lambda\in[\mu,L]$, let the gradient-noise
variance be $\sigma_\lambda^2>0$ and define
\[
 v_m(\lambda)=\frac{bm\lambda^2}{\sigma_\lambda^2}
                 \Var(\bar u_{0,m}^{(\lambda)}).
\]
Then, for every $m,b\ge1$,
\begin{equation}
 \sup_{\lambda\in[\mu,L]}|v_m(\lambda)-1|
 \le \frac{30\sqrt\kappa}{m}.
 \label{eq:nag-spectrum-bound}
\end{equation}
For Gaussian noise this also bounds the centered Wasserstein distance to
the PR-normalized standard Gaussian by
$30\sqrt{2/\pi}\sqrt\kappa/m$ on every eigendirection.
The order $\sqrt\kappa/m$ is sharp: when $\kappa\to\infty$ and
$m/\sqrt\kappa\to\infty$, the slow mode obeys
\[
 1-v_m(\mu)\sim\frac{11}{4}\frac{\sqrt\kappa}{m}.
\]
Thus the worst eigendirection covariance defect for NAG is
$\Theta(\sqrt\kappa/m)$ in this regime. For optimally tuned HB the stiff
mode alone has asymptotic to $\kappa^{3/2}/(8m)$ in the same regime.
\end{theorem}

\begin{proof}[Proof of \cref{thm:nag-spectrum}]
Fix an eigenvalue and put
$\lambda_{\rm rel}=\lambda/L$ and
$\beta_{\varkappa}=(\varkappa-1)/(\varkappa+1)$. For $\varkappa=1$
the claim reduces to the preceding SGD calculation, so assume $\varkappa>1$.
The deterministic NAG query response obeys
\[
 r_{j+2}^{(\lambda)}=\mathfrak a_\lambda r_{j+1}^{(\lambda)}
 -\mathfrak d_\lambda r_j^{(\lambda)},\qquad
 \mathfrak a_\lambda=(1+\beta_{\varkappa})(1-\lambda_{\rm rel}),\quad
 \mathfrak d_\lambda=\beta_{\varkappa}(1-\lambda_{\rm rel}),
\]
with $r_0^{(\lambda)}=1$ and
$r_1^{(\lambda)}=1-(1+\beta_{\varkappa})\lambda_{\rm rel}$.
At $\lambda_{\rm rel}=\varkappa^{-2}$ the roots coincide at
$\varrho_\mu=1-\varkappa^{-1}$. For $\lambda_{\rm rel}>\varkappa^{-2}$ they are
$\rho e^{\pm i\varphi}$, where
$\rho=\sqrt{\beta_{\varkappa}(1-\lambda_{\rm rel})}\le\varrho_\mu$.
The repeated-root case follows by
continuity.
If $\lambda_{\rm rel}=1$, then $r_0^{(\lambda)}=1$,
$r_1^{(\lambda)}=-\beta_{\varkappa}$, and
$r_j^{(\lambda)}=0$ for $j\ge2$; hence
$\sum_j|r_j^{(\lambda)}|\le2$ and
$\sum_j(r_j^{(\lambda)})^2\le2$, so the uniform estimates below
hold directly.  For their derivation, assume $\lambda_{\rm rel}<1$.
Writing the solution in trigonometric form and using
$|\sin(j\varphi)|\le j|\sin\varphi|$ gives
\begin{equation}
 |r_j^{(\lambda)}|\le\rho^j(1+c_\lambda j),
 \qquad
 c_\lambda\rho=\left|r_1^{(\lambda)}-\frac{\mathfrak a_\lambda}{2}\right|
 =\frac12\left|\frac{1+\beta_{\varkappa}}{\beta_{\varkappa}}\rho^2-2\beta_{\varkappa}\right|.
 \label{eq:nag-response-bound}
\end{equation}
 On $0\le\rho\le\varrho_\mu$, the convex quadratic
\[
 \frac{1+\beta_{\varkappa}}{\beta_{\varkappa}}\rho^2-4\rho+4-2\beta_{\varkappa}
\]
has its minimum at $\varrho_\mu$, where it equals
$2(1-\beta_{\varkappa})(\beta_{\varkappa}+2)/(1+\beta_{\varkappa})\ge0$. The concave quadratic
\[
 4-4\rho+2\beta_{\varkappa}-\frac{1+\beta_{\varkappa}}{\beta_{\varkappa}}\rho^2
\]
has its minimum over the interval at an endpoint; its values at $0$ and
$\varrho_\mu$ are $4+2\beta_{\varkappa}$ and
$2(1-\beta_{\varkappa})(2-\beta_{\varkappa})/(1+\beta_{\varkappa})$, respectively. Together these
two inequalities give $ c_\lambda\rho\le2(1-\rho).$
Define
$\mathcal S_m(\lambda):=\sum_{j=0}^{m-1}r_j^{(\lambda)}$ and
$\mathcal Q_m(\lambda):=\sum_{j=0}^{m-1}(r_j^{(\lambda)})^2$, with the subscript
$\infty$ denoting the corresponding infinite sums. Summing
\cref{eq:nag-response-bound} therefore gives
\begin{align*}
 \sum_{j\ge0}|r_j^{(\lambda)}|
 \le\frac1{1-\rho}+\frac{c_\lambda\rho}{(1-\rho)^2}
 \le\frac3{1-\rho}\le3\varkappa,\quad
 \sum_{j\ge0}(r_j^{(\lambda)})^2
 \le\frac2{1-\rho^2}
 +\frac{2c_\lambda^2\rho^2(1+\rho^2)}{(1-\rho^2)^3}
 \le18\varkappa.
\end{align*}
For every finite $m$, the scalar law thus yields
\[
 |v_m(\lambda)-1|
 =\frac{|\mathcal Q_m(\lambda)-2\mathcal S_m(\lambda)|}{m}
 \le\frac{\sum_j(r_j^{(\lambda)})^2+2\sum_j|r_j^{(\lambda)}|}{m}
 \le\frac{24\varkappa}{m}\le\frac{30\varkappa}{m}.
\]
For Gaussian noise,
$\Wass(\cN(0,v_m),\cN(0,1))
 =\sqrt{2/\pi}|\sqrt{v_m}-1|
 \le\sqrt{2/\pi}|v_m-1|$, which gives the displayed bound.

On the slow mode,
$\varrho_\mu=1-\varkappa^{-1}$, $c=(\varkappa+1)^{-1}$, and
$r_j^{(\mu)}=(1+cj)\varrho_\mu^j$. Hence
\[
 \mathcal S_\infty(\mu)=\frac{2\varkappa^2}{\varkappa+1},\qquad
 \mathcal Q_\infty(\mu)=\frac1{1-\varrho_\mu^2}
 +\frac{2c\varrho_\mu^2}{(1-\varrho_\mu^2)^2}
 +\frac{c^2\varrho_\mu^2(1+\varrho_\mu^2)}
        {(1-\varrho_\mu^2)^3}.
\]
As $\varkappa\to\infty$, $\mathcal S_\infty(\mu)\sim2\varkappa$,
$\mathcal Q_\infty(\mu)\sim5\varkappa/4$, and therefore
$2\mathcal S_\infty(\mu)-\mathcal Q_\infty(\mu)
\sim11\varkappa/4$. If $m/\varkappa\to\infty$,
the geometric tails are $o(\varkappa)$, so
$1-v_m(\mu)=\{2\mathcal S_m(\mu)-\mathcal Q_m(\mu)\}/m
\sim11\varkappa/(4m)$.
\end{proof}

\paragraph{Limits of the Full-Spectrum Statement.}
The theorem controls each Hessian eigendirection. If $H$ and $\Sigma$
commute, any directional relative variance is a convex combination of these
modal ratios, so the same bound holds in every direction with positive PR
variance. Without this commutation, the matrix formula still holds,
but a relative bound for arbitrary directions needs control of the noise
covariance geometry.

\section{Beyond Quadratics: First-Order ERM Equivalence for Averaged NAG}
\label{sec:local-global}

The quadratic theory identifies a method-dependent endpoint term as the gap
between an averaged online output and the ERM computed from the same data.
For smooth strongly convex objectives with additive noise, the corresponding
gap acquires a second-order trajectory remainder.  We first derive this
remainder for an arbitrary online procedure and then control it for
tail-averaged constant-parameter NAG, including the observations spent during
burn-in in the offline comparison.

\subsection{Online--Offline Bridge Beyond Quadratics}
\label{sec:nonlinear}

We first return to the bridge of \cref{prop:observable-bridge} for a
nonquadratic model that remains stochastically linear: the oracle is
$G(x,\zeta)=\nabla f(x)+\xi(\zeta)$, where $f\in C^2$ is $\mu$-strongly convex
and its Hessian is $M_2$-Lipschitz.

\begin{proposition}[Trajectory term of the online--offline gap]
\label{prop:trajectory-term}
Let an arbitrary online procedure query $s_t,\dots,s_{t+m-1}$, with $s_k$ being
$\cF_k$-measurable, and receive $g_k=\nabla f(s_k)+\eta^{(b)}_{k+1}$. Let
$\widehat x$ minimize the empirical risk
$f(x)+\inner{\bar\eta^{(b)}_{t,m}}{x}$ of the same $N=mb$ observations, and put
$H_{\widehat x}=\nabla^2f(\widehat x)$,
$\bar s=m^{-1}\sum_{k=t}^{t+m-1}s_k$ and
$\bar g=m^{-1}\sum_{k=t}^{t+m-1}g_k$. Then, pathwise,
\begin{gather}
 H_{\widehat x}(\bar s-\widehat x)=\bar g-\bar\rho,\qquad
 \bar\rho=\frac1m\sum_{k=t}^{t+m-1}
 \bigl[\nabla f(s_k)-\nabla f(\widehat x)-H_{\widehat x}(s_k-\widehat x)\bigr].
 \label{eq:trajectory-identity}\\
 \bigl\|\bar s-H_{\widehat x}^{-1}\bar g-\widehat x\bigr\|
 \le\frac{M_2}{2\mu}\,\frac1m\sum_{k=t}^{t+m-1}\|s_k-\widehat x\|^2.
 \label{eq:trajectory-bound}
\end{gather}
\end{proposition}

Indeed, $\nabla f(\widehat x)=-\bar\eta^{(b)}_{t,m}$, hence
$\bar g=m^{-1}\sum_k\{\nabla f(s_k)-\nabla f(\widehat x)\}$, which is
\cref{eq:trajectory-identity}; each bracket in $\bar\rho$ has norm at most
$M_2\|s_k-\widehat x\|^2/2$, and $H_{\widehat x}\succeq\mu I$.

In the quadratic model $M_2=0$ and \cref{prop:trajectory-term} reduces to
\cref{prop:observable-bridge}. For a nonquadratic model, the gap between the
online average and the offline estimator contains, besides the bridge correction
$H_{\widehat x}^{-1}\bar g$, a trajectory term of second order in the spread of the
queried points around the ERM. Beyond the bridge correction, momentum and step
policies enter the gap only through this spread. For constant parameters the
spread is governed by the stationary second moments of the queries, that is, by
the noise amplification of the method \citep{mohammadi2021robustness}, and the
trajectory term is the online--offline counterpart of the stationary bias of
constant-step SGD \citep{dieuleveut2020bridging}.
 {\Cref{thm:nonlinear-nag-closed} below supplies a quantitative
control for constant-gain NAG under global curvature assumptions.}

\begingroup
\subsection{Full-Budget ERM Equivalence for Averaged NAG}
\label{sec:nonlinear-nag-closed}
\label{app:nonlinear-nag-closed}

We now close the trajectory transfer for a nonlinear accelerated method.
The output is an average of the queried points, as in
\cref{prop:trajectory-term}. Both the burn-in and the averaging window use
the classical constant NAG parameters. The Lyapunov stability mechanism is
standard; see, for example, \citet[Section 5]{laborde2020lyapunov}.
Here it is combined with the trajectory identity and a comparison
with the ERM on \emph{all} observations. 

\begin{theorem}[Nonlinear NAG: verified moments and a full-budget ERM bound]
\label{thm:nonlinear-nag-closed}
Let $f\in C^2(\R^d)$ satisfy, globally,
\[
 0<\mu<L,\qquad \mu I\preceq\nabla^2 f(x)\preceq LI,\qquad
 \|\nabla^2 f(x)-\nabla^2 f(y)\|\le M_2\|x-y\|.
\]
Let $\xi_{k+1,j}$ be iid centered vectors with covariance $\Sigma$ and
$0<\sigma^2=\tr\Sigma<\infty$. For integers $b,m\ge1$ and $t\ge0$, set
$N=(t+m)b$ and, for $0\le k<t+m$, run from a deterministic
$x_{-1}=x_0$,
\begin{align}
 \upsilon&=\sqrt{\mu/L}=\kappa^{-1/2},\qquad
 \beta=\frac{1-\upsilon}{1+\upsilon},\qquad
 \eta_{k+1}^{(b)}=\frac1b\sum_{j=1}^b\xi_{k+1,j},\notag\\
 s_k&=x_k+\beta(x_k-x_{k-1}),\qquad
 x_{k+1}=s_k-\frac1L\nabla f(s_k)-\frac1L\eta_{k+1}^{(b)}.
 \label{eq:nag-closed-recursion}
\end{align}
Retain $\bar s_{t,m}=m^{-1}\sum_{k=t}^{t+m-1}s_k$.
For  $\ell(x,\xi)=f(x)+\langle\xi,x\rangle$, define the ERM
on the same $N$ observations by
\[
 \widehat x_N^{\rm ERM}=\arg\min_x\{f(x)+\langle\bar\eta_N,x\rangle\},\qquad
 \bar\eta_N=\frac1{t+m}\sum_{k=0}^{t+m-1}\eta_{k+1}^{(b)}.
\]
Write
\begin{equation}
 \psi_k=x_k+\frac{1-\upsilon}{\upsilon}(x_k-x_{k-1}),\qquad
 \mathcal V_k=f(x_k)-f(x_\star)+\frac\mu2\|\psi_k-x_\star\|^2.
 \label{eq:nag-closed-constants}
\end{equation}
Then the natural filtration satisfies
\begin{equation}
\begin{aligned}
 \E[\mathcal V_{k+1}\mid\mathcal F_k]
 &\le(1-\upsilon)\mathcal V_k+\frac{\sigma^2}{Lb},\qquad
 \E\mathcal V_k\le(1-\upsilon)^k\mathcal V_0+\frac{\sigma^2}{bL\upsilon},\\
 \E\|s_k-x_\star\|^2
 &\le\frac2\mu\left\{(1-\upsilon)^k\mathcal V_0+\frac{\sigma^2}{bL\upsilon}\right\}.
\end{aligned}
\label{eq:nag-closed-moments}
\end{equation}
If
$(1-\upsilon)^t\mathcal V_0\le\sigma^2/(bL\upsilon)$, then
\begin{equation}
 \frac{\mu\sqrt N}{\sigma}\,\E\|\bar s_{t,m}-\widehat x_N^{\rm ERM}\|
 \le\sqrt{1+\frac tm}\Biggl\{\frac{4\sqrt2}{(1+\upsilon)\sqrt{m\upsilon}}
 +\frac{M_2\sigma}{\mu L\upsilon}\sqrt{\frac mb}\left(2+\frac3{m\upsilon}\right)\Biggr\}
 +\sqrt{\frac tm}.
 \label{eq:nag-closed-full-budget}
\end{equation}
In particular, a sufficient burn-in is
\begin{equation}
 t=\left\lceil\frac1\upsilon
       \log\left(1+\frac{bL\upsilon\,\mathcal V_0}{\sigma^2}\right)\right\rceil.
                                                               \label{eq:nag-closed-burnin}
\end{equation}
Any known upper bound on $\mathcal V_0$ may replace it in this choice.
\end{theorem}

\begin{proof}[Proof of \cref{thm:nonlinear-nag-closed}]
All norms below are Euclidean, and matrix norms are their induced operator
norms. Let $\mathcal F_k$ contain the start and the first $k$ batches.
The next batch is independent of $\mathcal F_k$, with conditional mean
zero and conditional squared norm $\sigma^2/b$.

\paragraph{Global Stability with the Accelerated Parameters.}
The definitions in \cref{eq:nag-closed-recursion,eq:nag-closed-constants}
give the identities
\begin{equation}
 s_k=\frac{x_k+\upsilon\psi_k}{1+\upsilon},\qquad
 \psi_{k+1}=(1-\upsilon)\psi_k+\upsilon s_k-\frac \upsilon\mu
                    \{\nabla f(s_k)+\eta_{k+1}^{(b)}\}.        \label{eq:nag-closed-psi}
\end{equation}
First suppress the noise in one step, and denote the resulting states
by $x^+$ and $\psi^+$. With $g_s=\nabla f(s)$ and
$\omega=(1-\upsilon)\psi+\upsilon s$,
$L$-smoothness and $\upsilon^2/\mu=1/L$ imply
\begin{align*}
 f(x^+)-f(x_\star)+\frac\mu2\|\psi^+-x_\star\|^2
 &\le f(s)-f(x_\star)-\frac{\|g_s\|^2}{2L}
       +\frac\mu2\left\|\omega-x_\star-\frac \upsilon\mu g_s\right\|^2\\
 &=f(s)-f(x_\star)+\frac\mu2\|\omega-x_\star\|^2
       -\upsilon\langle \omega-x_\star,g_s\rangle.
\end{align*}
Use strong convexity at $s$, first with $x$ and then with $x_\star$,
with weights $1-\upsilon$ and $\upsilon$. The linear terms cancel because
$(1+\upsilon)s=x+\upsilon\psi$. Expanding the squared norm of
$\omega-x_\star$ yields
\begin{equation}
 f(x^+)-f(x_\star)+\frac\mu2\|\psi^+-x_\star\|^2
 \le(1-\upsilon)\mathcal V
   -\frac{\mu(1-\upsilon)}2
       \{\|x-s\|^2+\upsilon\|\psi-s\|^2\}
 \le(1-\upsilon)\mathcal V.                         \label{eq:nag-closed-deterministic}
\end{equation}
Restoring the noise changes $x^+$ by $-\eta_{k+1}^{(b)}/L$ and
$\psi^+$ by $-\upsilon\eta_{k+1}^{(b)}/\mu$. Conditional expectation eliminates
the linear terms. Smoothness bounds the first additional contribution
by $\sigma^2/(2Lb)$; the second is 
$\upsilon^2\sigma^2/(2\mu b)=\sigma^2/(2Lb)$. This proves the first inequality
in \cref{eq:nag-closed-moments}. Iteration gives
$\E\mathcal V_k\le(1-\upsilon)^k\mathcal V_0
 +\sigma^2/(bL\upsilon)$.

Strong convexity also gives, pathwise,
\[
 \mathcal V_k\ge\frac\mu2
     \{\|x_k-x_\star\|^2+\|\psi_k-x_\star\|^2\}.
\]
Since $s_k=(x_k+\upsilon\psi_k)/(1+\upsilon)$, it follows that
\[
 \E\|s_k-x_\star\|^2
 \le\frac2\mu\left\{(1-\upsilon)^k\mathcal V_0
            +\frac{\sigma^2}{bL\upsilon}\right\}.
\]
In particular,
\begin{equation}
 \frac1m\sum_{k=t}^{t+m-1}\E\|s_k-x_\star\|^2
 \le\frac2\mu\left\{
       \frac{(1-\upsilon)^t\mathcal V_0}{m\upsilon}
       +\frac{\sigma^2}{bL\upsilon}\right\}.
                                                               \label{eq:nag-closed-spread}
\end{equation}

\paragraph{Endpoint Control and the Retained-Window ERM.}
Set $\Delta x_k=x_k-x_{k-1}$ and use the corrected state
\[
 x_k^{\rm corr}=x_k+\frac{\beta}{1-\beta}\Delta x_k
 =\frac{x_k+\psi_k}{2}.
\]
Then
\[
 \E\|x_k^{\rm corr}-x_\star\|^2
 \le\frac1\mu\left\{(1-\upsilon)^k\mathcal V_0
             +\frac{\sigma^2}{bL\upsilon}\right\},
\]
and the recursion gives
\[
 x_{k+1}^{\rm corr}-x_k^{\rm corr}
 =-\frac{\nabla f(s_k)+\eta_{k+1}^{(b)}}{L(1-\beta)}.
\]
Consequently, for the observed average gradient on the retained window,
\begin{equation}
 \bar g_{t,m}=\frac1m\sum_{k=t}^{t+m-1}
                  \{\nabla f(s_k)+\eta_{k+1}^{(b)}\}
       =\frac{2L\upsilon}{(1+\upsilon)m}
          (x_t^{\rm corr}-x_{t+m}^{\rm corr}),
       \label{eq:nag-closed-endpoint}
\end{equation}
and
\[
 \E\|\bar g_{t,m}\|
 \le\frac{2L\upsilon}{(1+\upsilon)m\sqrt\mu}
 \Biggl[\sqrt{(1-\upsilon)^t\mathcal V_0+\frac{\sigma^2}{bL\upsilon}}
 +\sqrt{(1-\upsilon)^{t+m}\mathcal V_0+\frac{\sigma^2}{bL\upsilon}}\Biggr].
\]
Let $\bar\eta_{t,m}=m^{-1}\sum_{k=t}^{t+m-1}\eta_{k+1}^{(b)}$ and
$\widehat x_{t,m}^{\rm ERM}=\arg\min_x\{f(x)+\langle\bar\eta_{t,m},x\rangle\}$.
Strong monotonicity and independence of the noises imply
\begin{equation}
 \|\widehat x_{t,m}^{\rm ERM}-x_\star\|
 \le\frac{\|\bar\eta_{t,m}\|}{\mu},\qquad
 \E\|\widehat x_{t,m}^{\rm ERM}-x_\star\|^2
 \le\frac{\sigma^2}{\mu^2bm}.                         \label{eq:nag-closed-window-erm}
\end{equation}
Apply \cref{prop:trajectory-term} with this ERM. Since its Hessian is
bounded below by $\mu I$, and
$\|s_k-\widehat x_{t,m}^{\rm ERM}\|^2
 \le2\|s_k-x_\star\|^2+2\|\widehat x_{t,m}^{\rm ERM}-x_\star\|^2$, we obtain
\begin{align}
 \E\|\bar s_{t,m}-\widehat x_{t,m}^{\rm ERM}\|
 &\le\frac{2L\upsilon}{(1+\upsilon)\mu^{3/2}m}
 \Biggl[
 \sqrt{(1-\upsilon)^t\mathcal V_0
             +\frac{\sigma^2}{bL\upsilon}}
 +\sqrt{(1-\upsilon)^{t+m}\mathcal V_0
             +\frac{\sigma^2}{bL\upsilon}}
 \Biggr]\notag\\
 &\quad+\frac{2M_2}{\mu^2}
       \left\{\frac{(1-\upsilon)^t\mathcal V_0}{m\upsilon}
              +\frac{\sigma^2}{bL\upsilon}\right\}
       +\frac{M_2\sigma^2}{\mu^3bm}.                 \label{eq:nag-closed-window-bound}
\end{align}
If $(1-\upsilon)^t\mathcal V_0\le\sigma^2/(bL\upsilon)$, each
quantity under a square root is at most
$2\sigma^2/(bL\upsilon)$. Multiply
\cref{eq:nag-closed-window-bound} by $\mu\sqrt{mb}/\sigma$ and use
$\mu=L\upsilon^2$ to get
\begin{equation}
 \frac{\mu\sqrt{mb}}{\sigma}
     \E\|\bar s_{t,m}-\widehat x_{t,m}^{\rm ERM}\|
 \le\frac{4\sqrt2}{(1+\upsilon)\sqrt{m\upsilon}}
      +\frac{M_2\sigma}{\mu L\upsilon}\sqrt{\frac mb}
          \left(2+\frac3{m\upsilon}\right).
                                                               \label{eq:nag-closed-window-scaled}
\end{equation}

\paragraph{Charging the Burn-In Observations.}
The map $u\mapsto\arg\min_x\{f(x)+\langle u,x\rangle\}$ is
$1/\mu$-Lipschitz by strong monotonicity. Using the same noise batches for
both ERMs gives the second-moment identity
\[
 \E\|\bar\eta_{t,m}-\bar\eta_N\|^2
   =\frac{\sigma^2t}{bm(t+m)}.
\]
Thus
\begin{equation}
 \E\|\widehat x_{t,m}^{\rm ERM}-\widehat x_N^{\rm ERM}\|
 \le\frac{\sigma}{\mu}
       \sqrt{\frac{t}{bm(t+m)}}.                    \label{eq:nag-closed-budget-cost}
\end{equation}
For $t=0$, both differences are zero. Combining
\cref{eq:nag-closed-window-scaled,eq:nag-closed-budget-cost} and
$N=(t+m)b$ proves \cref{eq:nag-closed-full-budget}.
Finally, $(1-\upsilon)^t\le e^{-\upsilon t}$ proves that
\cref{eq:nag-closed-burnin} ensures
$(1-\upsilon)^t\mathcal V_0\le\sigma^2/(bL\upsilon)$.
The upper bound $\mathcal V_0\le(L+\mu)\|x_0-x_\star\|^2/2$
follows from smoothness and $\psi_0=x_0$.
\end{proof}

The three terms in \cref{eq:nag-closed-full-budget} have different roles.
The boundary term becomes small for $m\gg\sqrt\kappa$ after burn-in.
The term proportional to $M_2\sigma/(\mu L\upsilon)$ pays for nonlinear
trajectory spread.
The final $\sqrt{t/m}$ compares the retained-window ERM with the ERM using
also the $tb$ burn-in observations. Thus initialization is neither free nor
an assumed warm start. The burn-in has the accelerated contraction scale
$\upsilon^{-1}=\sqrt\kappa$.
A bound $\|x_0-x_\star\|\le R_{\rm init}$ gives the bound
$\mathcal V_0\le(L+\mu)R_{\rm init}^2/2$ in \cref{eq:nag-closed-burnin}.

\begin{corollary}[Nonlinear first-order equivalence on all observations]
\label{cor:nonlinear-nag-erm}
Fix $f$, the noise law, the dimension and $x_0$ under the assumptions of
\cref{thm:nonlinear-nag-closed}. Along any sequence with
\[
 m\longrightarrow\infty,\qquad \frac mb\longrightarrow0,\qquad
 \log b=o(m),
\]
choose $t$ by \cref{eq:nag-closed-burnin}. Then, with
$H=\nabla^2 f(x_\star)$ and $N=(t+m)b$,
\begin{equation}
 \sqrt N(\bar s_{t,m}-\widehat x_N^{\rm ERM})\xrightarrow{L_1}0,
 \qquad
 \sqrt N(\bar s_{t,m}-x_\star)\xrightarrow{\mathrm d}\cN(0,H^{-1}\Sigma H^{-1}).
\label{eq:nag-closed-clt}
\end{equation}
More, there is a remainder $R_{t,m,b}$ such that
\begin{align}
 \sqrt N(\bar s_{t,m}-x_\star)
 &=-\frac{H^{-1}}{\sqrt N}\sum_{k=0}^{t+m-1}\sum_{j=1}^b\xi_{k+1,j}+R_{t,m,b},\notag\\
 \E\|R_{t,m,b}\|
 &\le\frac{\sigma}{\mu}\sqrt{1+\frac tm}
 \left\{\frac{4\sqrt2}{(1+\upsilon)\sqrt{m\upsilon}}
   +\frac{M_2\sigma}{\mu L\upsilon}\sqrt{\frac mb}\left(2+\frac3{m\upsilon}\right)\right\}
 +\frac{\sigma}{\mu}\sqrt{\frac tm}+\frac{M_2\sigma^2}{2\mu^3\sqrt N}.
 \label{eq:nag-closed-expansion}
\end{align}
\end{corollary}

\begin{proof}[Proof of \cref{cor:nonlinear-nag-erm}]
For a fixed problem the prescribed burn-in is $O(1+\log b)$.
Hence $t/m\to0$, $m\upsilon\to\infty$, and
$\{M_2\sigma/(\mu L\upsilon)\}\sqrt{m/b}\to0$. Thus the right-hand
side of \cref{eq:nag-closed-full-budget} tends to zero. The first conclusion in
\cref{eq:nag-closed-clt} follows directly from
\cref{eq:nag-closed-full-budget}.

To identify the limiting law, strong monotonicity gives
$\|\widehat x_N^{\rm ERM}-x_\star\|\le\|\bar\eta_N\|/\mu$. Taylor expansion at
$x_\star$, using $H\succeq\mu I$ and the global Lipschitz Hessian, gives
\[
 \|\widehat x_N^{\rm ERM}-x_\star+H^{-1}\bar\eta_N\|
 \le\frac{M_2}{2\mu}\|\widehat x_N^{\rm ERM}-x_\star\|^2
 \le\frac{M_2}{2\mu^3}\|\bar\eta_N\|^2.
\]
Since $\E\|\bar\eta_N\|^2=\sigma^2/N$, this proves
\cref{eq:nag-closed-expansion}. The usual iid vector CLT needs only the
finite second moment assumed here, and completes the proof.
\end{proof}

For example, $b=m^2$ and \cref{eq:nag-closed-burnin} give an admissible
sequence, with $t=O(\log m)$ and $N\asymp m^3$.
\label{app:nag-examples}%
A statistical example is the regularized Bernoulli logistic loss
\begin{equation}
 \ell(x,Y)=\frac{\lambda_{\rm reg}}{2}x^2+\log(1+e^x)-Yx,\qquad
 Y\sim\operatorname{Bernoulli}(\pi_{\rm Ber}),\qquad
 \lambda_{\rm reg}>0,\quad 0<\pi_{\rm Ber}<1.
\label{eq:nag-closed-logistic}
\end{equation}
Here $f=\E\ell$, the additive gradient noise is $\xi=\pi_{\rm Ber}-Y$, and
$\mu=\lambda_{\rm reg}$, $L=\lambda_{\rm reg}+1/4$,
$M_2=1/(6\sqrt3)$ and
$\sigma^2=\pi_{\rm Ber}(1-\pi_{\rm Ber})$ are valid global constants.
Indeed, writing
$\pi_x=(1+e^{-x})^{-1}$ gives
$f''(x)=\lambda_{\rm reg}+\pi_x(1-\pi_x)$ and
$|f'''(x)|=\pi_x(1-\pi_x)|1-2\pi_x|\le1/(6\sqrt3)$.
Thus $\widehat x_N^{\rm ERM}$ is the usual regularized empirical minimizer for the
observed Bernoulli sample, and $x_\star$ is its population counterpart.

In arbitrary dimension the class also includes, for any positive-definite
matrix $H_0$,
\begin{equation}
 f(x)=\tfrac12x^\top H_0x+\varepsilon\sum_{i=1}^d(\sin x_i-x_i),\qquad
 0<\varepsilon<\lambda_{\min}(H_0).
\label{eq:nag-closed-example}
\end{equation}
Here $x_\star=0$,
$\mu=\lambda_{\min}(H_0)-\varepsilon$,
$L=\lambda_{\max}(H_0)+\varepsilon$, and $M_2=\varepsilon$ are valid global
constants. A complete proof is in
\cref{app:nonlinear-nag-closed}.

\par
\endgroup

\section{Additional Quadratic Identities and Risk Comparisons}
\label{app:additional-sgd}

Comparable guarantees refer to the same output of the procedure,
the same initialization, the same available information and the same total budget.
The main criteria are the mean squared error and the confidence ball:
\begin{equation}
 {\cal R}_N=\E\norm{\hat x_N-x_\star}^2,\qquad
 \Pp\{\norm{\hat x_N-x_\star}\le r_N(\delta)\}\ge1-\delta.
 \label{eq:common-ball}
\end{equation}
Individual directions and ellipsoids provide additional information,
but they do not replace the comparison in the chosen common metric.

We distinguish the following targets.  They use different normalizations and therefore favor different parameter policies.

\begin{table}[h]
\caption{Four Non-Equivalent Meanings of Convergence}
\label{tab:notions}
\begin{center}
\small
\begin{tabularx}{\linewidth}{@{}P{0.18\linewidth}X X@{}}
\toprule
\textbf{Target} & \textbf{Typical statement} &
\textbf{What it resolves / what it hides} \\
\midrule
Optimization transient & $\norm{x_k-x_\star}^2\le \rho^{2k}\norm{x_0-x_\star}^2+\cdots$ & Burn-in and conditioning; not the law of the output. \\
Finite-budget risk & $\E\norm{\hat x_N-x_\star}^2\le c_0/N+\varepsilon_N$ & Valid at every $N$; constants and directional uncertainty may be hidden. \\
Distributional limit & $\sqrt N(\hat x_N-x_\star)\xrightarrow{\mathrm d}\cN(0,V)$ & local geometry and inference; says little about how large $N$ must be. \\
Offline expansion & $\widehat x^{\rm ERM}_N-x_\star=-[\nabla^2 f(x_\star)]^{-1}\bar\eta+\text{remainder}$ & Law of the offline estimator on the same data, $\bar\eta$ being the empirical noise mean; says nothing about how an online method reaches it. \\
\bottomrule
\end{tabularx}
\end{center}
\end{table}
\subsection{One SGD Output: Two Nonasymptotic Guarantees and the Risk}
\label{app:same-sgd-new}
Fix $\beta=\gamma=0$, $0<\alpha\le1/L$, a deterministic
start $x_0$ with $\|x_0-x_\star\|\le R_0$, and the output
$\widehat x=m^{-1}\sum_{k=0}^{m-1}x_k$. Let
$N=mb$, and set
\[
 \sigma^2:=\tr\Sigma,\quad F_\alpha:=I-\alpha H,
 \quad \rho_\mu:=1-\alpha\mu,
 \quad \mathsf S_{\alpha,m}:=\sum_{j=0}^{m-1}F_\alpha^j,
 \quad \omega_m(\rho_\mu):=\sum_{j=0}^{m-1}(1-\rho_\mu^j)^2.
\]

\begin{proposition}[Comparable guarantees for the full average]
\label{prop:same-sgd-new}
In the additive quadratic model, we get
\begin{align}
 \mathcal R_{m,b}&=\|\mathsf S_{\alpha,m} e_0/m\|^2+\tr W_m,
 \label{eq:exact-sgd-new}\\
 W_m&:=\frac1{bm^2}\sum_{j=0}^{m-1}
       (I-F_\alpha^j)V_\star(I-F_\alpha^j).
 \label{eq:W-sgd-new}
\end{align}
It satisfies
\begin{align}
 \mathcal R_{m,b}&\le U_{\rm en}:=
 \frac{R_0^2}{\alpha\mu m}+\frac{\alpha\sigma^2}{\mu b},
 \label{eq:energy-sgd-new}\\
 \mathcal R_{m,b}&\le U_{\rm con}:=
 R_0^2\left(\frac{1-\rho_\mu^m}{\alpha\mu m}\right)^2+
 \frac{\sigma^2\omega_m(\rho_\mu)}{b\mu^2m^2}
 \le\frac{R_0^2}{\alpha^2\mu^2m^2}+\frac{\sigma^2}{\mu^2N}.
 \label{eq:contraction-sgd-new}
\end{align}
Bound \cref{eq:energy-sgd-new} also holds for a general smooth strongly convex
$f$ with the same $\mu,L$ and conditional noise variance
at most $\sigma^2/b$. For either $U\in\{U_{\rm en},U_{\rm con}\}$, the
radius $r=\sqrt{U/\delta}$ guarantees
$\Pp\{\|\widehat x-x_\star\|\le r\}\ge1-\delta$.
\end{proposition}

\begingroup
\color{black}
\begin{proof}[Proof of \cref{prop:same-sgd-new}]
Unrolling SGD and summing its pre-update iterates gives
\[
 \widehat x-x_\star
 =\frac{\mathsf S_{\alpha,m}e_0}{m}
 -\frac1m\sum_{s=1}^m
       (I-F_\alpha^{m-s})H^{-1}\eta_s^{(b)}.
\]
The last summand is zero. Independence, centering, and
$\Cov(\eta_s^{(b)})=\Sigma/b$ prove
\cref{eq:exact-sgd-new,eq:W-sgd-new}.
For $\lambda\in[\mu,L]$,
\[
 \frac{1-(1-\alpha\lambda)^j}{\lambda}
 =\alpha\sum_{i=0}^{j-1}(1-\alpha\lambda)^i
 \le\frac{1-\rho_\mu^j}{\mu}.
\]
Thus the bias is at most
$R_0(1-\rho_\mu^m)/(\alpha\mu m)$ and the trace of each covariance
summand is at most $\sigma^2(1-\rho_\mu^j)^2/\mu^2$.
This proves both inequalities in \cref{eq:contraction-sgd-new}.

For the energy bound, the gradient of any smooth strongly convex $f$
satisfies
$\|\nabla f(x)\|^2\le L\langle x-x_\star,\nabla f(x)\rangle$
and $\langle x-x_\star,\nabla f(x)\rangle\ge\mu\|x-x_\star\|^2$.
For $0<\alpha\le1/L$ these imply
\[
 \|x-x_\star-\alpha\nabla f(x)\|^2
 \le(1-\alpha\mu)\|x-x_\star\|^2.
\]
Conditional centering of the oracle noise therefore yields, with
$p_k=\E\|e_k\|^2$,
\[
 \alpha\mu p_k\le p_k-p_{k+1}+\alpha^2\sigma^2/b.
\]
Sum over $0\le k<m$, discard $-p_m$, and use Jensen's inequality for
the average. This gives \cref{eq:energy-sgd-new}.
Finally, Markov's inequality gives the stated confidence radius for
either valid risk bound.
\end{proof}
\par
\endgroup

The energy bound $U_{\rm en}$ follows from the standard descent argument,
whereas $U_{\rm con}$ uses the linear recursion and has a noise term that
decays at the statistical rate $1/N$.

\begin{table}[htbp]
\caption{Comparable Guarantees for One Averaged SGD Output}
\label{tab:same-sgd-new}
\begin{center}
\small
\begin{tabularx}{\linewidth}{@{}P{0.17\linewidth}X X@{}}
\toprule
\textbf{Description} & \textbf{MSE and confidence radius} &
\textbf{Assumptions and required knowledge}\\
\midrule
Descent & $U_{\rm en}$; $\sqrt{U_{\rm en}/\delta}$ &
Smoothness, strong convexity, conditional variance; $\mu,L,R_0,\sigma^2$.\\
Contraction & $U_{\rm con}$; $\sqrt{U_{\rm con}/\delta}$ &
Additive quadratic model, finite variance; the same scalar knowledge.\\
moments & \cref{eq:exact-sgd-new}; MSE without an asymptotic remainder &
In addition $H,\Sigma,e_0$; for unknown $e_0$ the bound $R_0$ is used.\\
PR approximation & $\tr V_\star/N$; a Gaussian quantile only after
controlling the center and the covariance & Not a finite-time upper bound.\\
\bottomrule
\end{tabularx}
\end{center}
\end{table}

For a reproducible second-order comparison, take an orthonormal eigenbasis
$\varphi_i$ of $H$ and set $\lambda_i=\varphi_i^\trans H\varphi_i$,
$c_i=\varphi_i^\trans e_0$,
$s_i^2=\varphi_i^\trans\Sigma\varphi_i$ and $\rho_i=1-\alpha\lambda_i$.
Then, ,
\begin{equation}
 \mathcal R_{m,b}=\frac1{m^2}\sum_i
 \frac{c_i^2(1-\rho_i^m)^2}{\alpha^2\lambda_i^2}
 +\frac1{bm^2}\sum_i\frac{s_i^2}{\lambda_i^2}
 \left[m-\frac{2(1-\rho_i^m)}{\alpha\lambda_i}
 +\frac{1-\rho_i^{2m}}{\alpha\lambda_i(2-\alpha\lambda_i)}\right].
 \label{eq:modal-sgd-new}
\end{equation}
For fixed $H,\alpha$ the second-order approximation is
\begin{align}
 \mathcal R^{(2)}_{m,b}&:=\frac{\tr V_\star}{bm}
 +\frac1{m^2}\left\{\frac{\|H^{-1}e_0\|^2}{\alpha^2}
 +\frac1b\sum_i\frac{s_i^2}{\lambda_i^2}
 \left[\frac1{\alpha\lambda_i(2-\alpha\lambda_i)}
       -\frac2{\alpha\lambda_i}\right]\right\}.
 \label{eq:second-sgd-new}
\end{align}
A remainder bound, valid also for varying parameters, is
\begin{equation}
 |\mathcal R_{m,b}-\mathcal R^{(2)}_{m,b}|
 \le\frac{2\rho_\mu^m\|H^{-1}e_0\|^2}{\alpha^2m^2}
 +\frac{3\rho_\mu^m\tr V_\star}{\alpha\mu bm^2}.
 \label{eq:second-remainder-new}
\end{equation}
It is obtained by subtracting \cref{eq:second-sgd-new} from
\cref{eq:modal-sgd-new} and using $0\le\rho_i\le\rho_\mu<1$.

We begin by comparing descriptions of the same averaged SGD
output. Placing the same noise variance along the \(\mu\)-eigendirection instead of the \(L\)-eigendirection leaves the bias and both scalar bounds unchanged, but increases the variance term by a factor tending to \(\kappa^2\) for fixed \(\kappa\) as \(m\to\infty\). Thus, the risk depends on how the noise covariance is aligned with the eigendirections of \(H\).

\begin{proposition}[Same scalar parameters, different risks]
\label{prop:geometry-pair-new}
Let $0<\mu<L$, $H=\diag(\mu,L)$, $e_0=(r,0)^\trans$,
$\alpha=1/L$, and let the two covariances be
\[
 \Sigma_{\rm soft}=\sigma^2\diag(1,0),\qquad
 \Sigma_{\rm stiff}=\sigma^2\diag(0,1).
\]
The two problems share $\mu,L,\tr\Sigma,R_0=|r|$, the start, the procedure and
the budget. With $\rho_\mu=1-\mu/L$ they have the same squared bias
$\operatorname{Bias}_m^2=r^2[(1-\rho_\mu^m)/(\alpha\mu m)]^2$ and the risks
\[
 \mathcal R_{\rm soft}=\operatorname{Bias}_m^2
       +\frac{\sigma^2\omega_m(\rho_\mu)}{b\mu^2m^2},
 \qquad
 \mathcal R_{\rm stiff}=\operatorname{Bias}_m^2
       +\frac{\sigma^2(m-1)}{bL^2m^2}.
\]
For $m>1$ the ratio of the variance parts equals
$\kappa^2\omega_m(\rho_\mu)/(m-1)$ and tends to $\kappa^2$ for fixed
$\kappa$ and $m\to\infty$. The PR terms equal $\sigma^2/(\mu^2N)$ and
$\sigma^2/(L^2N)$, respectively, whereas both scalar bounds
\cref{eq:energy-sgd-new,eq:contraction-sgd-new} coincide.
\end{proposition}

\begingroup
\color{black}
\begin{proof}[Proof of \cref{prop:geometry-pair-new}]
Apply \cref{eq:exact-sgd-new,eq:W-sgd-new} in the eigenbasis of
$H=\diag(\mu,L)$. The starting vector lies on the first coordinate in
both problems, so their biases coincide. On the second coordinate,
$1-\alpha L=0$, and hence
$\sum_{j=0}^{m-1}(1-(1-\alpha L)^j)^2=m-1$.
On the first coordinate this sum is $\omega_m(\rho_\mu)$.
These identities give the two displayed risks and their variance ratio.
For fixed $\kappa$, $\omega_m(\rho_\mu)=m+O(1)$ as $m\to\infty$,
which proves the limit. Both scalar upper bounds depend only on
$\mu,L,R_0,\tr\Sigma,m,b,\alpha$, which are the same in the two models.
\end{proof}
\par
\endgroup

A nondegenerate variant is obtained by replacing the covariances with
$\sigma^2\diag(1-\varepsilon,\varepsilon)$ and
$\sigma^2\diag(\varepsilon,1-\varepsilon)$, $0<\varepsilon<1/2$;
this variant is used in the experiments.
The example proves the loss of information in the chosen scalar bounds.
Nonasymptotic matrix estimates, including the equality above, retain
this information. When the common bias dominates, the ratio of the total risks
can be close to one even when the PR covariances differ strongly.

Loss of problem-specific geometry is compatible with exact
worst-case sharpness over the class described by those scalar parameters.
The following result fixes both the algorithm and its output.

\label{sec:sharp-scalar-new}
\begin{proposition}[worst-case risk of fixed SGD]
\label{prop:sharp-scalar-new}
Let $d\ge2$, $0<\mu<L$, and let the class of additive quadratic problems
fix $\lambda_{\min}(H)=\mu$, $\lambda_{\max}(H)=L$,
$\|e_0\|\le R_0$ and $\tr\Sigma=\sigma^2$.
For fixed $\alpha\le1/L$, $m$ and $b$, the supremum of the risk of the
considered SGD average over this class equals $U_{\rm con}$ from
\cref{eq:contraction-sgd-new}. The equality holds at every finite budget and
therefore persists along any sequences of these parameters satisfying the
stated constraints.
\end{proposition}

\begingroup
\color{black}
\begin{proof}[Proof of \cref{prop:sharp-scalar-new}]
The first expression for $U_{\rm con}$ in
\cref{eq:contraction-sgd-new} is an upper bound for every member of
the stated class. It is attained by taking
$H=\diag(\mu,L,\mu,\ldots,\mu)$,
$e_0=R_0\mathbf e_1$, and
$\Sigma=\sigma^2\mathbf e_1\mathbf e_1^\trans$.
Indeed, the formula in \cref{eq:W-sgd-new} then has only the
$\mu$-mode contribution, and both its squared bias and its variance
equal the corresponding terms of $U_{\rm con}$.
The construction also covers $R_0=0$ or $\sigma^2=0$.
\end{proof}
\par
\endgroup

\subsection{Laws for a Specified Initial Point}
\label{sec:exact-cold-laws}
\begin{proposition}[matrix formula from a cold start]
\label{prop:matrix-cold}
Let $J=(I,I)^\trans$, $R_j=\mathsf O_\gamma A^jJ$, and
$\mathsf S_m=\sum_{j=0}^{m-1}R_j$.
For deterministic initialization $e_{-1}=e_0$,
\begin{equation}
 \E\bar u_{0,m}=m^{-1}\mathsf S_m e_0,\qquad
 \Cov(\bar u_{0,m})=\frac1{bm^2}\sum_{j=0}^{m-1}(I-R_j)V_\star(I-R_j)^\trans.
 \label{eq:matrix-cold-exact}
\end{equation}
For fixed stable parameters, put
\[
 \mathsf S_\infty:=\frac{1-\beta}{\alpha}H^{-1},\qquad
 \mathsf Q_\infty:=\sum_{j\ge0}R_jV_\star R_j^\trans.
\]
Then the cold-start second-order coefficient is
\[
 \Cov(\bar u_{0,m})=
 \frac{V_\star}{bm}+
 \frac{\mathsf Q_\infty-\mathsf S_\infty V_\star
       -V_\star\mathsf S_\infty^\trans}{bm^2}
 +\text{exponentially small remainder}.
\]
This coefficient differs from the stationary coefficient in
\cref{thm:cov-expansion}.
\end{proposition}

\begingroup
\color{black}
\begin{proof}[Proof of \cref{prop:matrix-cold}]
The state-space matrices satisfy
$(I-A)J=-B_\alpha H$, and $I-A$ is invertible since $H\succ0$ and
$\alpha>0$. Consequently, for every $j\ge0$,
\[
 \sum_{i=0}^{j-1}\mathsf O_\gamma A^iB_\alpha
 =-(I-R_j)H^{-1}.
\]
Unrolling the state recursion from $\mathbf z_0=Je_0$ and summing
the output therefore gives
\[
 \bar u_{0,m}=\frac{\mathsf S_m e_0}{m}
 -\frac1m\sum_{s=1}^m(I-R_{m-s})H^{-1}\eta_s^{(b)}.
\]
Taking its mean and covariance proves \cref{eq:matrix-cold-exact}.
For fixed stable parameters, summing the noiseless endpoint identity
over the infinite trajectory gives
$\mathsf S_\infty=(1-\beta)H^{-1}/\alpha$.
Expand the covariance as
\[
 \frac{mV_\star-\mathsf S_mV_\star
       -V_\star\mathsf S_m^\trans
       +\sum_{j<m}R_jV_\star R_j^\trans}{bm^2}.
\]
This proves the second-order coefficient. More explicitly, choose
$K_R<\infty$ and $q_R\in(0,1)$ with
$\|R_j\|\le K_Rq_R^j$, which is possible for fixed stable $A$.
The norm of the displayed remainder is at most
\[
 \frac{\|V_\star\|}{bm^2}
 \left\{\frac{2K_Rq_R^m}{1-q_R}
             +\frac{K_R^2q_R^{2m}}{1-q_R^2}\right\}.
\]
This also specifies the asserted exponential decay.
\end{proof}
\par
\endgroup

\begingroup
\color{black}
\begin{theorem}[NAG covariance in arbitrary directions]
\label{thm:nag-all-directions}
Use the NAG parameters and deterministic zero-velocity start of
\cref{thm:nag-spectrum}. Suppose there are a matrix
$Q_\Sigma\succ0$ commuting with $H$ and numbers
$0<\ell_\Sigma\le u_\Sigma$ such that
\[
 \ell_\Sigma Q_\Sigma\preceq\Sigma\preceq u_\Sigma Q_\Sigma,
 \qquad \chi_\Sigma:=u_\Sigma/\ell_\Sigma.
\]
Then, for every $m,b\ge1$, with $N=mb$,
\begin{equation}
 \sup_{c\ne0}
 \left|\frac{N\,\Var(c^\trans\widehat x_{0,m})}
                 {c^\trans V_\star c}-1\right|
 \le(18\chi_\Sigma+6\sqrt{\chi_\Sigma})\frac{\sqrt\kappa}{m}.
 \label{eq:nag-all-directions}
\end{equation}
Equivalently, the same bound holds for
$\|V_\star^{-1/2}\{N\Cov(\widehat x_{0,m})-V_\star\}
V_\star^{-1/2}\|_{\rm op}$.
For arbitrary $\Sigma\succ0$, one may take $Q_\Sigma=I_d$ and
$\chi_\Sigma=\lambda_{\max}(\Sigma)/\lambda_{\min}(\Sigma)$.
If $\Sigma$ commutes with $H$, one may instead take $Q_\Sigma=\Sigma$
and $\chi_\Sigma=1$.
\end{theorem}

\begin{proof}
The response matrices $R_j$ of \cref{prop:matrix-cold} are symmetric
functions of $H$. The estimates proved in \cref{thm:nag-spectrum}
hold for every eigenvalue and imply
\[
 \left\|\sum_{j<m}R_j\right\|_{\rm op}\le3\sqrt\kappa,
 \qquad
 \sum_{j<m}R_j^2\preceq18\sqrt\kappa\,I_d.
\]
These statements include $\kappa=1$, where $R_0=I_d$ and $R_j=0$
for $j\ge1$. Put $\mathsf S_m=\sum_{j<m}R_j$ and
$\mathsf T_m=\sum_{j<m}R_j\Sigma R_j$.
The covariance formula gives
\[
 H\{N\Cov(\widehat x_{0,m})-V_\star\}H
 =\frac1m\{\mathsf T_m-\mathsf S_m\Sigma-\Sigma\mathsf S_m\}.
\]
Since $Q_\Sigma$ commutes with every $R_j$,
\[
 0\preceq\mathsf T_m
 \preceq u_\Sigma Q_\Sigma^{1/2}
            \left(\sum_{j<m}R_j^2\right)Q_\Sigma^{1/2}
 \preceq18\chi_\Sigma\sqrt\kappa\,\Sigma.
\]
Also, commuting $\mathsf S_m$ with $Q_\Sigma$ gives
\[
 \begin{split}
 \|\Sigma^{-1/2}\mathsf S_m\Sigma^{1/2}\|_{\rm op}
 &\le\|\Sigma^{-1/2}Q_\Sigma^{1/2}\|_{\rm op}
       \|\mathsf S_m\|_{\rm op}
       \|Q_\Sigma^{-1/2}\Sigma^{1/2}\|_{\rm op}\\
 &\le3\sqrt{\chi_\Sigma}\sqrt\kappa.
 \end{split}
\]
Whitening the covariance difference by $\Sigma^{-1/2}$
therefore gives the bound
$(18\chi_\Sigma+6\sqrt{\chi_\Sigma})\sqrt\kappa/m$ in operator norm.
For $d_c=H^{-1}c$, its Rayleigh quotient is the relative
variance error in \cref{eq:nag-all-directions}. Taking the supremum
over $c\ne0$ proves both equivalent formulations.
\end{proof}

\begin{remark}[Why noise geometry enters an all-direction bound]
\label{rem:nag-geometry-necessary}
A universal constant independent of $\Sigma$ is impossible.
Take $H=\diag(1,4)$,
$\Sigma=\mathbf 1\mathbf 1^\trans+\varepsilon I_2$,
$c=(1,-4)^\trans$, $b=1$, $m=2$, and zero initial error.
NAG has $\alpha=1/4$, $\beta=\gamma=1/3$, and its averaged query
error is $-\xi_1/6$. Thus
\[
 c^\trans V_\star c=2\varepsilon,\qquad
 \frac{N\,\Var(c^\trans\widehat x_{0,2})}{c^\trans V_\star c}
 =\frac{9+17\varepsilon}{36\varepsilon}.
\]
This diverges as $\varepsilon\downarrow0$ while $\kappa=4$ and
$m=2$ stay fixed. The order $O(\sqrt\kappa/m)$ in arbitrary
directions is consequently uniform on classes with controlled
noise geometry, not over all positive definite noise covariances
with one geometry-independent constant.
Here $\lambda_{\max}(\Sigma)/\lambda_{\min}(\Sigma)
=(2+\varepsilon)/\varepsilon$, so even the linear order of
dependence on the noise condition number is unavoidable.
\end{remark}
\par
\endgroup
\subsection{Stationary Initialization}
\label{sec:stationary-start}

Let the stable chain start from stationarity, and let $P_1$ solve $P_1=AP_1A^\trans+B_\alpha\Sigma B_\alpha^\trans.$
For batch $b$, its state covariance is $P_1/b$.  Define
\begin{equation}
 \Delta_{\rm stat}:=\mathsf O_\gamma A(I-A)^{-2}P_1\mathsf O_\gamma^\trans
       +\mathsf O_\gamma P_1A^\trans(I-A^\trans)^{-2}\mathsf O_\gamma^\trans.
 \label{eq:delta-stat}
\end{equation}

\begin{theorem}[covariance and second-order expansion]
\label{thm:cov-expansion}
With
$\Gamma_j^{\rm lag}=b^{-1}\mathsf O_\gamma P_1(A^\trans)^j
 \mathsf O_\gamma^\trans$,
\begin{align}
 \Cov(\bar u_{0,m})
 &=\frac1{m^2}\left[m\Gamma_0^{\rm lag}
 +\sum_{j=1}^{m-1}(m-j)
   (\Gamma_j^{\rm lag}+(\Gamma_j^{\rm lag})^\trans)\right],
 \label{eq:exact-cov}\\
 &=\frac{\Vstar}{bm}-\frac{\Delta_{\rm stat}}{bm^2}+R_m.      \label{eq:cov-expansion}
\end{align}
If $\norm{A^j}\le K\rho_A^j$, then
\[
 \norm{R_m}\le
 \frac{2K\norm{\mathsf O_\gamma}^2\norm{P_1}}{b m^2}
 \frac{\rho_A^{m+1}}{(1-\rho_A)^2}.
\]
\end{theorem}

\begin{proof}[Proof of \cref{thm:cov-expansion}]
At stationarity, $\Cov(\mathbf z_k)=P_1/b$, and for $j\ge0$,
\[
 \Gamma_j^{\rm lag}=\Cov(u_0,u_j)
 =b^{-1}\mathsf O_\gamma P_1(A^\trans)^j\mathsf O_\gamma^\trans.
\]
Expanding the covariance of $m^{-1}\sum_{k=0}^{m-1}u_k$ and grouping
pairs by lag gives \cref{eq:exact-cov}.

Let $\mathcal G_j^{\rm lag}=\Gamma_j^{\rm lag}+(\Gamma_j^{\rm lag})^\trans$ and
\[
 \Sigma_{\rm LR}=\Gamma_0^{\rm lag}+\sum_{j=1}^\infty \mathcal G_j^{\rm lag}.
\]
The moving-average representation
$u_k=\sum_{r=0}^\infty \mathsf O_\gamma A^rB_\alpha\eta_{k-r}^{(b)}$ implies
\begin{align*}
 \Sigma_{\rm LR}
 =\left(\sum_{r=0}^\infty \mathsf O_\gamma A^rB_\alpha\right)\frac\Sigma b
   \left(\sum_{r=0}^\infty \mathsf O_\gamma A^rB_\alpha\right)^\trans
 =\frac1b \mathsf O_\gamma(I-A)^{-1}B_\alpha\Sigma B_\alpha^\trans
      (I-A^\trans)^{-1}\mathsf O_\gamma^\trans
 =\frac{\Vstar}{b},
\end{align*}
where the last equality uses \cref{eq:dcgain}.  Hence
\begin{align}
 m^2\Cov(\bar u_{0,m})
 =m\Sigma_{\rm LR}-\sum_{j=1}^{m-1}j\mathcal G_j^{\rm lag}
       -m\sum_{j=m}^\infty \mathcal G_j^{\rm lag}
 =m\Sigma_{\rm LR}-\sum_{j=1}^\infty j\mathcal G_j^{\rm lag}
   +\sum_{j=m}^\infty(j-m)\mathcal G_j^{\rm lag}. \label{eq:cov-sum-rewrite}
\end{align}
Since $\sum_{j\ge1}jA^j=A(I-A)^{-2}$,
$\sum_{j=1}^\infty j\mathcal G_j^{\rm lag}=\Delta_{\rm stat}/b$. Dividing
\cref{eq:cov-sum-rewrite} by $m^2$ proves \cref{eq:cov-expansion}.  Finally,
\[
 \|R_m\|
 \le \frac{2K\|\mathsf O_\gamma\|^2\|P_1\|}{bm^2}\sum_{j=m}^\infty(j-m)\rho_A^j
 =\frac{2K\|\mathsf O_\gamma\|^2\|P_1\|}{bm^2}\frac{\rho_A^{m+1}}{(1-\rho_A)^2},
\]
which is the stated remainder bound.
\end{proof}

Thus the leading covariance is method-independent while the integrated
autocovariance moment $\Delta_{\rm stat}$ is not. In a finite sample, a method
can approach $\Vstar/(bm)$ from above, from below, or anisotropically.

\subsection{Meaning and Limits of the ERM Comparison}
\label{sec:erm-interpretation}

Let $\mathfrak b_{\rm out}=\E(\hat x-x_\star)$ and
$W=\Cov(\hat x)$ for any output with
finite second moments. Its two quadratic risks are 
\begin{align}
 \E\|\hat x-x_\star\|^2&=\|\mathfrak b_{\rm out}\|^2+\tr W,
 \label{eq:parameter-risk}\\
 \E[f(\hat x)-f(x_\star)]
 &=\tfrac12\mathfrak b_{\rm out}^\trans H\mathfrak b_{\rm out}
   +\tfrac12\tr(HW).
 \label{eq:function-risk}
\end{align}
For the ERM, these are respectively
$\tr(H^{-1}\Sigma H^{-1})/N$ and
$\tr(H^{-1}\Sigma)/(2N)$. Thus function risk and parameter risk can
weight the spectrum differently. For the online estimator the
moment formulas justify any
corresponding leading-risk statement.
In a regular correctly specified likelihood model, the additional
information identity $\Sigma=H$ gives covariance $H^{-1}/N$.

\begin{remark}[offline expansions and two confidence sets]
\label{rem:wilks-wald}
With $\bar\eta=\bar\eta^{(b)}_{t,m}$, the empirical risk of the retained-window
observations is
$L_N(x)=\tfrac12\inner{x-x_\star}{H(x-x_\star)}+\inner{\bar\eta}{x-x_\star}$.
Hence the Fisher expansion $\hat x^{\rm ERM}_{t,m,b}-x_\star=-H^{-1}\bar\eta$ and the
Wilks expansion
$2N\{L_N(x_\star)-L_N(\hat x^{\rm ERM}_{t,m,b})\}=N\bar\eta^\trans H^{-1}\bar\eta$
hold without remainder. For a cutoff $c_{\rm Wilks}\ge0$, the likelihood-type set
$\{x:N(x-\hat x^{\rm ERM}_{t,m,b})^\trans H
(x-\hat x^{\rm ERM}_{t,m,b})\le c_{\rm Wilks}\}$ has the
shape of $H$, and under Gaussian noise its quantile is that of a weighted
chi-square law with weights $\spec(H^{-1/2}\Sigma H^{-1/2})$. The PR ellipsoid has
the shape of $V_\star^{-1}=H\Sigma^{-1}H$. The two shapes coincide if and only if
$\Sigma$ is proportional to $H$, which includes the information identity
$\Sigma=H$; otherwise the likelihood-type quantile has to be estimated, for
instance by the bootstrap under misspecification of \citet{spokoiny2015bootstrap}.
By \cref{prop:observable-bridge}, every such offline statement transfers to the
online output up to the bridge correction $H^{-1}\bar g_{t,m}$, measured in the
corresponding metric.  Computing this correction requires known curvature or a
separately charged estimate of it.
\end{remark}

\begin{remark}[Information classes and a curvature correction]
\label{rem:info-classes}
The ERM uses the Hessian of the empirical risk, whereas the online methods
considered here use only first-order information and, at most, the scalar
parameters $\mu$ and $L$. For any symmetric $\widehat H\succ0$, the corrected
output $x_\star+\bar u_{t,m}-\widehat H^{-1}\bar g_{t,m}$ is computable from the
queried points and the observed gradients, and by \cref{eq:observable-bridge}
\[
 \bigl(x_\star+\bar u_{t,m}-\widehat H^{-1}\bar g_{t,m}\bigr)
 -\hat x^{\rm ERM}_{t,m,b}
 =(I-\widehat H^{-1}H)\,H^{-1}\bar g_{t,m}.
\]
The correction thus reduces the online--offline gap $H^{-1}\bar g_{t,m}$ by the
relative curvature error $\|I-\widehat H^{-1}H\|_{\rm op}$; oracle calls used to
estimate $H$, as in \cref{prop:calibration-new}, belong to the budget. With known
$H$, a single query point already attains the minimax risk
\cref{eq:minimax-location-new}: the estimator $x_0-H^{-1}\bar G(x_0)$, where
$\bar G(x_0)$ averages $N$ oracle answers at $x_0$, equals the ERM of these
observations. The comparison of online methods in this paper therefore concerns
procedures that do not know $H$.
\end{remark}

\begin{proposition}[Gaussian lower bound in a specified class]
\label{prop:gaussian-minimax-new}
For known $H\succ0$, $\Sigma\succ0$ and $N$ independent calls of the additive
Gaussian oracle, including calls at adaptively chosen points,
\begin{equation}
 \inf_{\widetilde x}\sup_{x_\star\in\mathbb R^d}
 \E_{x_\star}\|\widetilde x-x_\star\|^2
 =\frac{\tr(H^{-1}\Sigma H^{-1})}{N}.
 \label{eq:minimax-location-new}
\end{equation}
The infimum is taken over all estimators using these calls.
\end{proposition}

\begingroup
\color{black}
\begin{proof}[Proof of \cref{prop:gaussian-minimax-new}]
At any possibly adaptive query $x_i$, form the observable vector
\[
 Y_i=x_i-H^{-1}G(x_i)=x_\star-H^{-1}\xi_i.
\]
The $Y_i$ are independent $\cN(x_\star,V_\star)$ observations.
Conversely, given these observations and the procedure's random seed,
one reconstructs its successive queries and answers by
$G(x_i)=H(x_i-Y_i)$. Thus the adaptive transcript contains 
this Gaussian location information; independent algorithmic
randomization cannot decrease the optimal squared-error Bayes risk.

Give $x_\star$ the prior $\cN(0,\tau^2I_d)$. The posterior covariance
after $N$ observations is
$(\tau^{-2}I_d+NV_\star^{-1})^{-1}$, so every estimator has maximal
risk at least its trace. Letting $\tau\to\infty$ gives the lower
bound $\tr(V_\star)/N$. The estimator $N^{-1}\sum_iY_i$ has 
that risk for every $x_\star$, proving equality.
\end{proof}
\par
\endgroup

The class in \cref{eq:minimax-location-new} does not restrict $\|e_0\|$. For example, for
$e_0=0$ the cold average can have risk below the PR term. For the full average at
finite $m$, the supremum of the risk over unbounded $x_\star$ is even infinite
because of the bias. In the joint small-noise regime, the formula gives a lower
bound only for the stated Gaussian location problem at each $N$.

\FloatBarrier
\section{Further Calibration and Inference Results}
\label{app:further-calibration}
\dd{This section develops additional inference guarantees beyond the main
calibration results. It covers normal approximation, confidence-set
accuracy, the effects of burn-in and switching, and data-independent
selection among specified procedures.}

\subsection{Normal Approximation and the Actual Variance}
\label{sec:normal-approximation}

\subsubsection{A Boundary-Based Normal Approximation}
For a unit direction $a$, set $\sigma_a^2=a^\trans\Vstar a$ and
$\rho_a=\E\abs{a^\trans H^{-1}\xi}^{3}$.  The triangle inequality for Wasserstein distance and the classical Berry--Esseen bound applied to the first term in \cref{eq:exact-decomp} give the following finite-$m$ statement.

\begin{corollary}[Directional Gaussian approximation]
\label{cor:be}
If $\sigma_a^2>0$ and $\rho_a<\infty$, then for $N=mb$ there is a universal
constant $C_{\rm WBE}$ such that
\begin{equation}
 \Wass\!\left(\cL\!\left(\sqrt N\,a^\trans\bar u_{t,m}\right),
              \cN(0,\sigma_a^2)\right)
 \le \frac{C_{\rm WBE}\rho_a}{\sigma_a^2\sqrt N}
 +\frac{\sqrt{b/m}}{\alpha}\,
   \E\abs{a^\trans H^{-1}\Delta_{t,m}}.                         \label{eq:be-bound}
\end{equation}
\end{corollary}

\begin{proof}[Proof of \cref{cor:be}]
Set $Y_a=a^\trans H^{-1}\xi$, so that
$\E Y_a=0$, $\Var(Y_a)=\sigma_a^2$, and $\E|Y_a|^3=\rho_a$.  The
one-dimensional Wasserstein Berry--Esseen inequality gives
\[
 \Wass\!\left(
   \cL\!\left(\frac1{\sqrt N}\sum_{i=1}^N Y_{a,i}\right),
   \cN(0,\sigma_a^2)\right)
 \le \frac{C_{\rm WBE}\rho_a}{\sigma_a^2\sqrt N}.
\]
The sign of the sum is immaterial.  By \cref{eq:exact-decomp}, the remaining
scalar perturbation is
\[
 R_{\rm BE}=\frac{\sqrt N}{\alpha m}a^\trans H^{-1}\Delta_{t,m}.
\]
For random variables $X,R_{\rm BE}$ on the same probability space,
$\Wass(\cL(X+R_{\rm BE}),\cL(X))\le\E|R_{\rm BE}|$.  The triangle inequality now yields
\cref{eq:be-bound}.
\end{proof}

Here $\Wass$ is the Wasserstein distance of order one. This general perturbation bound is valid but can be loose; \cref{sec:sharp-normal} instead uses the independent weights. The first term depends only on the total number of oracle samples.  The second depends on the number of \emph{sequential rounds} and on the optimization dynamics.  This split is the distributional analogue of ``statistical error plus optimization error,'' but now the statistical term has its covariance and a quantitative normal-approximation scale.

\begin{proposition}
\label{prop:polynomial-impulses}
Let $r_{\rm pole}\ge1$ be fixed and let
$d_i^{\rm imp}\ge0$, $d_0^{\rm imp}>0$, satisfy
\[
 i\longmapsto
 \frac{d_i^{\rm imp}}{\binom{i+r_{\rm pole}-1}{r_{\rm pole}-1}}
 \quad\text{is nonincreasing}.
\]
Set $w_j=\sum_{i=0}^{j-1}d_i^{\rm imp}$, $m\ge2$, and let
$\varepsilon_{j,k}$ be i.i.d. standardized noises.  The weighted statistic
\[
 T_{m,b}=\frac{\sum_{j=0}^{m-1}\sum_{k=1}^b w_j\varepsilon_{j,k}}
                   {\sqrt{b\sum_{j=0}^{m-1}w_j^2}}
\]
obeys
\[
 d_{\rm K}(\cL(T_{m,b}),\cN(0,1))
 \le\frac{C_{\rm BE}\sqrt{2r_{\rm pole}+1}\,\nu_3}{\sqrt{mb}},
 \qquad
 \Wass(\cL(T_{m,b}),\cN(0,1))
 \le\frac{3\sqrt{2r_{\rm pole}+1}\,\nu_3}{\sqrt{mb}}.
\]
With a common finite-variance noise law, $T_{m,b}\xrightarrow{\mathrm d}\cN(0,1)$
whenever $mb\to\infty$, even if the impulse sequence varies arbitrarily
subject to the displayed assumption.  This includes repeated-pole
impulses
$d_i^{\rm imp}=c_{\rm imp}
 \binom{i+r_{\rm pole}-1}{r_{\rm pole}-1}\rho_{\rm imp}^i$,
$c_{\rm imp}>0$, $\rho_{\rm imp}\in[0,1)$,
and their positive mixtures. SGD admits $r_{\rm pole}=1$, giving $\sqrt3$ in
place of the common $\sqrt5$ constant for the three methods.
\end{proposition}

\begingroup
\color{black}
\begin{proof}[Proof of \cref{prop:polynomial-impulses}]
Write $r=r_{\rm pole}$ and
$D_j=\sum_{i=0}^{j-1}\binom{i+r-1}{r-1}
     =\binom{j+r-1}{r}$ for $j\ge1$, with $D_0=0$.
The nonincreasing normalized impulses imply that their weighted
average over any prefix is at least their weighted average over
$0,\ldots,m-2$. Hence
\[
 \frac{w_j}{w_{m-1}}\ge\frac{D_j}{D_{m-1}}
 \ge\left(\frac{j}{m-1}\right)^r,\qquad 0\le j<m.
\]
For completeness, the second inequality follows because
$D_{j+1}/D_j=1+r/j$ and
$(1+1/j)^r\ge1+r/j$ for $r\ge1$, so $D_j/j^r$ is nonincreasing.
This argument also applies if the binomial coefficients are
interpreted through the gamma function for noninteger $r$.
Convexity of $x^{2r}$ and the trapezoidal rule give
\[
 \sum_{j=0}^{m-1}\left(\frac{j}{m-1}\right)^{2r}
 \ge\frac{m-1}{2r+1}+\frac12
 \ge\frac m{2r+1}.
\]
It follows that $\sum_jw_j^2\ge m w_{m-1}^2/(2r+1)$.
The individual standardized coefficients
$a_{j,k}=w_j/\sqrt{b\sum_\ell w_\ell^2}$ thus satisfy
\[
 \sum_{j,k}a_{j,k}^2=1,\qquad
 \max_{j,k}|a_{j,k}|\le\sqrt{\frac{2r+1}{mb}},\qquad
 \sum_{j,k}|a_{j,k}|^3\le\sqrt{\frac{2r+1}{mb}}.
\]
The independent-summand Kolmogorov and Wasserstein
Berry--Esseen inequalities yield the two claimed bounds.
With only the common finite second moment, the Lindeberg sum is at most
\[
 \E\!\left[\varepsilon^2
   \mathbf 1_{\{|\varepsilon|>
          \epsilon\sqrt{mb/(2r+1)}\}}\right]\longrightarrow0
\]
for every $\epsilon>0$, which proves the triangular-array CLT.
Repeated-pole impulses and their positive mixtures satisfy the
assumed monotonicity directly.
\end{proof}
\par
\endgroup

\begin{proposition}[The covariance-calibration term is sharp]
\label{prop:sharp-normal-lower}
Suppose the noise is Gaussian.  Define
$Q_\infty=\sum_{j\ge0}r_j^2$ and
$\Delta_\infty^{\rm cov}=2S_\infty-Q_\infty$.
For fixed $\kappa$ and $m\to\infty$,
\begin{align*}
 \Wass(\cL(Z_m-h_m),\cN(0,1))
 &=\frac{\sqrt{2/\pi}}2\frac{\Delta_\infty^{\rm cov}}{m}+o(m^{-1}),\\
 d_{\rm K}(\cL(Z_m-h_m),\cN(0,1))
 &=\frac{\phi(1)}2\frac{\Delta_\infty^{\rm cov}}{m}+o(m^{-1}).
\end{align*}
Here $S_\infty\le\Delta_\infty^{\rm cov}\le2S_\infty$. If simultaneously
$\kappa\to\infty$ and $m/S_\infty\to\infty$, the same formulas hold with
$\Delta_\infty^{\rm cov}/m$ replaced by
$c_{\rm alg}S_\infty/m$ and remainder $o(S_\infty/m)$, where
\[
 c_{\rm SGD}=\frac32,\qquad c_{\rm HB}=c_{\rm NAG}=\frac{11}{8}.
\]
Therefore the term $S_\infty/m$ cannot be replaced uniformly by
$o(S_\infty/m)$ when the
target variance is the PR variance, even for Gaussian noise and after exact
centering.
\end{proposition}

\begin{proof}
For Gaussian noise, the centered statistic is 
$\cN(0,v_m)$, where
\[
 1-v_m=\frac{A_m}{m},
 \qquad A_m:=2S_m-Q_m.
\]
For fixed $\kappa$, summability gives
$A_m\to\Delta_\infty^{\rm cov}$. Coupling two
centered normals with the same standard Gaussian and expanding the square
root yield
\[
 \Wass(\cN(0,v_m),\cN(0,1))
 =\sqrt{\frac2\pi}\,(1-\sqrt{v_m})
 =\frac{\sqrt{2/\pi}}2\frac{\Delta_\infty^{\rm cov}}{m}+o(m^{-1}).
\]
Similarly, uniformly in $x$, a first-order expansion of
$\Phi(x/\sqrt{v_m})-\Phi(x)$ gives
\[
 d_{\rm K}(\cN(0,v_m),\cN(0,1))
 =\frac{1}{2}\sup_x|x|\phi(x)\frac{\Delta_\infty^{\rm cov}}{m}+o(m^{-1})
 =\frac{\phi(1)}2\frac{\Delta_\infty^{\rm cov}}{m}+o(m^{-1}).
\]
Since $0\le r_j\le1$, one has $0\le Q_\infty\le S_\infty$, and hence
$S_\infty\le\Delta_\infty^{\rm cov}\le2S_\infty$.

In the joint regime $m/S_\infty\to\infty$, the response tails in
\cref{prop:cold-bias} give
$S_m=S_\infty+o(S_\infty)$ and
$Q_m=Q_\infty+o(S_\infty)$.
Direct summation of $(1+cj)^2\varrho^{2j}$ gives
\[
 \frac{Q_\infty}{S_\infty}\longrightarrow
 \begin{cases}
  1/2,&\mathrm{SGD},\\
  5/8,&\mathrm{HB\ or\ NAG},
 \end{cases}
\]
so $A_m/S_\infty\to3/2$ for SGD and
$A_m/S_\infty\to11/8$ for HB and NAG.
Substitution in the two normal-distance expansions proves the joint claims.
\end{proof}

\begin{proposition}[A continuum of finite-mixing Gaussian limits]
\label{prop:finite-mixing}
Fix one of the standard policies
$\mathrm{alg}\in\{\mathrm{SGD},\mathrm{HB},\mathrm{NAG}\}$ and let
$S_\infty=S_{\infty,{\rm alg}}$. Assume that the standardized noise inputs
$\xi/\sigma$
have a common iid law of mean zero and variance one across the sequence of
experiments. Allow $\kappa\to\infty$,
$m/S_\infty\to\tau_{\rm mix}\in(0,\infty)$, and any
integer batch sequence $b\ge1$.  Define the limiting response profiles by
\[
 R_{\rm SGD}(x)=e^{-x},\qquad
 R_{\rm HB}(x)=R_{\rm NAG}(x)=(1+2x)e^{-2x},\qquad x\ge0,
\]
and define
\[
 a_{\rm alg}(\tau_{\rm mix})
 =\frac1{\tau_{\rm mix}}\int_0^{\tau_{\rm mix}} R_{\rm alg}(x)\,dx,
 \qquad
 v_{\rm alg}(\tau_{\rm mix})
 =\frac1{\tau_{\rm mix}}\int_0^{\tau_{\rm mix}}
    (1-R_{\rm alg}(x))^2\,dx.
\]
Then $S_m/m\to a_{\rm alg}(\tau_{\rm mix})$,
$v_m\to v_{\rm alg}(\tau_{\rm mix})\in(0,1)$,
and
\[
 Z_m-h_m\xrightarrow{\mathrm d}\cN(0,v_{\rm alg}(\tau_{\rm mix})).
\]
If the start is local in the sense that
$\mu\sqrt N e_0/\sigma\to\vartheta\in\R$, the uncentered limit is
$\cN(\vartheta a_{\rm alg}(\tau_{\rm mix}),
       v_{\rm alg}(\tau_{\rm mix}))$.
HB and NAG have identical limiting profiles after time is measured in units
of their respective $S_{\infty,{\rm alg}}$.
\end{proposition}

\begingroup
\color{black}
\begin{proof}[Proof of \cref{prop:finite-mixing}]
Write $B=S_{\infty,{\rm alg}}$. By \cref{prop:cold-bias},
$r_j=(1+cj)\varrho^j$. As $\kappa\to\infty$, the constants in
\cref{tab:hard-mode-parameters} satisfy
\[
 (B\log\varrho,Bc)\longrightarrow
 \begin{cases}
  (-1,0),&\mathrm{SGD},\\
  (-2,2),&\mathrm{HB\ or\ NAG}.
 \end{cases}
\]
Consequently $r_{\lfloor Bx\rfloor}\to R_{\rm alg}(x)$ uniformly on
each compact interval of $[0,\infty)$. Since $m/B\to\tau_{\rm mix}$,
the corresponding Riemann sums give
\begin{align*}
 \frac{S_m}{m}
 &=\frac Bm\frac1B\sum_{j=0}^{m-1}r_j
   \longrightarrow a_{\rm alg}(\tau_{\rm mix}),\\
 v_m&=\frac Bm\frac1B\sum_{j=0}^{m-1}(1-r_j)^2
   \longrightarrow v_{\rm alg}(\tau_{\rm mix}).
\end{align*}
For each $x>0$, $0<R_{\rm alg}(x)<1$, so the limiting variance lies
strictly between zero and one.

Let $\varepsilon_{j,i}$ have the common law.
The law \cref{eq:cold-exact-law} writes the centered statistic as
\[
 Z_m-h_m=-\sum_{j=0}^{m-1}\sum_{i=1}^b
                a_{j,i}\varepsilon_{j,i},\qquad
 a_{j,i}=\frac{1-r_j}{\sqrt{mb}}.
\]
The total variance is $\sum_{j,i}a_{j,i}^2=v_m$, and
$\max_{j,i}|a_{j,i}|\le N^{-1/2}\to0$, because $B\to\infty$,
$m/B\to\tau_{\rm mix}>0$, and $N=mb\ge m$.
For every $\epsilon>0$, the Lindeberg sum is bounded by
\[
 \sum_{j,i}a_{j,i}^2
 \E\bigl[\varepsilon_{j,i}^2
       \mathbf 1_{\{|a_{j,i}\varepsilon_{j,i}|>\epsilon\}}\bigr]
 \le v_m\E\bigl[\varepsilon^2
       \mathbf 1_{\{|\varepsilon|>\epsilon\sqrt N\}}\bigr]
 \longrightarrow0.
\]
The last step uses only the common finite second moment. The
Lindeberg--Feller theorem proves the centered Gaussian limit. Finally,
\[
 h_m=\frac{\mu\sqrt N e_0}{\sigma}\frac{S_m}{m}
 \longrightarrow\vartheta a_{\rm alg}(\tau_{\rm mix}),
\]
so Slutsky's theorem gives the stated uncentered limit. The identical
HB and NAG profiles yield their asserted agreement in units of $B$.
\end{proof}
\par
\endgroup

\subsection{Confidence Sets and Coverage Accuracy}
\label{sec:confidence-sets}

\begingroup
The next guarantees use two distinct approximation targets.
The multivariate balls below concern the SGD output and matrix geometry
of \cref{prop:same-sgd-new}; the scalar interval bounds use the
least-curved-mode law of \cref{thm:sharp-normal}.
\par
\endgroup

\subsubsection{Multidimensional Confidence Balls}
\label{sec:ball-new}
\begin{proposition}[Finite-time Gaussian ball]
\label{prop:gaussian-ball-new}
Under independent Gaussian additive noise, set
$\mu_m^{\rm out}=\mathsf S_{\alpha,m}e_0/m$ and take $W_m$ from
\cref{eq:W-sgd-new}.
Then $\widehat x-x_\star\overset{\mathrm d}{=}\mu_m^{\rm out}+W_m^{1/2}Z$,
$Z\sim\mathcal N(0,I_d)$. If
\[
 \operatorname{Quant}_{\pi_{\rm cov}}(m_0,W)
 :=\inf\{r\ge0:\Pp(\|m_0+W^{1/2}Z\|\le r)\ge\pi_{\rm cov}\},
\]
then the ball of radius
$\operatorname{Quant}_{1-\delta}(\mu_m^{\rm out},W_m)$ has coverage at least
$1-\delta$ ($1-\delta$ when $W_m\succ0$).
When only the bound $\|e_0\|\le R_0$ on the start is known, the computable radius
\begin{equation}
 r_{\rm G}=R_0\|\mathsf S_{\alpha,m}/m\|_{\rm op}
       +\operatorname{Quant}_{1-\delta}(0,W_m)
 \label{eq:gaussian-ball-new}
\end{equation}
guarantees the same lower coverage. Without computing the quantile one can take
\begin{equation}
 r_{\rm G}^{\rm up}=R_0\|\mathsf S_{\alpha,m}/m\|_{\rm op}
 +\sqrt{\tr W_m+2\sqrt{\tr(W_m^2)\log(1/\delta)}
              +2\|W_m\|_{\rm op}\log(1/\delta)}.
 \label{eq:gaussian-ball-upper-new}
\end{equation}
\end{proposition}

\begingroup
\color{black}
\begin{proof}[Proof of \cref{prop:gaussian-ball-new}]
The linear representation in the proof of
\cref{prop:same-sgd-new} is Gaussian, with mean $\mu_m^{\rm out}$
and covariance $W_m$. This proves the asserted law.
The quantile definition gives coverage at least $1-\delta$.
If $W_m\succ0$, the Gaussian has a positive density, and the
distribution of its distance from the origin is continuous and
strictly increasing on positive radii; the coverage is then exact.
Also $\|\mu_m^{\rm out}\|\le R_0\|\mathsf S_{\alpha,m}/m\|_{\rm op}$,
so the triangle inequality proves \cref{eq:gaussian-ball-new}.

Here is a direct concentration proof of the upper radius.
First assume $W_m\ne0$.
If $\lambda_i$ are the eigenvalues of $W_m$, set
$Q=\sum_i\lambda_iZ_i^2$,
$a_W=\tr(W_m^2)$ and $c_W=2\|W_m\|_{\rm op}$.
For $0<s<c_W^{-1}$,
\[
 \log\E e^{s(Q-\tr W_m)}
 =\frac12\sum_i\{-\log(1-2s\lambda_i)-2s\lambda_i\}
 \le\frac{a_Ws^2}{1-c_Ws}.
\]
The inequality follows by the power series of $-\log(1-u)$.
For $a_W>0$, Chernoff's inequality with
$s=\sqrt{x}/(\sqrt{a_W}+c_W\sqrt{x})$ gives
\[
 \Pp\{Q>\tr W_m+2\sqrt{a_Wx}+c_Wx\}\le e^{-x}.
\]
If $W_m=0$ the same conclusion is immediate.
Take $x=\log(1/\delta)$, then a square root and the same triangle
inequality, to obtain \cref{eq:gaussian-ball-upper-new}.
\end{proof}
\par
\endgroup

The noncentral quantile uses $e_0=x_0-x_\star$ and is therefore usually a diagnostic benchmark. Known $H,\Sigma$ do not by themselves make $e_0$ known.
The radius \cref{eq:gaussian-ball-new} uses only the stated bound $R_0$.

\begin{proposition}[Approximation control on the class of balls]
\label{prop:ball-be-new}
Let the noise be independent with a finite third moment and $W_m\succ0$.
Denote the independent contributions of individual samples by
\[
 Y_{s i}=-\frac1{mb}(I-F_\alpha^{m-s})H^{-1}\xi_{s i},
 \qquad 1\le s\le m,\quad 1\le i\le b,
\]
and set
\[
 \varepsilon_m^{\rm ball}=\min\left\{1,(42d^{1/4}+16)
       \sum_{s,i}\E\|W_m^{-1/2}Y_{s i}\|^3\right\}.
\]
Then, uniformly in $r\ge0$,
\begin{equation}
 \left|\Pp\{\|\widehat x-x_\star\|\le r\}
       -\Pp\{\|\mu_m^{\rm out}+W_m^{1/2}Z\|\le r\}\right|
       \le\varepsilon_m^{\rm ball}.
 \label{eq:ball-be-new}
\end{equation}
If $\varepsilon_m^{\rm ball}<\delta$, then
$R_0\|\mathsf S_{\alpha,m}/m\|_{\rm op}
+\operatorname{Quant}_{1-\delta+\varepsilon_m^{\rm ball}}(0,W_m)$
is a valid radius of level $1-\delta$.
\end{proposition}

\begingroup
\color{black}
\begin{proof}[Proof of \cref{prop:ball-be-new}]
Put $X_{si}=W_m^{-1/2}Y_{si}$. The variables are independent and
centered and satisfy $\sum_{s,i}\Cov(X_{si})=I_d$.
The convex-set Berry--Esseen inequality of
\href{https://doi.org/10.3150/18-BEJ1072}{Rai\v{c} (2019, Theorem~1.1)}
therefore bounds the Gaussian approximation error by
$(42d^{1/4}+16)\sum_{s,i}\E\|X_{si}\|^3$.
The set
\[
 \{z:\|\mu_m^{\rm out}+W_m^{1/2}z\|\le r\}
\]
is convex for every $r\ge0$. Applying the inequality to this set,
and also using the trivial bound one, proves \cref{eq:ball-be-new}.
For the enlarged radius in the proposition, the Gaussian
probability is at least $1-\delta+\varepsilon_m^{\rm ball}$ by the
triangle inequality and the bound on $\|\mu_m^{\rm out}\|$.
Subtracting the approximation error proves coverage at least
$1-\delta$.
\end{proof}
\par
\endgroup

\subsubsection{Scalar Intervals and High Confidence Levels}
\label{sec:confidence-level}
The scalar bounds below assume known curvature and noise scale. A small error
probability is kept separate from heavy tails of the oracle; the latter are
treated in \cref{sec:heavy-tails}.
For $m\ge4S_\infty$, \cref{thm:sharp-normal} gives the bound
\begin{equation}
 d_{\rm K}(\cL(Z_m),\cN(0,1))\le\epsilon_{N,m}:=
 \underbrace{\frac{\sqrt5 C_{\rm BE}\nu_3}{\sqrt N}}_{\text{shape}}
 +\underbrace{2\phi(1)\frac {S_\infty}m}_{\text{covariance}}
 +\underbrace{\frac{|\theta|S_m\sqrt{b/m}}{\sqrt{2\pi}}}_{\text{start}},
 \label{eq:three-errors}
\end{equation}
where $\theta=\mu e_0/\sigma$. For the interval with half-width
$\Phi^{-1}(1-\delta/2)\sigma/(\mu\sqrt N)$, its coverage differs from
$1-\delta$ by at most $2\epsilon_{N,m}$. This follows by subtracting the
cdf at the two endpoints (using the left limit when there is an atom).
For Gaussian noise inputs the shape term is zero and the law
$\cN(h_m,v_m)$ gives the coverage at every $\delta$.
Consequently, vanishing center and covariance errors together with the shape
bound imply nominal coverage at every fixed confidence level.

For a vanishing error probability $\delta_N$, a sufficient condition for
\emph{relative} miscoverage accuracy is
$\epsilon_{N,m}=o(\delta_N)$, which is stronger than convergence in law.
As a separate
concentration statement, if the standardized noise inputs satisfy
$\E e^{s\xi/\sigma}\le e^{s^2/2}$ for all real $s$, then for every
$0<\delta<1$ the deterministic weights give
\begin{equation}
 \Pp\!\left\{ |Z_m-h_m|>
       \sqrt{2v_m\log(2/\delta)}\right\}\le\delta.
 \label{eq:subgaussian-window}
\end{equation}
Indeed, independence bounds the mgf of $Z_m-h_m$ by
$\exp(s^2v_m/2)$, and Chernoff's inequality proves the result. Covering
$x_\star$ also requires a bound on $|h_m|$.

\begingroup
\color{black}
\begin{theorem}[Coverage of the entire PR interval family]
\label{thm:coverage-family}
Consider the least-curved direction, the three standard policies,
and a common iid standardized noise law of mean zero and variance
one as in \cref{cor:calibrated-clt}. Let $m\to\infty$; allow
$\kappa\ge1$, $b\ge1$, and $\theta$ to vary. With
$B=S_\infty$, the following assertions are equivalent:
\begin{enumerate}
\item For every fixed $\delta\in(0,1)$,
$\Pp\{|Z_m|\le z_{1-\delta/2}\}\to1-\delta$.
\item $h_m\to0$ and $v_m\to1$.
\item $B/m\to0$ and $|\theta|B\sqrt{b/m}\to0$.
\end{enumerate}
Equivalently in the first assertion, the supremum of the absolute
coverage error over $\delta\in(0,1)$ tends to zero.
\end{theorem}

\begin{proof}
The equivalence of the second and third assertions is
\cref{thm:uniform-calibration}.
By \cref{cor:calibrated-clt},
$T_m=(Z_m-h_m)/\sqrt{v_m}\xrightarrow{\mathrm d}\cN(0,1)$.
Thus the second assertion implies the first by Slutsky's theorem.

Conversely, $0<v_m\le1$. If $|h_m|$ were unbounded along a
subsequence, then for every fixed $z>0$, Chebyshev's inequality
would give
\[
 \Pp\{|Z_m|\le z\}
 \le\frac{v_m}{(|h_m|-z)^2}\longrightarrow0,
\]
contradicting the first assertion. Hence $h_m$ is bounded.
From any subsequence extract a further one with
$h_m\to h$ and $v_m\to v\in[0,1]$. The Gaussian shape limit gives
$Z_m\xrightarrow{\mathrm d}h+\sqrt v\,G$ for $G\sim\cN(0,1)$.
The first assertion identifies the law of its absolute value as
that of $|G|$; this follows first at the continuity points and then
everywhere by right continuity. In particular, its second and
fourth moments equal those of $G$:
\[
 h^2+v=1,\qquad h^4+6h^2v+3v^2=3.
\]
Subtracting three times the square of the first equation from the
second yields $-2h^4=0$, so $h=0$ and $v=1$.
Every convergent subsequence has this limit, proving the second assertion.
Finally, convergence in distribution of $|Z_m|$ to the continuous
law of $|G|$ implies uniform convergence of their distribution
functions, which gives the asserted uniform absolute coverage
formulation.
\end{proof}

\begin{corollary}[One confidence level with a nonvanishing start]
\label{cor:coverage-fixed-start}
Under the assumptions of \cref{thm:coverage-family}, suppose in
addition that $\liminf|\theta|>0$; in particular, $\theta\ne0$
may be fixed. For any one fixed $\delta\in(0,1)$,
\[
 \Pp\{|Z_m|\le z_{1-\delta/2}\}\to1-\delta
 \quad\Longleftrightarrow\quad
 \frac Bm\to0,\qquad |\theta|B\sqrt{\frac bm}\to0.
\]
For a fixed nonzero $\theta$, the right-hand side is equivalently
$bB^2/m\to0$, or $b=o(\sqrt N/(|\theta|B))$ with $N=mb$.
If $\theta=0$ identically, the same one-level equivalence holds
with the sole condition $B/m\to0$.
\end{corollary}

\begin{proof}
Sufficiency follows from \cref{thm:coverage-family}.
For necessity, put $z=z_{1-\delta/2}$. The same Chebyshev
argument, using the positive limit $1-\delta$, shows that $h_m$
is bounded. If $B/m$ did not vanish, the argument in the proof
of \cref{thm:uniform-calibration} would give
$S_m/m\ge c>0$ along a subsequence. But then
\[
 |h_m|=|\theta|\sqrt{mb}\,\frac{S_m}{m}
 \ge c|\theta|\sqrt{mb}\longrightarrow\infty,
\]
a contradiction. Hence $v_m\to1$.
Along any subsequence where $h_m\to h$, the limiting coverage is
$\Phi(z-h)-\Phi(-z-h)$.
For $z>0$ this even function is strictly decreasing in $|h|>0$,
so it equals $2\Phi(z)-1$ only if $h=0$.
Therefore $h_m\to0$, and \cref{thm:uniform-calibration} proves
the required conditions.
For fixed nonzero $\theta$, $bB^2/m\to0$ also implies $B/m\to0$
because $b\ge1$ and $B\ge1$. This proves the last equivalences.
If $\theta=0$ identically, then $h_m=0$. Along a subsequence with
$v_m\to v\in[0,1]$, the limiting coverage equals
$2\Phi(z/\sqrt v)-1$ for $v>0$, and equals one for $v=0$.
It is nominal only for $v=1$, proving the zero-start case.
\end{proof}

The additional hypothesis in \cref{cor:coverage-fixed-start} is
essential for a one-level converse. If the start is allowed to
shrink, the finite-mixing limits of \cref{prop:finite-mixing}
permit any limiting variance $v\in(0,1)$ in the corresponding
range and an adjustable limiting mean $h$.
For a fixed $z>0$, continuity and strict monotonicity in $|h|$
give a nonzero $h$ such that
$\Pp\{|h+\sqrt v\,G|\le z\}=2\Phi(z)-1$ for $G\sim\cN(0,1)$.
Thus one-level cancellation is compatible with the general model
but excluded by a nonvanishing standardized start.
\par
\endgroup
\subsubsection{Unknown Curvature and Covariance: Independent Calibration}
The next result keeps the center of the ball equal to the original output
$\widehat x$ and explicitly pays for the additional oracle calls.
It requires the oracle $G(x)=H(x-x_\star)+\xi$ with independent
$\xi\sim\mathcal N(0,\Sigma)$, $\Sigma\succ0$.
\begin{proposition}[Finite-time certificate with curvature estimation]
\label{prop:calibration-new}
After training with $N_{\rm train}$ calls, fix a probe spacing
$\delta_H>0$ and $n>d$
and the probe points $y_0^{\rm probe}=\widehat x$ and
$y_j^{\rm probe}=\widehat x+\delta_H\mathbf e_j$, where $\mathbf e_j$ is the
$j$th coordinate vector, $1\le j\le d$.
At each point obtain $n$ new independent gradients, their mean
$\bar G_j$ and the unbiased sample covariance $\widehat\Sigma_j$.
All new noises are independent of the training; conditionally on the training
sample the calibration groups are independent. For $\delta_j=\delta/(d+1)$ set
\begin{align}
 c_j&=\frac{d(n-1)}{n(n-d)}F^{-1}_{d,n-d}(1-\delta_j),\qquad
 \epsilon_j=\sqrt{c_j\lambda_{\max}(\widehat\Sigma_j)},\qquad
 \widetilde H_{\cdot j}=\frac{\bar G_j-\bar G_0}{\delta_H},\nonumber\\
 \widehat H&=\tfrac12(\widetilde H+\widetilde H^\trans),\qquad
 \epsilon_H=\frac1{\delta_H}\sqrt{\sum_{j=1}^d(\epsilon_j+\epsilon_0)^2},
 \qquad \widehat\mu=\lambda_{\min}(\widehat H).
 \label{eq:calibration-data-new}
\end{align}
Here $F^{-1}_{d,n-d}$ is the quantile of the Fisher distribution.
When $\widehat\mu>\epsilon_H$, define
\begin{equation}
 r_{\rm cal}=
 \frac{\widehat\mu\|\widehat H^{-1}\bar G_0\|+\epsilon_0}
      {\widehat\mu-\epsilon_H};
 \qquad r_{\rm cal}=+\infty\quad\text{otherwise}.
 \label{eq:calibration-radius-new}
\end{equation}
Then $\Pp\{\|\widehat x-x_\star\|\le r_{\rm cal}\}\ge1-\delta$
at the total budget $N_{\rm train}+(d+1)n$.
If a bound $\lambda_{\min}(H)\ge\mu_0>0$ is known, one group at $\widehat x$
suffices: the radius $(\|\bar G_0\|+\epsilon_0)/\mu_0$ with $\delta_0=\delta$
has the same coverage and requires only $n$ additional calls.
\end{proposition}

\begingroup
\color{black}
\begin{proof}[Proof of \cref{prop:calibration-new}]
Condition on the training data. The probe points are then fixed.
Write $g_j=H(y_j^{\rm probe}-x_\star)$ and
$r_j^{\rm probe}=\bar G_j-g_j$.
Gaussian sampling and $n>d$ give the Hotelling pivot
\[
 \frac{n(n-d)}{d(n-1)}
 (r_j^{\rm probe})^\trans\widehat\Sigma_j^{-1}r_j^{\rm probe}
 \ \overset{\mathrm d}{=}\ F_{d,n-d}.
\]
Consequently $\|r_j^{\rm probe}\|\le\epsilon_j$ with conditional
probability at least $1-\delta_j$.
A union bound gives all these inequalities simultaneously with
conditional probability at least $1-\delta$.
On that event,
\[
 \|\widetilde H_{\cdot j}-H_{\cdot j}\|
 \le\frac{\epsilon_j+\epsilon_0}{\delta_H},\qquad
 \|\widehat H-H\|_{\rm op}
 \le\|\widetilde H-H\|_{\rm F}\le\epsilon_H.
\]
Symmetrization cannot increase the operator norm error because $H$
is symmetric. Put $e=\widehat x-x_\star$.
If $\widehat\mu>\epsilon_H$, then $\widehat H\succ0$ and
$\bar G_0=He+r_0^{\rm probe}$ imply
\[
 \|e\|\le\|\widehat H^{-1}\bar G_0\|
       +\frac{\epsilon_H}{\widehat\mu}\|e\|
       +\frac{\epsilon_0}{\widehat\mu}.
\]
Solving this inequality gives \cref{eq:calibration-radius-new}.
Otherwise the infinite radius covers .
Taking expectations over the training data proves the unconditional
coverage. With known $\mu_0$, the one-group bound follows instead
from $\mu_0\|e\|\le\|He\|\le\|\bar G_0\|+\epsilon_0$.
The numbers of oracle calls are respectively $(d+1)n$ and $n$.
\end{proof}
\par
\endgroup

\FloatBarrier
\subsection{Full Averaging, Burn-In, and Switching}
\label{sec:outputs-erm}
We first fix the full budget and keep the randomness of the state after burn-in.
This allows a correct comparison between continuing the accelerated method and
switching to SGD.
The following finite-dimensional formulas keep the random state after
burn-in. They require only independent additive noise inputs with covariance
$\Sigma/b$.
Let $T=t+m$, $E_1=[I,0]$, $J=[I,I]^\trans$, and, for a fixed method, define
\[
 \mathsf L_m=\frac1m\sum_{k=0}^{m-1}\mathsf O_\gamma A^k,\qquad
 K_{m,s}=\frac1m\sum_{k=s}^{m-1}\mathsf O_\gamma A^{k-s}B_\alpha
 \quad(1\le s\le m),
\]
where an empty sum is zero. For a deterministic initial state $\mathbf z_0$,
write
\[
 a_t=A^t\mathbf z_0,\qquad
 P_t=\frac1b\sum_{i=0}^{t-1}A^iB_\alpha\Sigma
 B_\alpha^\trans(A^\trans)^i.
\]
Then the tail output has moments
\begin{equation}
 \E\bar u_{t,m}=\mathsf L_ma_t,\qquad
 \Cov(\bar u_{t,m})=\mathsf L_mP_t\mathsf L_m^\trans
 +\frac1b\sum_{s=1}^mK_{m,s}\Sigma K_{m,s}^\trans.
 \label{eq:tail-exact}
\end{equation}
Indeed, expansion of the recursion gives
$\bar u_{t,m}=\mathsf L_m\mathbf z_t
+\sum_{s=1}^mK_{m,s}\eta_{t+s}^{(b)}$;
the future noises are independent of $\mathbf z_t$. This proves both formulas.
On the scalar slow mode, with a zero-velocity initial state, the mean is also
$e_0m^{-1}\sum_{k=t}^{T-1}r_k$.

For an accelerated first stage followed by a second method, reset velocity
at the switch: the second state is $J e_t=JE_1z_t$. Put superscript
$(2)$ on $\mathsf L_m,K_{m,s}$ for the second method, and compute $a_t,P_t$
using the first method. The output then has
\begin{equation}
 \E\bar u^{(2)}=\mathsf L_m^{(2)}JE_1a_t,\qquad
 \Cov(\bar u^{(2)})=\mathsf L_m^{(2)}JE_1P_tE_1^\trans J^\trans(\mathsf L_m^{(2)})^\trans
 +\frac1b\sum_{s=1}^mK_{m,s}^{(2)}\Sigma(K_{m,s}^{(2)})^\trans.
 \label{eq:two-stage-exact}
\end{equation}
The first covariance term is the propagated uncertainty from warm-up;
discarding it would grant the second stage an artificial deterministic
start. These identities also show why there is no automatic ranking
between continuing accelerated tail averaging and switching to SGD.

All three protocols---full averaging, tail averaging, and two stages---are
compared at $T$ sequential rounds and $N=bT$ oracle samples. We report
both the retained-window reference $V_\star/(bm)$ and the all-sample ERM
reference $V_\star/(bT)$. A method calibrated to the former can still
lose efficiency relative to the latter by spending too many rounds on
warm-up. The paired all-sample discrepancy in the experiments is computed
by jointly propagating the output and the sum of all noise inputs, retaining
their cross-covariance. 

\paragraph{Link with the Scalar Remainder.}
Tail averaging can reduce deterministic bias.  In the same scalar model,
the mean of the average starting at round $t$ is 
$e_0m^{-1}\sum_{j=t}^{t+m-1}r_j$, so its normalized bias is
\[
 \frac{\mu e_0}{\sigma}\sqrt{b/m}
       \sum_{j=t}^{t+m-1}r_j.
\]
The infinite tail from \cref{eq:hard-mode-tail} provides a sufficient bound
with the repeated-root factor included.  If all burn-in and retained rounds
use the same batch, the full work is $(t+m)b$ samples and $t+m$ rounds;
$mb$ counts only retained-window work. 

\subsection{Choosing among Specified Procedures}
\label{sec:procedure-selection}
In the additive quadratic model, any linear procedure $\pi$ fixed in advance and
started with zero velocity can be written as
\[
 \widehat x_\pi-x_\star=F_\pi e_0+
       \sum_{s=1}^{T_\pi}W_{\pi,s}\eta_s^{(b_{\pi,s})},\qquad
 V_\pi=\sum_{s=1}^{T_\pi}W_{\pi,s}\Sigma W_{\pi,s}^\trans/b_{\pi,s}.
\]
The matrices are computed recursively, including the random state at the switching
time. Let $\Pi_N$ be a finite grid, specified in advance, of steps, methods,
windows, batches and switching times with $\sum_s b_{\pi,s}=N$ for every
candidate. All steps discarded by the averaging are paid for.
\begin{proposition}[Optimal choice within a given grid]
\label{prop:selection-new}
With known $H,\Sigma$ and the bound $\|e_0\|\le R_0$, put
\[
 U_\pi=R_0^2\|F_\pi\|_{\rm op}^2+\tr V_\pi,
 \qquad \pi_\star\in\arg\min_{\pi\in\Pi_N}U_\pi.
\]
Then
$\sup_{\|e_0\|\le R_0}\E\|\widehat x_{\pi_\star}-x_\star\|^2
=\min_{\pi\in\Pi_N}\sup_{\|e_0\|\le R_0}
\E\|\widehat x_\pi-x_\star\|^2=U_{\pi_\star}$.
Under Gaussian noise, the radius
$R_0\|F_\pi\|_{\rm op}
+\operatorname{Quant}_{1-\delta}(0,V_\pi)$ is valid for the
same procedure.
Minimizing this radius over $\Pi_N$ solves a separate selection
problem for the confidence ball.
\end{proposition}

\begingroup
\color{black}
\begin{proof}[Proof of \cref{prop:selection-new}]
The representation preceding the proposition gives
\[
 \E\|\widehat x_\pi-x_\star\|^2
 =\|F_\pi e_0\|^2+\tr V_\pi.
\]
The supremum of the first term over $\|e_0\|\le R_0$ equals
$R_0^2\|F_\pi\|_{\rm op}^2$, by a maximizing right singular vector.
Thus the worst-case risk is $U_\pi$ for each candidate.
Minimizing over the prescribed finite grid proves the asserted
minimax identity within that grid.
Under Gaussian noise, the centered random term has law
$V_\pi^{1/2}Z$, so the triangle inequality gives the stated radius.
The choice of $\pi$ here depends only on the known geometry and the
prescribed grid, not on the training observations; hence this choice
does not change its coverage guarantee.
\end{proof}
\par
\endgroup

With known $e_0$ one can minimize the risk; this is an \emph{oracle} choice
for the analysis, since $x_\star$ is usually unknown. The choice via $R_0$ uses
only the stated bound on the start. Thus the burn-in length and the switching
time are justified relative to an explicit grid and target metric.

If the parameters are estimated from an independent pilot sample of size
$N_{\rm pilot}$, the main budget equals $N-N_{\rm pilot}$.
On the simultaneous event $\max_{\pi\in\Pi_N}|\widehat U_\pi-U_\pi|\le\epsilon$,
the minimizer of $\widehat U$ satisfies $U_{\widehat\pi}\le\min_\pi U_\pi+2\epsilon$.
This follows from two applications of the error bound and the definition of the
minimizer. 

\FloatBarrier
\section{Heavy Tails and Clipping}
\label{sec:heavy-tails}

For a robust online procedure, the link between the optimization error and the
distribution of the output has to be checked separately. Clipping changes the
oracle, so its effect on the optimization guarantee, on the bias and on the
distribution of the output should be analyzed jointly. We first note what heavy
tails do inside the stochastically linear class, and then how clipping leaves it.

\subsection{Finite Variance and High Confidence Levels}
\label{sec:finite-variance-tails}

\begin{remark}[Heavy tails inside the stochastically linear class]
\label{rem:tails-sls}
In the additive quadratic model
$\hat x^{\rm ERM}_{t,m,b}-x_\star=-H^{-1}\bar\eta^{(b)}_{t,m}$ , so every
deviation bound for the empirical noise mean---Gaussian, Fuk--Nagaev or
stable---applies verbatim to the offline estimator, and by
\cref{prop:observable-bridge} the online output adds only the bridge correction
$H^{-1}\bar g_{t,m}$. Within this class, heavy tails therefore degrade the
statistical term identically in both languages; the languages differ in the
robustifications they admit (\cref{sec:clipped-transfer}).
\end{remark}

\paragraph{A Map of Moment Regimes.}
The available guarantees depend on the moments of the noise. With a finite
third moment, \cref{thm:sharp-normal} gives a quantitative independent normal
approximation. With a finite variance but no third moment,
\cref{cor:calibrated-clt} still gives a qualitative CLT with centering and
variance; the Nagaev bound below additionally requires a moment of some order
$r_{\rm tail}>2$, which does not follow from a finite variance alone, and
uniformly over the class no sub-Gaussian radius is available for the ERM or the
averaged outputs (\cref{prop:catoni-lower}). With infinite variance, the stable
limit transfers under a domain-of-attraction assumption and the correct
normalization, strictly stable noise admits an stable law and calibration
criteria (\cref{prop:stable-calibration}), and the Gaussian covariance and the
MSE may not exist.

\subsubsection{Transfer of the Nagaev Bound}

Let $a\in\R^d$, $Y_a=a^{\trans}H^{-1}\xi$, and assume
$\E Y_a=0$, $\E Y_a^2=\sigma_a^2>0$, and
$\E|Y_a|^{r_{\rm tail}}<\infty$ for some $r_{\rm tail}>2$.
For independent copies $Y_{a,1},\ldots,Y_{a,N}$, a Fuk--Nagaev inequality gives
constants $c_{r_{\rm tail}},C_{r_{\rm tail}}>0$, depending only on
$r_{\rm tail}$, such that
\begin{equation}
 \Pp\!\left(\left|\frac1N\sum_{i=1}^NY_{a,i}\right|>x\right)
 \le 2\exp\!\left(-c_{r_{\rm tail}}\frac{Nx^2}{\sigma_a^2}\right)
 +\frac{C_{r_{\rm tail}}\E|Y_a|^{r_{\rm tail}}}
       {N^{r_{\rm tail}-1}x^{r_{\rm tail}}},\qquad x>0.
\label{eq:nagaev-mean}
\end{equation}
See \citet{fuk1971probability,nagaev1979large,nagaev1998probabilistic}.
The zero-variance case is trivial.

The decomposition \cref{eq:exact-decomp} and a union bound imply,
for every $x,r_{\partial}>0$ and $N=mb$,
\begin{equation}
 \Pp(|a^{\trans}\bar u_{t,m}|>x+r_{\partial})
 \le2\exp\!\left(-c_{r_{\rm tail}}\frac{Nx^2}{\sigma_a^2}\right)
 +\frac{C_{r_{\rm tail}}\E|Y_a|^{r_{\rm tail}}}{N^{r_{\rm tail}-1}x^{r_{\rm tail}}}
 +\Pp\!\left(\frac{|a^{\trans}H^{-1}\Delta_{t,m}|}{\alpha m}>r_{\partial}\right).
 \label{eq:nagaev-optimizer}
\end{equation}
A sufficient empirical-mean radius at confidence $1-\delta$ is
\begin{equation}
 C_{r_{\rm tail}}'\left[
 \sigma_a\sqrt{\frac{\log(4/\delta)}N}
 +\left(\frac{\E|Y_a|^{r_{\rm tail}}}
 {\delta N^{r_{\rm tail}-1}}\right)^{1/r_{\rm tail}}\right].
\label{eq:nagaev-radius}
\end{equation}
In particular, the boundary can be
correlated with and can partly cancel the empirical-noise term.
Robust nonlinear estimators under finite covariance can also have
better high-confidence behavior than an empirical mean; \citep{gorbunov2020stochastic,puchkin2024breaking}.

\subsubsection{The High-Confidence Limit of the Empirical Mean}
\label{sec:catoni-lower}

The Nagaev radius \cref{eq:nagaev-radius} contains a term polynomial in $1/\delta$.
The construction is the classical one of
\citet{catoni2012challenging}.

\begin{proposition}[High-confidence limit of the empirical mean and of averaged outputs]
\label{prop:catoni-lower}
Consider the scalar additive quadratic model with curvature $\mu>0$ and centered
i.i.d. noise of variance $\sigma^2$, and let $N=mb$.
\begin{enumerate}
\item For every $N\ge1$ and $\delta\in(0,1/8]$ there is a symmetric noise law with
variance $\sigma^2$ such that the ERM of the $N$ observations satisfies
\[
 \Pp\Bigl\{|\hat x^{\rm ERM}-x_\star|\ge\frac{\sigma}{\mu\sqrt{2N\delta}}\Bigr\}
 \ge\delta.
\]
\item Let $\bar u_{0,m}$ be the full average of SGD, HB or NAG with the standard
parameters of \cref{prop:cold-bias}, started at $e_0=0$, and let
$m\ge4S_\infty$. For
every $\delta\in(0,1/16]$ there is a symmetric noise law with variance $\sigma^2$
such that
\[
 \Pp\Bigl\{|\bar u_{0,m}|\ge\frac{\sigma}{4\mu\sqrt{N\delta}}\Bigr\}\ge\delta.
\]
\end{enumerate}
\end{proposition}

Indeed, let the noise take the values $\pm\omega_{\rm spike}$ with probability
$\pi_N/2$ each and the value $0$ otherwise, with
$\omega_{\rm spike}=\sigma/\sqrt{\pi_N}$. The event that one
of the $N$ noises is nonzero has probability
$N\pi_N(1-\pi_N)^{N-1}$, and on it the ERM
error equals $\omega_{\rm spike}/(\mu N)$ in absolute value. With
$\pi_N=2\delta/N$ this equals
$\sigma/(\mu\sqrt{2N\delta})$, while
$N\pi_N(1-\pi_N)^{N-1}
=2\delta(1-2\delta/N)^{N-1}\ge2\delta e^{-4\delta}\ge\delta$ for
$\delta\le1/8$. For the averaged outputs, \cref{eq:cold-exact-law} with $e_0=0$
gives $\bar u_{0,m}=-(\mu N)^{-1}\sum_{s,i}w_{m-s}\xi_{s,i}$ with
$w_j=1-r_j\in[0,1]$. Since $r_j\ge0$ and
$\sum_jr_j=S_\infty$, fewer than $2S_\infty$ indices satisfy $r_j>1/2$;
for $m\ge4S_\infty$ at least $N/2$ of the $N$ noises therefore carry a
weight of at least $1/2$. With $\pi_N=4\delta/N$, the event that one noise is
nonzero and that it carries such a weight has probability at least
$(N/2)\pi_N(1-\pi_N)^{N-1}\ge2\delta e^{-8\delta}\ge\delta$
for $\delta\le1/16$, and on it
$|\bar u_{0,m}|\ge\omega_{\rm spike}/(2\mu N)
=\sigma/(4\mu\sqrt{N\delta})$.

The Gaussian radius $\Phi^{-1}(1-\delta/2)\sigma/(\mu\sqrt N)$ is of order
$\sigma\sqrt{\log(1/\delta)}/(\mu\sqrt N)$, so the ratio of the two radii grows like
$(\delta\log(1/\delta))^{-1/2}$ as $\delta\to0$. No ball of sub-Gaussian radius
centered at the ERM or at the averaged output is therefore valid uniformly over the
finite-variance class, whatever Fisher or Wilks expansion is used to compute it.
A one-pass clipped estimator does not have this obstruction: the radius of
\cref{prop:robust-fixed-query} depends on $\delta$ only through
$\sqrt{\log(2d/\delta)}$ under a finite variance alone. This is the precise sense
in which heavy tails separate the offline empirical-risk description from an
online procedure that is allowed to clip. The lower bound concerns uniformity over
the class: for a fixed noise law and a fixed $\delta$, the normal approximations of
\cref{sec:sharp-normal} remain valid as $N\to\infty$.

\subsection{Infinite Variance}

A stable domain of attraction generally requires a slowly varying factor
in addition to the power of the sample size
\citep{samorodnitsky1994stable}. For SGD with infinite noise variance,
$\alpha$-stable limits of Polyak--Ruppert averages were obtained by
\citet{wang2021convergence}.

\begin{proposition}[Stable-law transfer with the correct normalizer]
\label{prop:stable-limit}
Suppose the centered i.i.d. single-sample noises satisfy, for positive
constants $c_N$,
\[
 c_N^{-1}\sum_{i=1}^N\xi_i\xrightarrow{\mathrm d} Z_{\alpha_0},
 \qquad 1<\alpha_0<2,
\]
where $Z_{\alpha_0}$ is an $\alpha_0$-stable vector. In a stable domain
of attraction one may take
$c_N=N^{1/\alpha_0}\mathcal L_{\rm sv}(N)$ with
$\mathcal L_{\rm sv}$ slowly
varying.  For $N=mb\to\infty$, if
\begin{equation}
 \frac{N}{\alpha m c_N}\Delta_{t,m}\xrightarrow{\Pp}0,
                                                    \label{eq:stable-boundary}
\end{equation}
then the additive quadratic algorithm satisfies
\begin{equation}
 \frac{N}{c_N}\bar u_{t,m}\xrightarrow{\mathrm d}-H^{-1}Z_{\alpha_0}.
                                                    \label{eq:stable-pr}
\end{equation}
\end{proposition}

\begingroup
\color{black}
\begin{proof}[Proof of \cref{prop:stable-limit}]
Enumerate the $N=mb$ individual noises in the retained window.
Multiplying \cref{eq:exact-decomp} by $N/c_N$ gives
\[
 \frac N{c_N}\bar u_{t,m}
 =-H^{-1}\frac1{c_N}\sum_{i=1}^N\xi_i
   +H^{-1}\frac N{\alpha m c_N}\Delta_{t,m}.
\]
The first term has the stable limit by the domain-of-attraction
assumption. With the fixed Hessian of the additive model, the second
term converges to zero in probability by \cref{eq:stable-boundary}.
Slutsky's theorem proves \cref{eq:stable-pr}.
\end{proof}
\par
\endgroup

The normalization $c_N=N^{1/\alpha_0}$ is the special case of
normal attraction, not every stable domain of attraction.  In that case,
for fixed algorithmic gain $\alpha$, the boundary condition becomes
$b^{1-1/\alpha_0}m^{-1/\alpha_0}\Delta_{t,m}\xrightarrow{\Pp}0$.
If the driving noise is strictly $\alpha_0$-stable and the linear
chain is stationary, each state scales in distribution as
$b^{1/\alpha_0-1}$ times its single-sample counterpart.  Indeed,
independent batching has this scaling and the stationary state
is the convergent stable linear filter
$\sum_{j\ge0}A^jB_\alpha\eta_{-j}^{(b)}$; stability makes its coefficient
$\alpha_0$-powers summable.  Tightness of its endpoint marginals then
makes the normalized stochastic boundary $O_{\Pp}(m^{-1/\alpha_0})$.

\paragraph{Calibration under Strictly Stable Noise.}
When the noise is strictly stable, the cold-start law of
\cref{thm:cold-exact} remains exact, and the calibration analysis of
\cref{sec:batching} carries over with the exponent $2$ replaced by $\alpha_0$.

\begin{proposition}[Stable calibration of the averaged output]
\label{prop:stable-calibration}
In the scalar model of \cref{thm:cold-exact}, let the single-sample noises be
i.i.d. symmetric $\alpha_0$-stable with $1<\alpha_0<2$ and scale $\sigma_{\alpha_0}$,
that is, $\E e^{\mathrm{i}s\xi}=\exp(-\sigma_{\alpha_0}^{\alpha_0}|s|^{\alpha_0})$, and
let $Z_{\alpha_0}^{\rm st}$ be standard symmetric $\alpha_0$-stable.
Put $N=mb$, $w_j=1-r_j$ and
\[
 \theta_{\alpha_0}=\frac{\mu e_0}{\sigma_{\alpha_0}},\qquad
 h_m^{(\alpha_0)}=\theta_{\alpha_0}S_m\,b^{1-1/\alpha_0}m^{-1/\alpha_0},\qquad
 v_m^{(\alpha_0)}=\frac1m\sum_{j=0}^{m-1}|w_j|^{\alpha_0}.
\]
Then, for all $m,b\ge1$ and any two-step parameters,
\begin{equation}
 \frac{\mu N^{1-1/\alpha_0}}{\sigma_{\alpha_0}}\,\bar u_{0,m}
 \overset{\mathrm d}{=} h_m^{(\alpha_0)}
 +\bigl(v_m^{(\alpha_0)}\bigr)^{1/\alpha_0}Z_{\alpha_0}^{\rm st},
 \label{eq:stable-exact-law}
\end{equation}
whereas the ERM of the same $N$ observations satisfies
$\mu N^{1-1/\alpha_0}(\hat x^{\rm ERM}-x_\star)/\sigma_{\alpha_0}
 \overset{\mathrm d}{=}Z_{\alpha_0}^{\rm st}$.
For the standard SGD, HB and NAG parameters of \cref{prop:cold-bias}, along any
sequence with $m\to\infty$, integer $b\ge1$, arbitrary $\kappa\ge1$ and arbitrary
signed $\theta_{\alpha_0}$,
\begin{align}
 v_m^{(\alpha_0)}\to1&\quad\Longleftrightarrow\quad S_\infty/m\to0,
 \label{eq:stable-scale-iff}\\
 \bigl(h_m^{(\alpha_0)},v_m^{(\alpha_0)}\bigr)\to(0,1)&\quad\Longleftrightarrow\quad
 S_\infty/m\to0\ \text{and}\
 |\theta_{\alpha_0}|S_\infty b^{1-1/\alpha_0}m^{-1/\alpha_0}\to0.
 \label{eq:stable-calibration-iff}
\end{align}
\end{proposition}

Indeed, for symmetric strictly stable noise and deterministic coefficients
$\omega_i^{\rm st}$,
$\sum_i\omega_i^{\rm st}\xi_i\overset{\mathrm d}{=}
\sigma_{\alpha_0}
(\sum_i|\omega_i^{\rm st}|^{\alpha_0})^{1/\alpha_0}
Z_{\alpha_0}^{\rm st}$.
Applied to \cref{eq:cold-exact-law}, where each weight $w_{m-s}/(\mu N)$ appears
$b$ times, this gives the scale
\[
 \frac{\sigma_{\alpha_0}(bm\,v_m^{(\alpha_0)})^{1/\alpha_0}}{\mu N}
 =\frac{\sigma_{\alpha_0}}{\mu}\bigl(v_m^{(\alpha_0)}\bigr)^{1/\alpha_0}N^{1/\alpha_0-1},
\]
and the deterministic part $e_0S_m/m$ becomes $h_m^{(\alpha_0)}$ after normalization; the
ERM is the same computation with unit weights. For the standard parameters
$0\le r_j\le1$, and for $r\in[0,1]$ one has $r\le1-(1-r)^{\alpha_0}\le\alpha_0r$.
Hence $S_m/m\le1-v_m^{(\alpha_0)}\le\alpha_0S_m/m$, which gives
\cref{eq:stable-scale-iff} as for \cref{eq:covariance-calibration-iff},
while \cref{eq:stable-calibration-iff} follows because
$S_m=S_\infty(1+o(1))$ when $S_\infty/m\to0$.

Consequently, at $b=1$ and fixed $\theta_{\alpha_0}\ne0$, calibrated inference from
the full average requires
$m\gg(|\theta_{\alpha_0}|S_\infty)^{\alpha_0}$, that is,
$m\gg\kappa^{\alpha_0}$ for SGD and $m\gg\kappa^{\alpha_0/2}$ for HB and NAG, up to
the factor $|\theta_{\alpha_0}|^{\alpha_0}$. At a fixed budget $N$ and with
$bS_\infty/N\to0$, the start-induced critical batch becomes
\begin{equation}
 b_{\rm crit}^{(\alpha_0)}(N)
 \asymp\frac{N^{1/\alpha_0}}{|\theta_{\alpha_0}|S_\infty},
 \label{eq:stable-bcrit}
\end{equation}
which reduces to \cref{eq:bcrit} at $\alpha_0=2$; momentum again enlarges it by a
factor of order $\sqrt\kappa$. Heavier tails lower the onset and raise the critical
batch: the statistical scale $N^{1/\alpha_0-1}$ decays more slowly, so a fixed start
is hidden sooner, although the accuracy is worse. The optimization
description is silent in this regime: the mean squared error is infinite, so
guarantees in mean square carry no information, whereas \cref{eq:stable-exact-law}
yields calibrated confidence intervals from stable quantiles. These intervals
require the index $\alpha_0$ and the scale $\sigma_{\alpha_0}$, whose estimation is
a separate problem; for noise in a domain of attraction, \cref{prop:stable-limit}
gives the corresponding asymptotic statement.

\begin{conjecture}[Stiff mode of HB under stable noise]
\label{rem:stable-stiff}
On the stiffest direction, the relative scale defect of optimally tuned HB is
$v_m^{(\alpha_0)}-1=m^{-1}\sum_{j<m}(|1-r_j|^{\alpha_0}-1)$ with the residual
$r_j=(-\varrho)^j\{1+(1+\varrho)j\}$ of \cref{thm:stiff}. The sum is dominated by the indices
with $|r_j|\gg1$, and replacing it by the corresponding integral suggests
\[
 v_m^{(\alpha_0)}-1\sim C_{\alpha_0}\frac{\kappa^{(1+\alpha_0)/2}}{m},\qquad
 C_{\alpha_0}=\frac{2^{\alpha_0}\Gamma(\alpha_0+1)}{(2\alpha_0)^{\alpha_0+1}},
\]
when $\kappa\to\infty$ and $m/\sqrt\kappa\to\infty$. At $\alpha_0=2$ this reproduces
the coefficient $\kappa^{3/2}/(8m)$ of \cref{eq:stiff-asymptotic}. The conjecture predicts persistence of the obstruction under infinite
variance, with the exponent $3/2$ replaced by
$(1+\alpha_0)/2$; a complete proof requires control of the oscillating lower-order
terms.
\end{conjecture}

A nonlinear stable PR theorem would additionally require control of its
remainders at scale $N/c_N$.

\subsection{A Fully Specified Finite-Variance Robust Benchmark}
\label{sec:robust-benchmark}

\begingroup
For a threshold $\tau_{\rm clip}>0$, write
$\operatorname{clip}_{\tau_{\rm clip}}(v)
=v\min\{1,\tau_{\rm clip}/\|v\|\}$ for $v\ne0$ and
$\operatorname{clip}_{\tau_{\rm clip}}(0)=0$. 
\par
\endgroup

We first give a fully specified clipping guarantee. Consider the additive quadratic oracle with known $H\succ0$. Fix a query
$x_0$ before observing the samples, and suppose known bounds satisfy
$\|x_0-x_\star\|\le R_{\rm init}$, $\|H\|\le L$, and
$\E\|\xi\|^2\le \nu_2$. Let
$M_{\rm clip}=L^2R_{\rm init}^2+\nu_2>0$.
For $N=mb$ independent observations, clip each full gradient before
averaging within a batch:
\[
 Y_{k,j}=G(x_0,\zeta_{k,j}),\qquad
 \widetilde Y_{k,j}^{\rm cl}
 =\operatorname{clip}_{\tau_{\rm clip}}(Y_{k,j}),\qquad
 \widehat x_{\rm clip}=x_0-H^{-1}\frac1{mb}
    \sum_{k=1}^m\sum_{j=1}^b \widetilde Y_{k,j}^{\rm cl}.
\]
This is a streaming estimator with one running vector sum; an offline
implementation on the same observations returns the same output.
It uses $N$ oracle calls and one final linear solve.

\begin{proposition}[Finite-budget coverage after full-gradient clipping]
\label{prop:robust-fixed-query}
Assume only independent centered oracle noise with the finite second
moment above. For $0<\delta<1$, let
$\Lambda_\delta=\log(2d/\delta)$. For any $\tau_{\rm clip}>0$,
with probability at least $1-\delta$,
\begin{equation}
 \|H(\widehat x_{\rm clip}-x_\star)\|
 \le \frac{M_{\rm clip}}{\tau_{\rm clip}}
 +\sqrt d\left\{\sqrt{\frac{2\nu_2\Lambda_\delta}{N}}+
        \frac{4\tau_{\rm clip}\Lambda_\delta}{3N}\right\}.
 \label{eq:robust-fixed-radius}
\end{equation}
In particular, the choice
$\tau_{\rm clip}=\sqrt{M_{\rm clip}N/\Lambda_\delta}$ gives
\begin{equation}
 \Pp\left\{\|\widehat x_{\rm clip}-x_\star\|
       \le\frac4\mu\sqrt{\frac{dM_{\rm clip}\log(2d/\delta)}{N}}\right\}
 \ge1-\delta,\qquad \mu=\lambda_{\min}(H).
 \label{eq:robust-fixed-ball}
\end{equation}
For $M_{\rm clip}=0$ the oracle is deterministic and $x_0=x_\star$, so the
zero-radius statement is immediate.
\end{proposition}

\begingroup
\color{black}
\begin{proof}[Proof of \cref{prop:robust-fixed-query}]
Write $g=H(x_0-x_\star)$, $\tau=\tau_{\rm clip}$, and
$T_\tau(y)=\operatorname{clip}_\tau(y)$.
For an individual answer $Y=g+\xi$,
$\E\|Y\|^2=\|g\|^2+\E\|\xi\|^2\le M_{\rm clip}$, and
\[
 \|\E T_\tau(Y)-g\|
 \le\E\|Y-T_\tau(Y)\|
 \le\E[\|Y\|\mathbf 1_{\{\|Y\|>\tau\}}]
 \le M_{\rm clip}/\tau.
\]
The clipping map is Euclidean projection onto the closed ball of
radius $\tau$, hence it is nonexpansive.
If $Y'$ is an independent copy of $Y$, then for each coordinate $i$,
\[
 \Var((T_\tau(Y))_i)
 =\tfrac12\E[(T_\tau(Y))_i-(T_\tau(Y'))_i]^2
 \le\tfrac12\E\|Y-Y'\|^2
 =\E\|\xi\|^2\le\nu_2.
\]
The centered clipped coordinate is bounded in absolute value by
$2\tau$. Bernstein's inequality and a union bound over both tails
and all $d$ coordinates imply, with probability at least $1-\delta$,
\[
 \left\|\frac1N\sum_{\ell=1}^N
     \{T_\tau(Y_\ell)-\E T_\tau(Y_\ell)\}\right\|
 \le\sqrt d\left\{\sqrt{\frac{2\nu_2\Lambda_\delta}{N}}
                     +\frac{4\tau\Lambda_\delta}{3N}\right\}.
\]
Since $H(\widehat x_{\rm clip}-x_\star)
=g-N^{-1}\sum_\ell T_\tau(Y_\ell)$, adding the bias bound proves
\cref{eq:robust-fixed-radius}. Substituting the stated threshold,
using $\nu_2\le M_{\rm clip}$ and $d\ge1$, bounds its right-hand side by
\[
 (1+\sqrt2+4/3)\sqrt{dM_{\rm clip}\Lambda_\delta/N}
 <4\sqrt{dM_{\rm clip}\Lambda_\delta/N}.
\]
Finally $\|H^{-1}\|=1/\mu$ proves \cref{eq:robust-fixed-ball}.
The case $M_{\rm clip}=0$ is deterministic as stated.
\end{proof}
\par
\endgroup

The benchmark closes a finite-variance, finite-budget comparison for the
same target and data stream. 

\subsection{Clipping along a Moving Trajectory}
\label{sec:clipped-transfer}

A finite third absolute
moment permits the quantitative independent-weight bounds in
\cref{sec:sharp-normal}.  Finite variance without a third moment can
still permit a qualitative CLT, but not that Berry--Esseen estimate.
With infinite variance, the original $\Vstar$ and MSE may not exist.

For the moving-query recursion, specify $b=1$ and clip the \emph{full gradient}, not an
unobservable isolated noise vector. For a threshold $\tau_{\rm cl}>0$, define
$\operatorname{clip}_{\tau_{\rm cl}}(v)
=v\min\{1,\tau_{\rm cl}/\|v\|\}$ when $v\ne0$, and zero
when $v=0$. Let $\tau_{k,\rm cl}>0$ be predictable and set
\begin{equation}
 g_k^{\rm cl}=\operatorname{clip}_{\tau_{k,\rm cl}}(Hu_k+\xi_{k+1}),\qquad
 \mathfrak b_k^{\rm cl}=\E[g_k^{\rm cl}\mid\cF_k]-Hu_k,\qquad
 \varepsilon_{k+1}^{\rm cl}=g_k^{\rm cl}-\E[g_k^{\rm cl}\mid\cF_k].
 \label{eq:clipping-martingale}
\end{equation}
Thus $\mathfrak b_k^{\rm cl}$ is predictable and
$\E[\varepsilon_{k+1}^{\rm cl}\mid\cF_k]=0$.
The recursion is $e_{k+1}=e_k+\beta(e_k-e_{k-1})-\alpha g_k^{\rm cl}$,
with the same query definition as before.  Summing it yields the exact
pathwise identity
\begin{equation}
 \sqrt m\bar u_{0,m}
 =-\frac{H^{-1}}{\sqrt m}\sum_{k=0}^{m-1}\varepsilon_{k+1}^{\rm cl}
  -\frac{H^{-1}}{\sqrt m}\sum_{k=0}^{m-1}\mathfrak b_k^{\rm cl}
  +\frac{H^{-1}\Delta_{0,m}}{\alpha\sqrt m}.
 \label{eq:clipped-boundary}
\end{equation}
It exposes the accumulated clipping bias as well as the
algorithmic boundary.

\begin{theorem}[Conditional clipped-gradient CLT]
\label{thm:clipped-transfer}
Fix $H\succ0$ and the parameters $\alpha,\beta,\gamma$.  Consider
clipped recursions of length $m$, allowing thresholds, initial states,
and hence trajectories to depend on $m$; the row index is suppressed.
Assume the clipped gradients are integrable and the martingale
differences in \cref{eq:clipping-martingale} are square-integrable, and
\begin{align}
 \frac1{\sqrt m}\sum_{k=0}^{m-1}\mathfrak b_k^{\rm cl}&\xrightarrow{\Pp}0,
 &\frac{\Delta_{0,m}}{\alpha\sqrt m}&\xrightarrow{\Pp}0,
 \label{eq:clipped-negligible}\\
 \frac1m\sum_{k=0}^{m-1}
 \E[\varepsilon_{k+1}^{\rm cl}(\varepsilon_{k+1}^{\rm cl})^{\trans}\mid\cF_k]
 &\xrightarrow{\Pp}\Sigma_{\rm cl},
 \label{eq:clipped-qv}
\end{align}
where $\Sigma_{\rm cl}$ is deterministic and positive semidefinite.
For every $\varepsilon_{\rm Lin}>0$, assume conditional Lindeberg:
\begin{equation}
 \frac1m\sum_{k=0}^{m-1}
 \E\!\left[\|\varepsilon_{k+1}^{\rm cl}\|^2
 \ind_{\{\|\varepsilon_{k+1}^{\rm cl}\|>\varepsilon_{\rm Lin}\sqrt m\}}\mid\cF_k\right]
 \xrightarrow{\Pp}0.
 \label{eq:clipped-lindeberg}
\end{equation}
Then
\[
 \sqrt m\bar u_{0,m}\xrightarrow{\mathrm d}
 \cN(0,H^{-1}\Sigma_{\rm cl}H^{-1}).
\]
The original PR covariance is recovered only if
$\Sigma_{\rm cl}=\Sigma$ with finite $\Sigma$.
\end{theorem}

\begingroup
\color{black}
\begin{proof}[Proof of \cref{thm:clipped-transfer}]
For any fixed vector $a$, the variables
$a^\trans\varepsilon_{k+1}^{\rm cl}/\sqrt m$ form a martingale
triangular array. Its predictable quadratic variation converges to
$a^\trans\Sigma_{\rm cl}a$ by \cref{eq:clipped-qv}.
The scalar conditional Lindeberg condition follows from
\cref{eq:clipped-lindeberg} and
$|a^\trans\varepsilon|\le\|a\|\|\varepsilon\|$.
The martingale CLT and Cram\'er--Wold therefore give
\[
 m^{-1/2}\sum_{k=0}^{m-1}\varepsilon_{k+1}^{\rm cl}
 \xrightarrow{\mathrm d}\cN(0,\Sigma_{\rm cl}).
\]
Both other terms in \cref{eq:clipped-boundary} vanish in probability
by \cref{eq:clipped-negligible}, since $H$ and the method parameters
are fixed. Multiplication by $-H^{-1}$ and Slutsky's theorem prove
the conclusion, including the possibly singular Gaussian case.
\end{proof}
\par
\endgroup

The conditions are checkable targets.  Two elementary bounds illustrate their distinct
roles. With centered raw oracle noise and moment order $r_{\rm mom}>1$,
\begin{equation}
 \|\mathfrak b_k^{\rm cl}\|\le
 \E[\|Hu_k+\xi_{k+1}\|
       \ind_{\{\|Hu_k+\xi_{k+1}\|>\tau_{k,\rm cl}\}}\mid\cF_k]
 \le\frac{\E[\|Hu_k+\xi_{k+1}\|^{r_{\rm mom}}\mid\cF_k]}
              {\tau_{k,\rm cl}^{r_{\rm mom}-1}}.
 \label{eq:clipping-bias-bound}
\end{equation}
Indeed, clipping changes a vector by norm $(\|v\|-\tau_{\rm cl})_+$.
Also $\|\varepsilon_{k+1}^{\rm cl}\|\le2\tau_{k,\rm cl}$; hence the envelope
$\tau_{\max,m}:=\max_{k<m}\tau_{k,\rm cl}=o(\sqrt m)$ ensures Lindeberg.
For example, if the raw gradients have uniformly bounded
$r_{\rm mom}$th moments for $r_{\rm mom}>2$, using
$\tau_{k,\rm cl}=m^{\omega_{\rm cl}}$ with
$1/[2(r_{\rm mom}-1)]<\omega_{\rm cl}<1/2$ makes the accumulated bias
vanish in $L_1$ and ensures Lindeberg.  

Full-gradient clipping depends on the current state, so the independent
deterministic-weight representation from \cref{sec:sharp-normal} cannot
be imported into this theorem.  For $b>1$ one distinguishes
\[
\frac1b\sum_{j=1}^b\operatorname{clip}_{\tau_{\rm cl}}(G_{k,j})
 \quad\hbox{and}\quad
 \operatorname{clip}_{\tau_{\rm cl}}\!\left(\frac1b\sum_{j=1}^bG_{k,j}\right).
\]
They have different predictable biases and conditional covariances.
At a fixed threshold, asymmetric noise can give $\mathfrak b_k^{\rm cl}\ne0$ even at
$u_k=0$, changing the target.  In the scalar location model, clipping
an individual quadratic-loss gradient is the derivative of a Huber
loss, so an offline robust estimating equation provides a natural
comparison.  \Cref{rem:clipping-sls} makes precise when such a potential
exists in general. Finite-time coverage, a verified accelerated boundary
bound, and comparison of online and offline procedures at the same
target and observation budget require further analysis.

\begin{remark}
\label{rem:clipping-sls}
The clipped martingale difference $\varepsilon^{\rm cl}_{k+1}$ in
\cref{eq:clipping-martingale} depends on the
query point whenever clipping is active with positive probability. A potential exists in two
important cases. With the Huber function
$\mathcal H_{\tau_{\rm H}}(r)=r^2/2$ for $|r|\le \tau_{\rm H}$ and
$\mathcal H_{\tau_{\rm H}}(r)=\tau_{\rm H}|r|-\tau_{\rm H}^2/2$ otherwise,
the vector location loss
$\ell(x,z)=\frac12\|x-z\|^2$ gives
$\operatorname{clip}_{\tau_{\rm H}}(x-z)
=\nabla_x\mathcal H_{\tau_{\rm H}}(\|x-z\|)$, and least squares,
$\ell(x;a_{\rm feat},y)=\frac12(a_{\rm feat}^\trans x-y)^2$, gives
$\operatorname{clip}_{\tau_{\rm H}}\{a_{\rm feat}(a_{\rm feat}^\trans x-y)\}
=\nabla_x\mathcal H_{\tau_{\rm H}/\|a_{\rm feat}\|}
  (a_{\rm feat}^\trans x-y)$,
a Huber loss with a design-dependent threshold. In both cases, the offline
counterpart is a robust $M$-estimator. For the adaptive Huber estimator, whose
threshold grows with the sample size, sub-Gaussian deviation bounds and a
nonasymptotic Bahadur representation hold under finite second moments
\citep{zhou2018new}; regression versions are available under weaker moment
assumptions \citep{sun2020adaptive}. In general, for $g=\nabla_x\ell$,
$r=\|g\|$ and
$\omega_{\tau_{\rm H}}(r)=\min\{1,\tau_{\rm H}/r\}$, the clipped field
$\omega_{\tau_{\rm H}}(r)g$ has Jacobian
$\omega_{\tau_{\rm H}}\nabla^2_x\ell
+\omega_{\tau_{\rm H}}'(r)r^{-1}gg^\trans\nabla^2_x\ell$;
where clipping is active,
this matrix is symmetric if and only if $g$ is an eigenvector of
$\nabla^2_x\ell$. Otherwise the clipped field has no potential, and the natural
offline counterpart is the estimating equation
$N^{-1}\sum_{i=1}^N
\operatorname{clip}_{\tau_{\rm H}}(\nabla_x\ell(x,Z_i))=0$, i.e., a
$Z$-estimator \citep{vandervaart1998asymptotic}, rather than a clipped
empirical risk.
\end{remark}

\FloatBarrier
\section{Additional Model Extensions}
\label{app:additional-extensions}

\dd{This section delineates the scope of the main finite-budget theory under
broader stochastic dynamics. It treats multiplicative noise, decreasing
steps and changing momentum, and Richardson--Romberg extrapolation, while
stating explicitly which conclusions are unconditional and which require
additional moment or stability assumptions.}

\subsection{Multiplicative Noise}
\label{sec:multiplicative-extension}

\begin{remark}[Multiplicative linear noise]
\label{rem:multiplicative-linear}
For general LSA, random matrices multiply the current state.  Then a fixed step can contribute a step-dependent covariance correction, as in \citet{mou2020linear}.  Our universality theorem is for additive gradient noise and for decreasing gains that localize the multiplicative component.  This boundary is important: it prevents the false conclusion that momentum is covariance-neutral in every linear stochastic approximation problem.
\end{remark}

\paragraph{The Oracle and Covariance Correction.}

The additive model has the same oracle covariance at the minimizer and
at every query point.  To separate these quantities, consider $b=1$ and
\begin{equation}
 G(x_\star+u,\zeta_{k+1})
 =Hu+\xi_{k+1}+\mathcal M_{k+1}u,\qquad
 \E\xi_{k+1}=0,\quad \E \mathcal M_{k+1}=0,
 \label{eq:multiplicative-oracle}
\end{equation}
where the pairs $(\xi_{k+1},\mathcal M_{k+1})$ are i.i.d. and independent of the
past $\cF_k$, and $u_k$ is $\cF_k$-measurable.  Within a pair $\xi$ and
$\mathcal M$ need not be independent.  Assume their second moments are finite.
For the two-step recursion driven by this oracle, summation gives
\begin{equation}
 \sqrt m\bar u_{0,m}
 =-\frac{H^{-1}}{\sqrt m}\sum_{k=0}^{m-1}\xi_{k+1}
  -\frac{H^{-1}}{\sqrt m}\sum_{k=0}^{m-1}\mathcal M_{k+1}u_k
  +\frac{H^{-1}\Delta_{0,m}}{\alpha\sqrt m}.
 \label{eq:multiplicative-boundary}
\end{equation}
The martingale difference is
$\xi_{k+1}+\mathcal M_{k+1}u_k$, whose conditional covariance is
\begin{align}
 \Sigma_{\rm traj}(u_k)={}&\Sigma
 +\E[\xi u_k^{\trans}\mathcal M^{\trans}
       +\mathcal M u_k\xi^{\trans}\mid\cF_k]
 +\E[\mathcal M u_ku_k^{\trans}\mathcal M^{\trans}\mid\cF_k].
 \label{eq:multiplicative-covariance}
\end{align}
Here the conditional expectations integrate a pair $(\xi,\mathcal M)$.
If the average of
$\Sigma_{\rm traj}(u_k)$ converges in probability to a deterministic
$\Sigma_{\rm eff}$,
conditional Lindeberg holds, and the scaled boundary vanishes, the
martingale CLT gives covariance $H^{-1}\Sigma_{\rm eff}H^{-1}$.

Writing $\|\mathcal M\|$ for the operator norm, the sufficient
localization condition
\[
 \frac1m\sum_{k=0}^{m-1}\E\|u_k\|^2\longrightarrow0
\]
makes the extra term in \cref{eq:multiplicative-boundary} negligible in
$L^2$: martingale orthogonality bounds its squared norm before the
factor $H^{-1}$ by
$\E\|\mathcal M\|^2\,m^{-1}\sum_k\E\|u_k\|^2$.
At a fixed step with nonzero additive noise, such localization need not
hold.  Consequently there is no general fixed-step covariance neutrality
for momentum in this model.

\begin{proposition}[A scalar fixed-step covariance correction]
\label{prop:multiplicative-scalar}
Consider stationary SGD
\[
 e_{k+1}=(1-\alpha(\lambda_0+\mathcal M_{k+1}))e_k-\alpha\xi_{k+1},
 \qquad \lambda_0>0,
\]
where the i.i.d. centered pairs have independent components, with
$\E \mathcal M^2=\sigma_{\mathcal M}^2$ and
$\E\xi^2=\sigma^2>0$.  Suppose
\[
 (1-\alpha\lambda_0)^2+\alpha^2\sigma_{\mathcal M}^2<1.
\]
There is a causal stationary solution with finite second moment
\[
 P=\E e_k^2
 =\frac{\alpha\sigma^2}
 {2\lambda_0-\alpha(\lambda_0^2+\sigma_{\mathcal M}^2)}.
\]
For this solution,
\begin{equation}
 \sqrt m\bar e_{0,m}\xrightarrow{\mathrm d}
 \cN\!\left(0,\frac{\sigma^2+\sigma_{\mathcal M}^2P}{\lambda_0^2}\right).
 \label{eq:multiplicative-scalar-clt}
\end{equation}
In particular the variance exceeds $\sigma^2/\lambda_0^2$ when
$\sigma_{\mathcal M}^2>0$.
\end{proposition}

\begingroup
\color{black}
\begin{proof}[Proof of \cref{prop:multiplicative-scalar}]
Put $a_k=1-\alpha(\lambda_0+\mathcal M_k)$ and
$q_{\rm ms}^2=\E a_k^2<1$. Extend the iid pairs to all integer times.
The series
\[
 e_k=-\alpha\sum_{j=0}^\infty
       \left(\prod_{r=0}^{j-1}a_{k-r}\right)\xi_{k-j}
\]
converges in $L^2$: the $L^2$ norm of its $j$th term is
$\alpha\sigma q_{\rm ms}^j$. Empty products equal one.
It defines a causal stationary solution. As a measurable function
of the shifted iid sequence, this solution is ergodic.
pairs are independent of $e_k$ and have independent centered
components; taking first and second moments of the recursion gives
$\E e_k=0$ and
\[
 P=q_{\rm ms}^2P+\alpha^2\sigma^2,\qquad
 P=\frac{\alpha\sigma^2}
         {2\lambda_0-\alpha(\lambda_0^2+\sigma_{\mathcal M}^2)}.
\]
Define the effective martingale difference
$\zeta_{k+1}^{\rm eff}=\xi_{k+1}+\mathcal M_{k+1}e_k$.
Its conditional variance is
$\sigma^2+\sigma_{\mathcal M}^2e_k^2$, whose average converges
almost surely to $\sigma^2+\sigma_{\mathcal M}^2P$ by ergodicity.
For every $\epsilon>0$, stationarity and its finite second moment give
\[
 \E\left[\frac1m\sum_{k=0}^{m-1}
  \E\!\left[(\zeta_{k+1}^{\rm eff})^2
    \mathbf 1_{\{|\zeta_{k+1}^{\rm eff}|>\epsilon\sqrt m\}}
       \mid\cF_k\right]\right]
 =\E\!\left[(\zeta_1^{\rm eff})^2
    \mathbf 1_{\{|\zeta_1^{\rm eff}|>\epsilon\sqrt m\}}\right]\to0.
\]
Thus conditional Lindeberg holds in probability.
Summing the scalar recursion gives
\[
 \lambda_0\sqrt m\,\bar e_{0,m}
 =-\frac1{\sqrt m}\sum_{k=0}^{m-1}\zeta_{k+1}^{\rm eff}
       +\frac{e_0-e_m}{\alpha\sqrt m}.
\]
The endpoint tends to zero in $L^2$ by stationarity.
The martingale CLT and Slutsky's theorem prove
\cref{eq:multiplicative-scalar-clt}. Since $\sigma^2>0$, one has
$P>0$, so the asserted variance inflation is strict when
$\sigma_{\mathcal M}^2>0$.
\end{proof}
\par
\endgroup

\paragraph{Entry Cost and Limits of Local Information.}
\begin{proposition}[Second moments of the accelerated multiplicative oracle]
\label{prop:ms-operator-new}
In the scalar model \cref{eq:multiplicative-oracle}, let
$\mathcal M,\xi$ be independent within a pair,
$\E \mathcal M^2=\sigma_{\mathcal M}^2$ and $\E\xi^2=\sigma^2$.
For any fixed parameters of the two-step procedure, set
$\mathsf M_\alpha=B_\alpha \mathsf O_\gamma$ and
$\mathcal K(P)=APA^\trans
+\sigma_{\mathcal M}^2\mathsf M_\alpha P\mathsf M_\alpha^\trans$.
For a deterministic $\mathbf z_0$ the uncentered second moment equals
\begin{equation}
 U_k=\E \mathbf z_k\mathbf z_k^\trans
 =\mathcal K^k(\mathbf z_0\mathbf z_0^\trans)
       +\sigma^2\sum_{j=0}^{k-1}\mathcal K^j
          (B_\alpha B_\alpha^\trans).
 \label{eq:ms-operator-new}
\end{equation}
A sufficient condition for the stability of second moments is
$\rho(A\otimes A
+\sigma_{\mathcal M}^2\mathsf M_\alpha\otimes\mathsf M_\alpha)<1$.
For $E_1=[1,0]$, the condition $E_1U_tE_1^\trans\le\varepsilon r^2$
guarantees $\Pp\{|e_t|>r\}\le\varepsilon$.
For staying in the region over the whole finite interval $t\le k\le T$
it suffices that $\sum_{k=t}^T E_1U_kE_1^\trans\le\varepsilon r^2$.
\end{proposition}

\begingroup
\color{black}
\begin{proof}[Proof of \cref{prop:ms-operator-new}]
In the scalar multiplicative model the state recursion is
\[
 \mathbf z_{k+1}
 =(A+\mathcal M_{k+1}\mathsf M_\alpha)\mathbf z_k
       +B_\alpha\xi_{k+1}.
\]
Condition on $\cF_k$. Centering removes the cross terms involving
$A$ and one noise, and independence within a pair removes
the term involving $\mathcal M_{k+1}\xi_{k+1}$.
Taking expectations therefore yields
$U_{k+1}=\mathcal K(U_k)+\sigma^2B_\alpha B_\alpha^\trans$.
Iteration proves \cref{eq:ms-operator-new}.
Under column vectorization, $\mathcal K$ is represented by
$A\otimes A+\sigma_{\mathcal M}^2\mathsf M_\alpha\otimes\mathsf M_\alpha$.
Spectral radius below one makes its powers decay geometrically
and its geometric series converge, proving bounded second moments.
Finally, $\E e_k^2=E_1U_kE_1^\trans$; Markov's inequality proves
the single-time certificate, and a union bound proves the
certificate for the whole finite interval.
\end{proof}
\par
\endgroup

This is a finite-dimensional check for SGD, HB and NAG; it accounts for the
acceleration parameters and for the global intensity $\sigma_{\mathcal M}$.
The smallest index
$t$ satisfying the chosen inequality is a certified burn-in in this class and
enters the budget $N=b(t+m)$. At a high stationary noise level such an index may
not exist. For independent averaging of $b$ pairs,
$\sigma_{\mathcal M}^2$ and $\sigma^2$ are replaced by
$\sigma_{\mathcal M}^2/b$ and $\sigma^2/b$ in the formulas.

\begin{proposition}[Local covariance does not bound the risk of this family]
\label{prop:local-insufficient-new}
Fix $\lambda_0>0$, $\sigma^2$, $e_{-1}=e_0\ne0$, $\alpha>0$,
$\beta\in[0,1)$ and $\gamma\ge0$. Over the class of centered independent pairs
$(\mathcal M,\xi)$ with arbitrary finite $\sigma_{\mathcal M}^2$, the risk of the two-point average
$(u_0+u_1)/2$ is not bounded by any finite function of
$(\lambda_0,\sigma^2,e_0,\alpha,\beta,\gamma)$.
\end{proposition}

\begingroup
\color{black}
\begin{proof}[Proof of \cref{prop:local-insufficient-new}]
The zero initial velocity gives $u_0=e_0$ and
$e_1=e_0-\alpha(\lambda_0e_0+\mathcal M_1e_0+\xi_1)$.
Hence
\[
 \frac{u_0+u_1}{2}
 =\left(1-\frac{\alpha(1+\gamma)\lambda_0}{2}\right)e_0
    -\frac{\alpha(1+\gamma)}2(\mathcal M_1e_0+\xi_1).
\]
For the independent centered components allowed in the class, its
variance equals
$\alpha^2(1+\gamma)^2
 (\sigma_{\mathcal M}^2e_0^2+\sigma^2)/4$.
It diverges as $\sigma_{\mathcal M}^2\to\infty$, whereas the
specified scalar parameters and the squared bias stay fixed.
The risk is at least this variance, which proves the claim.
\end{proof}
\par
\endgroup

\subsection{Decreasing Steps and Changing Momentum}
\label{sec:step-policies}

\subsubsection{Fixed Momentum: A Conditional Transfer}
Consider the normalized recursion
\begin{align}
 \Delta x_{k+1}&=\beta_k\Delta x_k
                  -(1-\beta_k)\alpha_kG(s_k,\zeta_{k+1}),
                                                    \label{eq:nonlinear-increment}\\
 x_{k+1}&=x_k+\Delta x_{k+1},\qquad
 s_k=x_k+\lambda_k^{\rm qry}\Delta x_k,            \label{eq:nonlinear-x}
\end{align}
where $\lambda_k^{\rm qry}=0$ for HB and
$\lambda_k^{\rm qry}=\beta_k$ for NAG. We average
$\bar x_n=n^{-1}\sum_{k=1}^n x_k$.  Write $e_k=x_k-x_\star$ and
$\xi_{k+1}=G(s_k,\zeta_{k+1})-F(s_k)$, with $F=\nabla f$.
These results are conditional transfer theorems: the stability and moment
bounds below are hypotheses, not conclusions.  Asymptotic equivalence of
averaged momentum and SGD has prior results, including
\citet{tang2023acceleration}; increasing-momentum averaging is also studied
by \citet{an2026polyakruppert}. 

\begin{theorem}[Fixed momentum with a H\"older linearization remainder]
\label{thm:nonlinear-pr}
Let $\beta_k=\beta\in[0,1)$ be fixed and
$\alpha_k=a_{\rm gain}(k+k_0)^{-\gamma_{\rm step}}$ with
$a_{\rm gain}>0$. Assume the following.
\begin{enumerate}
\item $F(x_\star)=0$, $H=\nabla F(x_\star)\succ0$, and, for some
$\nu\in(0,1]$, the queried points satisfy almost surely
\[
 \norm{F(s_k)-H(s_k-x_\star)}
 \le M_\nu\norm{s_k-x_\star}^{1+\nu}.
\]
\item $\xi_{k+1}$ is a martingale difference, its conditional covariance
converges in probability to the deterministic matrix $\Sigma_\star$, and
$\sup_k\E\norm{\xi_{k+1}}^{2+\varepsilon_{\rm mom}}<\infty$
for some $\varepsilon_{\rm mom}>0$.
\item For a constant $C_{\rm mom}$ and all $k$, including the finite starting indices,
\[
 \E\norm{e_k}^2\le C_{\rm mom}\alpha_k,\qquad
 \E\norm{\Delta x_k}^2\le C_{\rm mom}\alpha_k^2.
\]
\end{enumerate}
If $1/(1+\nu)<\gamma_{\rm step}<1$, then
\begin{equation}
 \sqrt n(\bar x_n-x_\star)
 =-H^{-1}\frac1{\sqrt n}\sum_{k=1}^n\xi_{k+1}+R_n,
                                                    \label{eq:nonlinear-expansion}
\end{equation}
where
\begin{equation}
 R_n=\Op\!\left(n^{-(1-\gamma_{\rm step})/2}
       +n^{-[\gamma_{\rm step}(1+\nu)-1]/2}\right).
                                                    \label{eq:nonlinear-rem}
\end{equation}
Consequently,
\begin{equation}
 \sqrt n(\bar x_n-x_\star)
 \xrightarrow{\mathrm d}\cN(0,H^{-1}\Sigma_\star H^{-1}).
                                                    \label{eq:nonlinear-clt}
\end{equation}
If $F(x)=H(x-x_\star)$ at every queried point, the curvature term vanishes
and the conclusion holds for every $0<\gamma_{\rm step}<1$ under the same moment
assumptions.
\end{theorem}

The usual locally Lipschitz Hessian case is $\nu=1$, giving the original
$\gamma_{\rm step}\in(1/2,1)$ range. Balancing the two displayed remainder
powers gives $\gamma_{\rm step}=2/(2+\nu)$ and rate
$n^{-\nu/[2(2+\nu)]}$.

For fixed momentum, the familiar corrected state is
\begin{equation}
 x_k^{\rm corr}=x_k+\frac{\beta}{1-\beta}\Delta x_k,\qquad
 x_{k+1}^{\rm corr}=x_k^{\rm corr}-\alpha_kG(s_k,\zeta_{k+1}).
                                                    \label{eq:corrected-state}
\end{equation}
The important improvement over bounding velocities separately is
\[
 \sum_{k=1}^n \Delta x_k=x_n-x_0.
\]
Thus the linear query mismatch is an endpoint

\begingroup
\color{black}

\begin{proof}[Proof of \cref{thm:nonlinear-pr}]
For this proof put $\gamma=\gamma_{\rm step}$ and
$\|Y\|_{L^2}=(\E\|Y\|^2)^{1/2}$. All constants below may depend on
the fixed parameters in the assumptions but not on $n$.
Define
\[
 \widetilde e_k=x_k^{\rm corr}-x_\star
    =e_k+\frac{\beta}{1-\beta}\Delta x_k,
 \qquad
 \mathfrak r_k=F(s_k)-H(s_k-x_\star).
\]
The moment assumptions give
$\|\widetilde e_k\|_{L^2}=O(\alpha_k^{1/2})$ and
$\E\|s_k-x_\star\|^2=O(\alpha_k)$. Here
$\lambda^{\rm qry}=0$ for HB and $\lambda^{\rm qry}=\beta$ for NAG
is constant. The corrected-state recursion yields the identity
\begin{equation}
 H\sum_{k=1}^n e_k
 =T_n-\sum_{k=1}^n\xi_{k+1}-\sum_{k=1}^n\mathfrak r_k
      -H\lambda^{\rm qry}(e_n-e_0),
 \qquad
 T_n=\sum_{k=1}^n\frac{\widetilde e_k-\widetilde e_{k+1}}{\alpha_k}.
 \label{eq:fixed-momentum-proof-identity}
\end{equation}
In particular, no sum of the norms of the increments is needed for
the linear query mismatch.

Summation by parts gives
\begin{equation}
 T_n=\frac{\widetilde e_1}{\alpha_1}
       -\frac{\widetilde e_{n+1}}{\alpha_n}
       +\sum_{k=2}^n
          (\alpha_k^{-1}-\alpha_{k-1}^{-1})\widetilde e_k.
 \label{eq:fixed-momentum-abel-sum}
\end{equation}
Since $\alpha_k^{-1}-\alpha_{k-1}^{-1}
=O((k+k_0)^{\gamma-1})$, Minkowski's inequality shows
\[
 \|T_n\|_{L^2}
 \le C\left(1+n^{\gamma/2}
               +\sum_{k=2}^n(k+k_0)^{\gamma/2-1}\right)
 =O(n^{\gamma/2}).
\]
Also $\|e_n-e_0\|_{L^2}=O(1)$. By the Taylor bound and
$1+\nu\le2$,
\[
 \E\|\mathfrak r_k\|
 \le M_\nu(\E\|s_k-x_\star\|^2)^{(1+\nu)/2}
 \le C\alpha_k^{(1+\nu)/2}.
\]
For every $0<\gamma<1$ and $0<\nu\le1$ the exponent
$\gamma(1+\nu)/2$ is strictly below one. Hence
\[
 \frac1{\sqrt n}\sum_{k=1}^n\E\|\mathfrak r_k\|
 =O\bigl(n^{[1-\gamma(1+\nu)]/2}\bigr).
\]
Multiplying \cref{eq:fixed-momentum-proof-identity} by
$H^{-1}/\sqrt n$, and using Markov's inequality for the curvature
sum and the $L^2$ bounds for the other terms, proves
\cref{eq:nonlinear-expansion,eq:nonlinear-rem}.

For completeness, let $C_k=\E[\xi_{k+1}\xi_{k+1}^\trans\mid
\mathcal F_k]$, with $\mathcal F_k$ the history before that noise is
observed. The uniform $(2+\varepsilon_{\rm mom})$ moment implies
uniform integrability of $C_k$. Its assumed convergence in probability
therefore implies $L^1$ convergence and
$n^{-1}\sum_{k=1}^n C_k\to\Sigma_\star$ in probability.
For every $\epsilon>0$,
\[
 \frac1n\sum_{k=1}^n
 \E\bigl[\|\xi_{k+1}\|^2
       \mathbf 1_{\{\|\xi_{k+1}\|>\epsilon\sqrt n\}}\bigr]
 \le C\epsilon^{-\varepsilon_{\rm mom}}
              n^{-\varepsilon_{\rm mom}/2}\longrightarrow0.
\]
Thus the conditional Lindeberg sum tends to zero in probability.
The martingale CLT, applied to every fixed scalar projection, gives
$n^{-1/2}\sum_k\xi_{k+1}\xrightarrow{\mathrm d}
\cN(0,\Sigma_\star)$. When $\gamma(1+\nu)>1$, both remainder
powers vanish, and Slutsky's theorem proves \cref{eq:nonlinear-clt}.
If $F$ is linear, $\mathfrak r_k=0$ identically, so every
$0<\gamma<1$ is allowed under the same moment assumptions.
\end{proof}
\par
\endgroup

\subsubsection{Batching and the Cost of Localization}
\begin{corollary}[Batch-dependent moments remove an artificial batch cap]
\label{cor:nonlinear-batch}
For each horizon $n$, run the fixed-momentum recursion with a batch
$b=b_n\ge1$.  Assume the Taylor bound in \cref{thm:nonlinear-pr} holds
uniformly, and that for a constant $C_{\rm mom}$ independent of $n$, $b$, and
$k$,
\begin{equation}
 \E\norm{e_{k,b}}^2\le C_{\rm mom}\alpha_k/b,\qquad
 \E\norm{\Delta x_{k,b}}^2\le C_{\rm mom}\alpha_k^2/b
                                                    \label{eq:batch-scaled-moments}
\end{equation}
for all starting and subsequent indices, uniformly in $n$.
The individual oracle noises in each batch are conditionally independent
martingale differences with uniformly bounded
$(2+\varepsilon_{\rm mom})$ moments;
their averaged conditional covariance converges in probability to
$\Sigma_\star$.  Then, for $0<\gamma_{\rm step}<1$,
\begin{align}
 \sqrt{nb}(\bar x_{n,b}-x_\star)
 &=-\frac{H^{-1}}{\sqrt{nb}}
   \sum_{k=1}^n\sum_{j=1}^b\xi_{k+1,j}+R_{n,b},\notag\\
 R_{n,b}
 &=\Op\!\left(n^{-(1-\gamma_{\rm step})/2}
       +b^{-\nu/2}n^{[1-\gamma_{\rm step}(1+\nu)]/2}\right).
                                                    \label{eq:nonlinear-batch}
\end{align}
Hence the same sandwich CLT holds whenever
$b^{-\nu/2}n^{[1-\gamma_{\rm step}(1+\nu)]/2}\to0$.
\end{corollary}

\begingroup
\color{black}
In the triangular-array assumption of \cref{cor:nonlinear-batch},
\emph{averaged conditional covariance} means the row-wise condition
\[
 \frac1{nb_n}\sum_{k=1}^n\sum_{j=1}^{b_n}
 \E\!\left[\xi^{(n)}_{k+1,j}
       (\xi^{(n)}_{k+1,j})^\trans\mid\mathcal F_{k,n}\right]
 \xrightarrow{\Pp}\Sigma_\star,
\]
where $\mathcal F_{k,n}$ is the history before the $k$th batch
in row $n$. Conditional independence within a batch preserves
these conditional moments when its individual noises are revealed
successively.
\par
\endgroup

\begingroup
\color{black}
\begin{proof}[Proof of \cref{cor:nonlinear-batch}]
Put $\gamma=\gamma_{\rm step}$ and
$\bar\xi_{k+1,b}=b^{-1}\sum_{j=1}^b\xi_{k+1,j}$.
The pathwise identity \cref{eq:fixed-momentum-proof-identity} applies
with this batch noise. Under \cref{eq:batch-scaled-moments}, its
corrected errors and query errors satisfy, uniformly in the array,
\[
 \|\widetilde e_{k,b}\|_{L^2}\le C\sqrt{\alpha_k/b},\qquad
 \E\|s_{k,b}-x_\star\|^2\le C\alpha_k/b.
\]
The same Abel sum and endpoint cancellation consequently give
\[
 \sqrt{\frac bn}\,\|T_{n,b}\|_{L^2}
 =O\bigl(n^{-(1-\gamma)/2}\bigr),\qquad
 \sqrt{\frac bn}\,\|e_{n,b}-e_{0,b}\|_{L^2}=O(n^{-1/2}).
\]
The uniform moment bound at the starting indices is used in both
estimates. For the curvature remainder,
\[
 \sqrt{\frac bn}\sum_{k=1}^n\E\|\mathfrak r_{k,b}\|
 \le C\sqrt{\frac bn}\sum_{k=1}^n
                    (\alpha_k/b)^{(1+\nu)/2}
 =O\bigl(b^{-\nu/2}n^{[1-\gamma(1+\nu)]/2}\bigr).
\]
Multiplication of the pathwise identity by $H^{-1}\sqrt{b/n}$
now proves the expansion and remainder in \cref{eq:nonlinear-batch}.

To obtain the CLT, reveal the conditionally independent individual
noises within each batch one at a time. They form a martingale
triangular array whose normalized sum is
$(nb)^{-1/2}\sum_{k,j}\xi_{k+1,j}$. The assumed averaged
conditional-covariance limit gives its predictable covariance
convergence to $\Sigma_\star$. The uniform individual moment bound
implies, for every $\epsilon>0$,
\[
 \frac1{nb}\sum_{k=1}^n\sum_{j=1}^b
 \E\bigl[\|\xi_{k+1,j}\|^2
    \mathbf 1_{\{\|\xi_{k+1,j}\|>\epsilon\sqrt{nb}\}}\bigr]
 \le C\epsilon^{-\varepsilon_{\rm mom}}
              (nb)^{-\varepsilon_{\rm mom}/2}\longrightarrow0.
\]
This verifies conditional Lindeberg in probability. The martingale
CLT and Slutsky's theorem prove the sandwich limit whenever the
displayed curvature term vanishes; the endpoint term already tends
to zero for $0<\gamma<1$.
\end{proof}
\par
\endgroup

For $\nu=1$ and $\gamma_{\rm step}>1/2$, this condition imposes no upper bound on
batch growth. The condition $b=o(n^{1/3})$ at $\gamma_{\rm step}=2/3$, which follows from
unscaled uniform moment bounds, is a sufficient condition.
The stronger \cref{eq:batch-scaled-moments} excludes a fixed nonzero cold
start as $b\to\infty$.  A warm start or a separately proved burn-in should
supply those bounds.  The statistical gain does not pay for that
initialization for free.

\paragraph{Limits of the Statement.}
A local Taylor bound can replace the bound at all query points for a
single infinite trajectory if an almost surely finite time exists after
which every query stays in the neighborhood; the finite-prefix remainder
is then $O_{\Pp}(n^{-1/2})$. Triangular batch arrays need
a corresponding uniform argument. At $\alpha_k\asymp1/k$, arithmetic
averaging is generally inefficient even in one dimension: a factor-two variance can already occur in a scalar linear example.
For a fixed positive gain on a nonlinear problem, stationary averaging
generally targets a biased stationary mean. 

\subsubsection{Verified Moment Bounds for SGD}
\begin{proposition}[Verification of the moment assumptions for SGD]
\label{prop:global-sgd-new}
Let $f\in C^2$, $\mu I\preceq\nabla^2 f(x)\preceq LI$ for all $x$, let the
Hessian be globally Lipschitz, and let the noise be independent, additive,
centered, with a finite moment of order $2+\varepsilon_{\rm mom}$ and
covariance $\Sigma$.
Let $\sigma^2=\tr\Sigma$, $b\ge1$, $\beta=0$, $\Delta x_0=0$,
and let the nonincreasing steps satisfy
\begin{equation}
 0<\alpha_k\le 1/L,\qquad
 \alpha_{k+1}\ge \alpha_k-\tfrac\mu2 \alpha_k^2.
 \label{eq:step-slow-new}
\end{equation}
Then, for $C_{\rm sgd}\ge2\sigma^2/\mu$ and any deterministic start,
\begin{equation}
 \E\|e_k\|^2\le\|e_0\|^2\prod_{i=0}^{k-1}(1-\mu\alpha_i)^2
                     +C_{\rm sgd}\alpha_k/b.\label{eq:global-moment-new}
\end{equation}
If $\alpha_{\rm first}:=\alpha_k|_{k=0}$ and
$\E\|e_0\|^2\le C_{\rm sgd}\alpha_{\rm first}/b$, the first term can be dropped.
With such a start, $\E\|\Delta x_{k+1}\|^2
\le \alpha_k^2(L^2C_{\rm sgd}\alpha_k+\sigma^2)/b$.
For $\alpha_k=a_{\rm sgd}(k+k_0)^{-\gamma_{\rm step}}$ with $a_{\rm sgd}>0$,
$1/2<\gamma_{\rm step}<1$, the conditions
\cref{eq:step-slow-new} are ensured by a sufficiently large $k_0$;
hence the assumptions of \cref{thm:nonlinear-pr} with $\nu=1$
hold for every fixed $b$, including a cold start.
For an array of growing $b$, the warm-start condition gives the uniform moment
bounds \cref{eq:batch-scaled-moments}.
\end{proposition}

\begingroup
\color{black}
\begin{proof}[Proof of \cref{prop:global-sgd-new}]
Write
$\nabla f(x_k)=H_k e_k$ with
$H_k=\int_0^1\nabla^2f(x_\star+s e_k)\,ds$.
Its eigenvalues lie in $[\mu,L]$, and $0<\alpha_k\le1/L$ gives
$\|I-\alpha_kH_k\|\le1-\mu\alpha_k$.
Thus, for $p_k=\E\|e_k\|^2$,
\[
 p_{k+1}\le(1-\mu\alpha_k)^2p_k+\alpha_k^2\sigma^2/b.
\]
The slow-step condition and $\mu\alpha_k\le1$ imply
\[
 \begin{split}
 C_{\rm sgd}\alpha_{k+1}
   -C_{\rm sgd}\alpha_k(1-\mu\alpha_k)^2
 &\ge C_{\rm sgd}\mu\alpha_k^2(3/2-\mu\alpha_k)\\
 &\ge C_{\rm sgd}\mu\alpha_k^2/2
 \ge\sigma^2\alpha_k^2.
 \end{split}
\]
Induction, separating the homogeneous recursion from its noise
term, proves \cref{eq:global-moment-new}. If
$p_0\le C_{\rm sgd}\alpha_{\rm first}/b$, the same induction starts
directly with this envelope and removes the homogeneous term.
Independence and centering of the noise also give
\[
 \E\|\Delta x_{k+1}\|^2
 =\alpha_k^2\{\E\|\nabla f(x_k)\|^2+\sigma^2/b\}
 \le\alpha_k^2(L^2p_k+\sigma^2/b),
\]
which is the claimed velocity bound.

For $\alpha_k=a_{\rm sgd}(k+k_0)^{-\gamma_{\rm step}}$,
the mean value theorem gives
\[
 \frac{\alpha_k-\alpha_{k+1}}{\alpha_k^2}
 \le\frac{\gamma_{\rm step}}{a_{\rm sgd}}
          (k+k_0)^{\gamma_{\rm step}-1}.
\]
A sufficiently large $k_0$ ensures both conditions in
\cref{eq:step-slow-new}. Moreover,
$\prod_{i<k}(1-\mu\alpha_i)^2
 \le\exp(-2\mu\sum_{i<k}\alpha_i)$ decays faster than any inverse
power of $k$ when $\gamma_{\rm step}<1$.
For fixed $b$, it can therefore be absorbed into an
$O(\alpha_k)$ moment bound even from a cold start.
The velocity bound, the bounded adjacent step ratios, and the
globally Lipschitz Hessian verify the remaining moment and Taylor
assumptions of \cref{thm:nonlinear-pr} with $\nu=1$.
For growing batches, the warm-start envelope gives the same
bounds uniformly with factor $1/b$, as required
by \cref{eq:batch-scaled-moments}.
\end{proof}
\par
\endgroup

With a fixed cold start and a growing batch, the remainder in
\cref{eq:global-moment-new} has to be paid for by burn-in. It suffices to ensure
$2\mu\sum_{i<t}\alpha_i\ge
\log(\|e_0\|^2b/(C_{\rm sgd}\alpha_t))$ in order to suppress the initial
contribution to $C_{\rm sgd}\alpha_t/b$; the constant of the final moment bound
can then be doubled.
For example, the gradient of the nonlinear experiment
$F(x)=x+0.4x^2/(1+x^2)$ satisfies
$1-3\sqrt3/20\le F'(x)\le1+3\sqrt3/20$ and $|F''(x)|\le0.8$.
For $\alpha_k=1.2(k+40)^{-2/3}$ both conditions
\cref{eq:step-slow-new} hold: the maximal step is below $1/L$, and
$(\alpha_k-\alpha_{k+1})/\alpha_k^2\le(5/9)40^{-1/3}<\mu/2$.
Thus the moment assumptions of the nonlinear experiment for SGD with this
policy are verified globally.

\subsubsection{Increasing Momentum: A Conditional Transfer}
We use the assumptions of \cref{thm:nonlinear-pr}.
The exponent $\vartheta_{\rm vel}$ below controls the velocity moment and is
unrelated to the initial signal-to-noise ratio $\theta=\mu e_0/\sigma$.
\begin{theorem}[Direct Abel transfer for increasing momentum]
\label{thm:varying-momentum}
Retain assumptions (i)--(ii) of \cref{thm:nonlinear-pr} and
$\E\norm{e_k}^2\le C_{\rm mom}\alpha_k$. Let
\[
 \alpha_k=a_{\rm mom}(k+k_0)^{-\gamma_{\rm step}},\qquad
 \varpi_k=1-\beta_k=\varpi_0(k+k_0)^{-\rho_{\rm mom}}\in(0,1],
\]
with $a_{\rm mom},\varpi_0>0$, $0<\gamma_{\rm step}<1$, and
$\rho_{\rm mom}\ge0$, and suppose, for some
$\vartheta_{\rm vel}\in[0,1]$,
\begin{equation}
 \E\norm{\Delta x_k}^2
 \le C_{\rm mom}\alpha_k^2\varpi_k^{\vartheta_{\rm vel}}.
                                                    \label{eq:velocity-scale}
\end{equation}
For HB or NAG, \cref{eq:nonlinear-expansion} holds with
\begin{equation}
 R_n=\Op\!\left(
 n^{-(1-\gamma_{\rm step})/2}
 +n^{-[\gamma_{\rm step}(1+\nu)-1]/2}
 +n^{\rho_{\rm mom}(1-\vartheta_{\rm vel}/2)-1/2}\right).
                                                    \label{eq:varying-remainder}
\end{equation}
In particular, the CLT \cref{eq:nonlinear-clt} holds if
\[
 \frac1{1+\nu}<\gamma_{\rm step}<1,\qquad
 \rho_{\rm mom}<\frac1{2-\vartheta_{\rm vel}}.
\]
The curvature term and its restriction on $\gamma_{\rm step}$ are omitted for an
linear $F$.
\end{theorem}

\begingroup
\color{black}
\begin{proof}[Proof of \cref{thm:varying-momentum}]
Use the local abbreviations $\gamma=\gamma_{\rm step}$,
$r=\rho_{\rm mom}$, and $\vartheta=\vartheta_{\rm vel}$.
If $r=0$, the momentum is fixed and the estimates in the proof of
\cref{thm:nonlinear-pr} apply; those remainder estimates hold for
every $0<\gamma<1$, before requiring their decay. Suppose now that
$r>0$. Set
\[
 c_k=\alpha_k^{-1},\qquad
 d_k=(\alpha_k\varpi_k)^{-1},\qquad
 \mathfrak r_k=F(s_k)-H(s_k-x_\star).
\]
Rearrangement of \cref{eq:nonlinear-increment} gives, ,
\[
 F(s_k)+\xi_{k+1}
 =-c_k\Delta x_k-d_k(\Delta x_{k+1}-\Delta x_k).
\]
Consequently
\begin{equation}
\begin{aligned}
 H\sum_{k=1}^n e_k
 &=-\sum_{k=1}^n\xi_{k+1}-A_n-D_n
       -\sum_{k=1}^n\mathfrak r_k-HQ_n,\\
 A_n&=\sum_{k=1}^n c_k\Delta x_k,\qquad
 D_n=\sum_{k=1}^n d_k(\Delta x_{k+1}-\Delta x_k),\qquad
 Q_n=\sum_{k=1}^n\lambda_k^{\rm qry}\Delta x_k.
\end{aligned}
\label{eq:varying-momentum-proof-identity}
\end{equation}
Apply summation by parts first to the errors and then to the
increments:
\begin{align*}
 A_n&=c_ne_n-c_1e_0
       -\sum_{k=1}^{n-1}(c_{k+1}-c_k)e_k,\\
 D_n&=d_n\Delta x_{n+1}-d_1\Delta x_1
       -\sum_{k=2}^n(d_k-d_{k-1})\Delta x_k.
\end{align*}
Since $c_{k+1}-c_k=O((k+k_0)^{\gamma-1})$ and
$\|e_k\|_{L^2}=O((k+k_0)^{-\gamma/2})$, the first line yields
$\|A_n\|_{L^2}=O(n^{\gamma/2})$.
For the second line use
\[
 d_k-d_{k-1}=O((k+k_0)^{\gamma+r-1}),\qquad
 \|\Delta x_k\|_{L^2}
       =O((k+k_0)^{-\gamma-r\vartheta/2}).
\]
Writing $a=r(1-\vartheta/2)>0$, this gives
\[
 \|D_n\|_{L^2}
 \le C\left(1+n^a+\sum_{k=2}^n(k+k_0)^{a-1}\right)
 =O(n^a).
\]

For HB, $Q_n=0$. For NAG, $\lambda_k^{\rm qry}=1-\varpi_k$,
so a further telescoping identity gives
\[
 Q_n=e_n-e_0-
 \left\{\varpi_ne_n-\varpi_1e_0
       +\sum_{k=1}^{n-1}(\varpi_k-\varpi_{k+1})e_k\right\}.
\]
The sequence $\varpi_k$ is nonincreasing and bounded, its successive
differences have a bounded sum, and $\sup_k\|e_k\|_{L^2}<\infty$.
Thus $\|Q_n\|_{L^2}=O(1)$, retaining the cancellation of the
linear query mismatch even when momentum varies.

Finally $\varpi_k^{\vartheta}\le1$ implies
$\E\|s_k-x_\star\|^2=O(\alpha_k)$. The same Taylor-remainder
estimate as in the fixed-momentum proof gives
\[
 \frac1{\sqrt n}\sum_{k=1}^n\E\|\mathfrak r_k\|
 =O\bigl(n^{[1-\gamma(1+\nu)]/2}\bigr).
\]
Multiply \cref{eq:varying-momentum-proof-identity} by $H^{-1}/\sqrt n$.
The bounds for $A_n$, $D_n$, $Q_n$, and the curvature sum prove
\cref{eq:varying-remainder}; the $O_{\Pp}(n^{-1/2})$ query endpoint
is absorbed by the first term. The noise assumptions give the same
martingale CLT as in \cref{thm:nonlinear-pr}. All remainders vanish
when $\gamma(1+\nu)>1$ and
$r(1-\vartheta/2)<1/2$, which is the stated range.
For an linear $F$, $\mathfrak r_k=0$ and its restriction on
$\gamma$ disappears.
\end{proof}
\par
\endgroup

With the basic velocity estimate ($\vartheta_{\rm vel}=0$), this gives the
conditional range $\rho_{\rm mom}<1/2$, with no additional restriction coupling
$\rho_{\rm mom}$ and
$\gamma_{\rm step}$.
With the stronger estimate
$\E\norm{\Delta x_k}^2=O(\alpha_k^2\varpi_k)$, the conditional range is
$\rho_{\rm mom}<1$.
The latter estimate is a substantial assumption when momentum tends to
one.  This theorem does not establish it for general nonlinear dynamics,
and it does not prove that the boundary $\rho_{\rm mom}=1$ is necessary.

\subsection{Richardson--Romberg: A Different Correction Mechanism}
\label{sec:rr-context}
\citet[Theorem~6, Corollary~7, equations~(36)--(41)]{sheshukova2025richardson}
analyze two coupled constant-step SGD trajectories with tail averaging.
With step $\alpha_{\rm RR}=n^{-1/2}$, their bound on
$\{\E\|H_\star(\widehat x^{\rm RR}-x_\star)\|^2\}^{1/2}$ has
a leading $n^{-1/2}$ term, an $O(n^{-3/4})$ remainder, and an
initialization term; its hypotheses include sixth-moment controls and
fourth-order smoothness. The exponent $3/4$ is described as best known,
not as a universal lower bound. Squaring an upper bound
$K_1 n^{-1/2}+K_2 n^{-3/4}$ also produces
$2K_1K_2 n^{-5/4}$, so it does not
identify an $n^{-3/2}$ first correction to the MSE.
Their shared data stream preserves the leading noise term in
$2\widehat x_{\alpha_{\rm RR}}-\widehat x_{2\alpha_{\rm RR}}$.
For independent copies with leading covariance $\Sigma_{\rm lead}/n$,
elementary covariance algebra instead gives $5\Sigma_{\rm lead}/n$.
Their convention uses $2n$ observations including burn-in and
two gradient evaluations per observation. This correction targets
nonlinear stationary bias; that bias is absent in our additive quadratic
model, whose initialization and covariance transients remain relevant.

\section{Details of the Census Experiments}
\label{app:experiments-census}

\dd{This section gives the complete protocol for the Census experiments
summarized in the main text. It specifies the data processing, stochastic
oracles, algorithmic parameters, interval construction, and additional
numerical results needed for reproduction.}

\subsection{Data, Oracle, and Parameters}
\label{app:census-setup}

\paragraph{Data.}
The file \texttt{adult\_reconstruction.csv} of \citet{ding2021retiring}
contains the $49{,}531$ persons of the 1994 Current Population Survey from
which the UCI Adult data were drawn, with their total annual income (reported
up to $\$99{,}999$). Age, usual weekly hours, and years of education are
standardized; class of worker, marital status, relationship, race, sex,
occupation, and region of birth enter as indicators, with levels rarer than
$0.1\%$ merged. Write $(z_i,y_i)$ for the encoded covariate--response pair of
record $i$, with $z_i\in\R^{44}$ including the intercept. The linear model
uses log income as response; the logistic model uses the indicator that income
exceeds $\$50{,}000$. All coefficients are penalized with a ridge level
$\lambda$; besides the cross-validated level we use two larger levels that
reduce the condition number by factors $10$ and $100$
(\cref{tab:census-constants}).
The reconstruction is governed by the
\href{https://ipums.org/about/terms}{IPUMS terms of use}; the accompanying
Folktables code is released under the MIT license. The supplementary code downloads the data from its cited source and does not redistribute it. The file contains demographic attributes, including race and sex. We use them only as covariates in aggregate algorithmic simulations, report no individual-level results.

\paragraph{Prediction-Powered Oracles.}
A pilot sample of $10^4$ persons, drawn once and independently of all
replicates, is used to fit gradient-boosting predictors $\hat f$ of log income
and $\hat p$ of the income indicator (\texttt{scikit-learn}, default
hyperparameters). Let $\Gamma_z:=\E[zz^\trans]$. With
$r=y-\hat f(z)$ and the identity
$(\Gamma_z+\lambda I)x_\star=\E[zy]$, the linear oracle of
\cref{sec:numerical-summary} equals
\[
 g_i(x)=(\Gamma_z+\lambda I)(x-x_\star)
 -\bigl(z_ir_i-\E[zr]\bigr),
\]
so its noise has covariance $\Sigma=\Cov(zr)$ and does not depend on $x$; a
batch averages $b$ independent draws with replacement, and
$\Vstar=H^{-1}\Sigma H^{-1}$ with
$H=\Gamma_z+\lambda I$. Define
$\operatorname{sigm}(t):=(1+e^{-t})^{-1}$. The logistic oracle is
$\E[z\operatorname{sigm}(z^\trans x)]+\lambda x
-\E[z\hat p(z)]-z_i\{y_i-\hat p(z_i)\}$; its global constants in
\cref{thm:nonlinear-nag-closed} are $\mu=\lambda$ and
$L=\lambda+\lambda_{\max}(\Gamma_z)/4$. All expectations are finite-population
averages over the file.

\begin{table}[H]
\caption{Constants of the Census Experiments}
\label{tab:census-constants}
\begin{center}
\small
\begin{tabular}{@{}lccc@{}}
\toprule
\textbf{Quantity} & \textbf{Cross-validated $\boldsymbol{\lambda}$} &
\textbf{Second $\boldsymbol{\lambda}$} & \textbf{Third $\boldsymbol{\lambda}$}\\
\midrule
$\kappa$ of $H$, linear model & $1.5\times10^4$ & $1532$ & $153$\\
$\kappa$ of $H$ at $x_\star$, logistic model & $4.1\times10^4$ & $209$ & $23$\\
$L/\mu$ with global constants, logistic model & $1.3\times10^5$ & $425$ & $39$\\
Variance of log income explained by $\hat f$ & \multicolumn{3}{c}{$66\%$}\\
Reduction of $\tr\Sigma$ by the control variate & \multicolumn{3}{c}{factor $168$}\\
\bottomrule
\end{tabular}
\end{center}
\end{table}

\subsection{Methods, Intervals, and Protocol}
\label{app:census-protocol}

\paragraph{Methods.}
SGD, HB, and NAG use the policies of \cref{tab:hard-mode-parameters} with
$\mu=\lambda_{\min}(H)$ and $L=\lambda_{\max}(H)$, zero initial velocity, and
the start $x_0$ with intercept equal to the mean pilot response and all other
coefficients zero. The output is the full average of the query points. In the
logistic experiment, which uses the third ridge level and $x_0=0$, NAG uses the
burn-in \cref{eq:nag-closed-burnin} with $\mathcal V_0$ computed on the file and
$\sigma^2=\tr\Sigma$, followed by tail averaging; constant-step SGD uses the
analogous burn-in with $\kappa^{-1/2}$ replaced by $\kappa^{-1}$. The batch
schedules for NAG are $b=m^2$ ($m=4,\dots,128$), $b=m$ ($m=8,\dots,1024$), and
$b=64$ ($m=16,\dots,1024$), and SGD uses $b=m^2$ ($m=4,\dots,64$); every
acquired response, including those spent in burn-in,
is counted in $N$ and used by the ERM.

\paragraph{Intervals.}
All intervals have level $95\%$ and are centered at the output. The
oracle interval uses $\Vstar$. The plug-in interval replaces $\Sigma$ by the
empirical covariance of the acquired terms $z_ir_i$ (the matrix $H$ is known in
this setting). The batch-means interval splits the $m$ query points into
$\lfloor\sqrt m\rfloor$ consecutive blocks of equal length and uses Student
quantiles; equal blocks are the constant-step counterpart of the construction
of \citet{chen2020statistical}. Random scaling \citep{lee2022fast} uses the
partial averages $\bar u_k:=k^{-1}\sum_{\ell=0}^{k-1}u_\ell$ through
$\hat V=m^{-2}\sum_{k\le m}k^2(\bar u_k-\bar u_m)(\bar u_k-\bar u_m)^\trans$ and
the critical value $6.747$. Batch means and random scaling rely on a functional
central limit theorem for the partial sums of the query points. Applying
\cref{thm:exact} separately to each prefix gives only pointwise decompositions;
a functional result also requires uniform endpoint control and tightness. We
therefore include these two intervals as empirical comparators and do not claim
their asymptotic validity from \cref{thm:exact} alone.

\paragraph{Gaussian Coverage Predictions from Finite-Sample Moments.}
The mean of the output is obtained by running the noiseless recursion from
$x_0$. For its variance, the recursion decouples in the eigenbasis
$(\lambda_i,\varphi_i)$ of $H$; with $r_n(\lambda)$ the response of
\cref{thm:cold-exact}, for a fixed contrast $c\in\R^d$,
\[
 \Var(c^\trans\bar u_{0,m})=\frac{1}{bm^2}\sum_{s=1}^{m}w_s^\trans\Sigma w_s,
 \qquad
 w_s=\sum_{i=1}^d\frac{1-r_{m-s}(\lambda_i)}{\lambda_i}
       (\varphi_i^\trans c)\,\varphi_i,
\]
which reduces to \cref{eq:cold-exact-risk} in one dimension. The predicted
coverage is $\Phi((1.96-h_c)/\sqrt{v_c})-\Phi((-1.96-h_c)/\sqrt{v_c})$, where
\[
 h_c:=\frac{\E[c^\trans\bar u_{0,m}]}{\sqrt{c^\trans V_\star c/N}},
 \qquad
 v_c:=\frac{\Var(c^\trans\bar u_{0,m})}{c^\trans V_\star c/N}.
\]
We call a batch size the
breakdown batch if coverage stays below $90\%$ from it on. This coverage
formula is under Gaussian oracle noise; for Census resampling it is a
Gaussian approximation using the finite-sample mean and variance.

\paragraph{Protocol.}
Each configuration of the linear experiments is repeated $500$ times and each
configuration of the logistic experiment $150$ times. Within a replicate of
the linear experiment, all methods acquire the same responses, and for a fixed
batch size the outputs at smaller budgets are prefixes of the same runs. The
logistic schedules are run independently across methods and values of $m$.

\subsection{Additional Results}
\label{app:census-additional}
\paragraph{Further Results.}
The normalized mean squared error has the same thresholds
(\cref{fig:census-mse}). Variance estimates do not correct a shifted center:
beyond the breakdown of SGD, plug-in intervals cover $0\%$, and trajectory-based
intervals require five to eight times the oracle width.

\paragraph{Accuracy of the Gaussian-Moment Predictions.}
\Cref{fig:census-diagnostics}(a) compares Gaussian-moment predictions with
observed coverage over all
$1026$ configurations of the linear experiment: the median difference is $0.7$
percentage points, the $90\%$ quantile $1.9$, and the maximum $5.0$, while the
Monte Carlo standard error is about one point near the nominal level and two
points near $50\%$. Along the two eigendirections of
\cref{fig:census-coverage} the differences do not exceed $0.04$. The predicted
breakdown batch matches the observed grid point in $85\%$ of the
$81$ combinations of method, ridge level, functional, and budget, and is
within one grid step in $98\%$.
\Cref{fig:census-diagnostics}(b) shows that it grows with the budget, whereas the
gradient noise scale $\tr\Sigma/\norm{\nabla f(x_0)}^2$, which governs the
optimization speedup from batching \citep{mccandlish2018empirical}, does not
depend on $N$ (\cref{rem:critical-batch}).

\paragraph{The Spectrum and the Start.}
\Cref{fig:census-diagnostics}(c) shows the relative variance excess $v-1$ along
every eigenvector of $H$ at $N=2^{16}$ and $b=4$. The observations follow the
values up to Monte Carlo fluctuations of order $0.1$. The only large
excess is that of HB along the top eigenvector, $16$ against the value $14$ and the value
$\kappa^{3/2}/(8m)=14.5$ of \textcolor{black}{Statement \ref{thm:main-hb-stiff}}; the largest excess of NAG
is $0.17$, within the bound $30\sqrt\kappa/m=0.23$ of \textcolor{black}{Statement \ref{thm:main-nag}}.
\Cref{fig:census-diagnostics}(d) repeats \cref{fig:census-coverage}(c) from
$x_0=0$. The larger start lowers the breakdown batch of SGD from $2^{10}$ to
$2^6$ and that of NAG from $2^{12}$ to $2^{11}$, as the factor $\abs\theta$ in
\textcolor{black}{Statement \ref{cor:main-critical-batch}} predicts, while that of HB stays at $2^3$: HB is
limited by its covariance, not by its center.

\begin{figure}[H]
\centering
\includegraphics[width=0.7\textwidth]{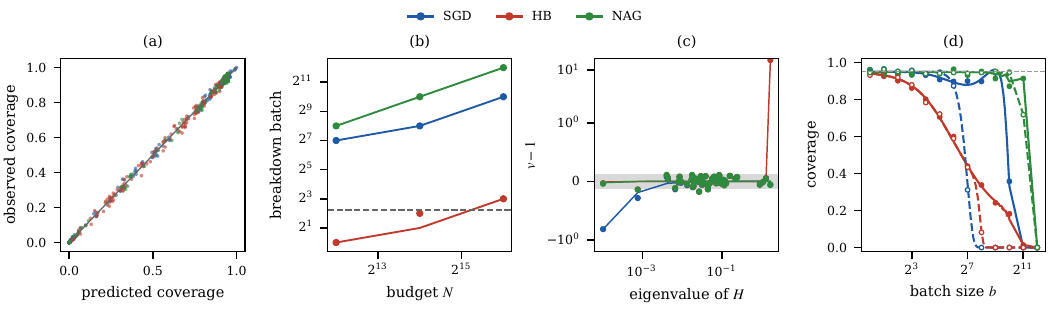}
\caption{\small (a) Predicted against observed coverage over all budgets, batch
sizes, ridge levels, functionals, and methods. (b) Breakdown batch for the age
coefficient against the budget: observed (markers) and predicted (lines); the
dashed line is the gradient noise scale. (c) Relative variance excess along the
eigenvectors of $H$ at $N=2^{16}$ and $b=4$ (markers: Monte Carlo; lines: exact;
grey band: $\pm2$ Monte Carlo standard errors). (d) Coverage for the age
coefficient starting from the pilot mean (solid lines, filled markers) and from
$x_0=0$ (dashed lines, open markers); lines are predictions and markers
Monte Carlo estimates.}
\label{fig:census-diagnostics}
\end{figure}

\paragraph{Mean Squared Error.}
\Cref{fig:census-mse} reports the same runs in the optimization metric
\[
 R_c:=\frac{N\,\E\bigl[(c^\trans(\hat x-x_\star))^2\bigr]}
            {c^\trans\Vstar c}=h_c^2+v_c,
\]
which tends to one in the PR regime. The batch at which this ratio first
exceeds two closely tracks the coverage breakdown in
\cref{fig:census-coverage} and gives the same method ordering. Along the
slowest eigendirection it occurs at
$b=4$ for SGD, $b=128$ for NAG, and $b=256$ for HB. Along the stiffest
eigendirection the HB ratio is already $4.2$ at $b=1$, whereas NAG remains
near one throughout the tested range. For the age coefficient, HB exceeds
two from $b=16$, SGD from $b=2^{10}$, and NAG only at $b=2^{12}$. The exact
finite-moment predictions track the observed curves. Across the displayed
points, the Monte Carlo estimates differ from the values by at most
$14\%$, on the scale of the shown $\pm2$-standard-error bands.

\begin{figure}[H]
\centering
\includegraphics[width=0.7\textwidth]{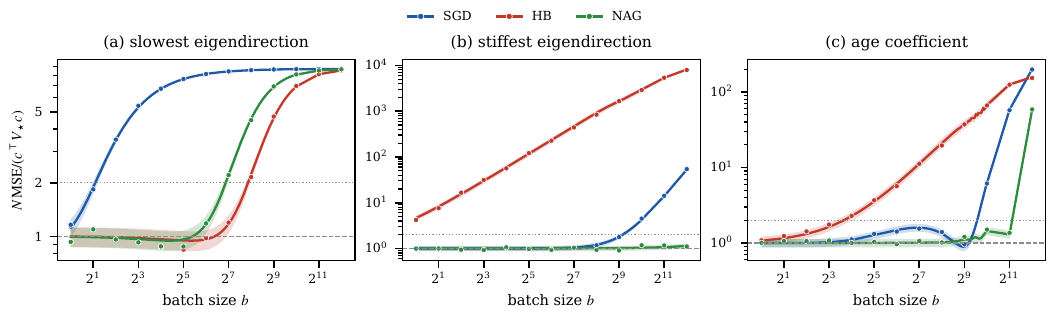}
\caption{Normalized mean squared error against the batch size at
$N=2^{16}$ and $\kappa=1.5\times10^4$, for the functionals of
\cref{fig:census-coverage}. Lines show the values $h_c^2+v_c$, shaded
bands $\pm2$ Monte Carlo standard errors for $500$ replicates around them, and
markers the Monte Carlo estimates. Gray horizontal dashed and dotted lines mark
the PR level and twice that level.}
\label{fig:census-mse}
\end{figure}

\paragraph{Gaussian Shape.}
To probe the shape of the output under skewed, bounded non-Gaussian noise, we
repeat the linear
experiment with untransformed income (in units of $\$10{,}000$) and $4000$
replicates; the standardized third moment of the influence values of the age
coefficient is $3.7$. After centering and scaling, the Kolmogorov distance
to the standard normal law never exceeds $0.018$, below the $5\%$ one-sample
KS critical value $0.022$, for all methods, $b\in\{1,16,256\}$, and
$N\in\{2^8,\dots,2^{14}\}$, although the variance of these outputs ranges from
$0.25$ to $14$ times the PR variance (\cref{fig:census-shape}a). Normalization by
the PR variance gives distances up to $0.96$ (\cref{fig:census-shape}b).
Across this tested grid, Gaussian shape is close after normalization
even when PR calibration is poor, consistent with the separation in
Statement \ref{thm:main-shape}.
A distance below the one-sample KS critical value
indicates non-rejection in that test, not a bound on the distance between
the underlying distributions. Exact standardization requires $m=N/b\ge2$;
when $N=b=256$, the output is the deterministic $x_0$ and its variance is zero.

\paragraph{Logistic Regression.}
Here $N=(t+m)b$ includes burn-in observations,
$\widehat x_N^{\rm ERM}$ is computed from all $N$ observations,
and $\Vstar=H^{-1}\Sigma H^{-1}$ uses $H=\nabla^2f(x_\star)$.
\Cref{fig:census-shape}(c) shows the mean of
$\sqrt N\norm{\Vstar^{-1/2}(\hat x-\widehat x_N^{\rm ERM})}$, the normalized gap between the
online output and the ERM on the same $N$ responses. With $b=m^2$, as in
\cref{cor:nonlinear-nag-erm}, the gap of NAG decreases from $5.7$ to $3.4$ at
$N\approx2.8\times10^6$ after $168$ rounds; with $b=m$ it reaches $2.6$ after
$1047$ rounds. Both are well below $6.6$, the normalized error of the ERM, so most of the sampling error of the online output is shared with the
ERM. With a constant batch the gap grows from $3.6$ to $9.5$, the signature of
the stationary bias of constant-gain averaging. Constant-step SGD with $b=m^2$
needs a three to four times longer burn-in than NAG and stays near $6.6$, because
its boundary term scales with $\kappa$ rather than $\sqrt\kappa$.

\begin{figure}[H]
\centering
\includegraphics[width=0.7\textwidth]{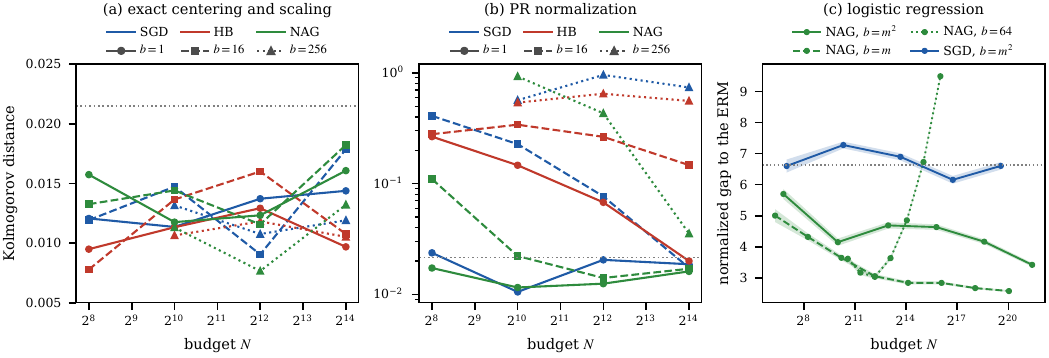}
\caption{\small (a, b) Kolmogorov distance to the standard normal law for
untransformed income after (a) and PR (b) normalization; colors mark the
methods, line styles the batch sizes, and the dotted line is the $5\%$ critical
value for $4000$ draws. (c) Normalized gap to the ERM in the
logistic model; the dotted line is the normalized error of the ERM, and bands are
$\pm2$ Monte Carlo standard errors.}
\label{fig:census-shape}
\end{figure}

\paragraph{Practical Intervals.}
\Cref{tab:census-practical} reports, for the age coefficient at $N=2^{16}$ and
$\kappa=1.5\times10^4$, the coverage of nominal $95\%$ intervals and, in
parentheses, their median width relative to the oracle interval, at a small
batch and past the breakdown of SGD. When the center is calibrated, the plug-in interval matches
the oracle interval. When it is not, no variance estimate restores coverage at
a comparable width: the plug-in interval undercovers, and the trajectory-based
intervals cover only by becoming several times wider. HB inflates them already
at $b=16$, because its variance is not the PR variance.

\begin{table}[H]
\caption{Coverage and Relative Width of Practical Intervals}
\label{tab:census-practical}
\begin{center}
\small
\begin{tabular}{@{}llcccc@{}}
\toprule
\textbf{Batch} & \textbf{Method} & \textbf{Oracle} & \textbf{Plug-in} &
\textbf{Batch means} & \textbf{Random scaling}\\
\midrule
$b=16$ & SGD & $0.93$ & $0.93$ $(1.0)$ & $0.91$ $(0.95)$ & $0.92$ $(1.1)$\\
 & HB & $0.80$ & $0.81$ $(1.0)$ & $1.00$ $(6.9)$ & $1.00$ $(4.4)$\\
 & NAG & $0.94$ & $0.94$ $(1.0)$ & $0.96$ $(1.1)$ & $0.95$ $(1.2)$\\
\midrule
$b=2^{11}$ & SGD & $0.00$ & $0.00$ $(1.0)$ & $0.91$ $(4.8)$ & $1.00$ $(7.8)$\\
 & HB & $0.01$ & $0.01$ $(1.0)$ & $0.95$ $(9.6)$ & $1.00$ $(33)$\\
 & NAG & $0.91$ & $0.91$ $(1.0)$ & $1.00$ $(12)$ & $1.00$ $(7.6)$\\
\bottomrule
\end{tabular}
\end{center}
\end{table}

\end{document}